\documentclass{article}

\PassOptionsToPackage{dvipsnames}{xcolor}
\usepackage{iclr2027_conference,times}
\usepackage[utf8]{inputenc} 
\usepackage[T1]{fontenc}    
\usepackage{hyperref}       
\usepackage{url}            
 
\usepackage{booktabs}       
\usepackage{amsfonts}       
\usepackage{nicefrac}       
\usepackage{microtype}      
\usepackage[dvipsnames]{xcolor}         
\definecolor{mypurple}{HTML}{6663A1}
\definecolor{myred}{HTML}{ED7D8D}
\definecolor{myblue}{HTML}{7CC9ED}
\definecolor{mygray}{HTML}{D9D9D9}
\usepackage{enumitem}
\usepackage{amsthm}
\usepackage{amsmath}
\usepackage{bm}
\usepackage{amssymb}
\usepackage{mathtools}
\usepackage[linesnumbered,ruled]{algorithm2e}
\usepackage{makecell}
\usepackage{multirow}
\usepackage{threeparttable}
\usepackage{capt-of}
\usepackage{wrapfig}
\usepackage{colortbl}
\usepackage{tikz}
\usepackage{pifont}
\usepackage[most]{tcolorbox}
\usepackage{setspace}
\usepackage{adjustbox}
\usepackage{fvextra}
\tcbuselibrary{breakable}
\usepackage{longtable}
\usepackage{wasysym}
\usepackage{graphicx}
\newcommand{\method}{\textsc{SAGE}}

\providecommand{\eps}{\epsilon}

\newcommand{\mycite}[1]{%
  \nocite{#1}%
  \hyperlink{cite.#1}{\ensuremath{\rhd}}%
}

\renewcommand\eqref[1]{\textup{(\ref{#1})}}

\theoremstyle{definition}

\newcommand{\Prob}{\mathbb{P}}

\newtcolorbox{mathbox}{
  colframe=black,
  colback=white,
  rounded corners,
  boxrule=0.5pt,
  breakable,
  left=0.5pt, right=0.5pt, top=0.5pt, bottom=0.3pt
}

\providecommand{\placeholdercell}{\cellcolor{mygray!2}{--}}
\providecommand{\placeholderrank}{\cellcolor{myred!2}{--}}
\providecommand{\NA}{--}
\providecommand{\litA}{\textsuperscript{\ensuremath{\dagger}}}
\providecommand{\litB}{\textsuperscript{\ensuremath{\ddagger}}}
\providecommand{\litC}{\textsuperscript{\ensuremath{\star}}}
\providecommand{\zsmark}{\textsuperscript{\textcolor{red}{0-shot}}}

\newcommand{\constructor}{\pi_{\theta}}
\newcommand{\retriever}{f_{\phi}}

\newcommand{\support}{\mathcal{D}^{+}}
\newcommand{\contextset}{\mathcal{D}}
\newcommand{\kg}{\mathcal{G}}
\newcommand{\triples}{\mathcal{T}}
\newcommand{\entityset}{\mathcal{V}_{e}}
\newcommand{\docset}{\mathcal{V}_{d}}
\newcommand{\edgeset}{\mathcal{E}}
\newcommand{\TopK}{\mathrm{Top}\text{-}k}
\newcommand{\I}{\mathbb{I}}
\newcommand{\R}{\mathbb{R}}
\newcommand{\E}{\mathbb{E}}
\newcommand{\A}{\mathcal{A}}
\newcommand{\Sset}{\mathcal{S}}
\newcommand{\act}{a}
\newcommand{\state}{s}

\newcommand{\reward}{\mathcal{R}}

\newcommand{\V}{\mathcal{V}}
\newcommand{\M}{\mathbf{M}}
\newcommand{\Hh}{\mathbf{H}}
\newcommand{\xx}{\mathbf{x}}
\newcommand{\hh}{\mathbf{h}}
\newcommand{\zz}{\mathbf{z}}
\newcommand{\rr}{\mathbf{r}}

\newcommand{\mm}{\mathbf{m}}

\newcommand{\one}{\mathbf{1}}
\newcommand{\sigmoid}{\operatorname{sigmoid}}
\newcommand{\softmax}{\operatorname{softmax}}
\newcommand{\mean}{\operatorname{mean}}
\newcommand{\std}{\operatorname{std}}
\newcommand{\diag}{\operatorname{diag}}
\newcommand{\dens}{\operatorname{dens}}

\newcommand{\Emb}{\operatorname{Emb}}
\newcommand{\MLP}{\operatorname{MLP}}
\newcommand{\BCE}{\operatorname{BCE}}
\providecommand{\BCEWithLogits}{\operatorname{BCEWithLogits}}

\newcommand{\gv}{\mathbf{g}}
\newcommand{\mv}{\mathbf{m}}
\newcommand{\bv}{\mathbf{b}}
\newcommand{\pv}{\mathbf{p}}
\providecommand{\ppi}{\boldsymbol{\pi}}
\newcommand{\Pv}{\mathbf{P}}
\newcommand{\uv}{\mathbf{u}}
\newcommand{\vvec}{\mathbf{v}}
\providecommand{\uu}{\mathbf{u}}
\providecommand{\vv}{\mathbf{v}}
\newcommand{\av}{\mathbf{a}}
\newcommand{\PReLU}{\operatorname{PReLU}}
\providecommand{\Dropout}{\operatorname{Dropout}}
\newcommand{\LayerNorm}{\operatorname{LayerNorm}}

\providecommand{\f}{f}

\newcommand{\SNR}{\operatorname{SNR}}
\newcommand{\ctx}{\mathrm{ctx}}
\newcommand{\sch}{\mathrm{sch}}

\newcommand{\leak}{\mathrm{leak}}

\newcommand{\bbone}{\mathbf{1}}

\newcommand{\calN}{\mathcal{N}}

\newcommand{\calB}{\mathcal{B}}

\providecommand{\method}{SAGE}

\newcommand{\cH}{\mathcal{H}}
\newcommand{\cG}{\mathcal{G}}
\newcommand{\cD}{\mathcal{D}}
\newcommand{\cR}{\mathcal{R}}

\newcommand{\norm}[1]{\left\lVert #1\right\rVert}
\newcommand{\abs}[1]{\left\lvert #1\right\rvert}

\DeclareMathOperator{\dist}{dist}
\DeclareMathOperator{\TV}{TV}

\DeclareMathOperator{\LN}{LN}

\DeclareMathOperator{\Rad}{Rad}
\theoremstyle{plain}
\newtheorem{theorem}{Theorem}[section]
\newtheorem{lemma}[theorem]{Lemma}
\newtheorem{corollary}[theorem]{Corollary}
\theoremstyle{definition}
\newtheorem{definition}[theorem]{Definition}
\newtheorem{assumption}[theorem]{Assumption}
\theoremstyle{remark}
\newtheorem{remark}[theorem]{Remark}
\providecommand{\Rad}{\operatorname{Rad}}
\providecommand{\dist}{\operatorname{dist}}
\newtheorem{proposition}{Proposition}

\tcbset{
  sagepropbase/.style={
    enhanced,
    breakable,
    arc=2mm,
    boxrule=0.7pt,
    left=1.2mm,
    right=1.2mm,
    top=1.0mm,
    bottom=1.0mm,
    before skip=8pt,
    after skip=8pt,
    fonttitle=\sffamily\bfseries\small,
    coltitle=white,
    attach boxed title to top left={xshift=2.5mm,yshift=-2.1mm},
    boxed title style={
      sharp corners,
      rounded corners=northwest,
      rounded corners=southeast,
      boxrule=0pt,
      left=1.2mm,
      right=1.2mm,
      top=0.6mm,
      bottom=0.6mm
    }
  },
  sagesnrbox/.style={
    sagepropbase,
    colback=SAGEBlueLight,
    colframe=SAGEBlue!75!black,
    boxed title style={colback=SAGEBlue!85!black},
    underlay unbroken and first={
      \fill[SAGEBlue!75!black]
      ([xshift=0pt]frame.south west)
      rectangle
      ([xshift=2.2pt]frame.north west);
    }
  },
  sagebudgetbox/.style={
    sagepropbase,
    colback=SAGEVioletLight,
    colframe=SAGEViolet!75!black,
    boxed title style={colback=SAGEViolet!85!black},
    underlay unbroken and first={
      \fill[SAGEViolet!75!black]
      ([xshift=0pt]frame.south west)
      rectangle
      ([xshift=2.2pt]frame.north west);
    }
  }
}

\newenvironment{sagepropbox}[3]{%
  \refstepcounter{proposition}\label{#2}%
  \begin{tcolorbox}[
    #3,
    title={Proposition~\theproposition. #1}
  ]%
}{%
  \end{tcolorbox}
}

\usepackage{amsmath,amssymb}
\usepackage[nameinlink,noabbrev]{cleveref}

\title{\textsc{SAGE}: A Self-Evolving Agentic Graph-Memory Engine for Structure-Aware\\Associative Memory}

\author{
Juntong Wang$^{1,2}$ \hspace{2mm}
Haoyue Zhao$^{3}$ \hspace{2mm}
Guanghui Pan$^{3}$ \hspace{2mm}
Yanbo Wang$^{1,2}$ \\
Xiyuan Wang$^{1,2}$ \hspace{2mm}
Qiyan Deng$^{3}$ \hspace{2mm}
Muhan Zhang$^{1}$\thanks{Corresponding author.} \\
$^1$Institute for Artificial Intelligence, Peking University \\
$^2$School of Intelligence Science and Technology, Peking University \\
$^3$School of Computer Science and Technology, Beijing Institute of Technology \\
\texttt{jtwang25@stu.pku.edu.cn},
\texttt{18503260963@163.com},
\texttt{3220251221@bit.edu.cn}, \\
\texttt{wangyanbo@stu.pku.edu.cn},
\texttt{wangxiyuan@pku.edu.cn}, \\
\texttt{qiyandeng@bit.edu.cn},
\texttt{muhan@pku.edu.cn}
}

\iclrfinalcopy 

\begin{document}

\maketitle

\begin{abstract}
Long-term memory is becoming a central bottleneck for language agents. Existing RAG and GraphRAG systems largely treat memory graphs as static retrieval middleware, which limits their ability to recover complete evidence chains from partial cues, exploit reusable graph-structural roles, and improve the memory itself through downstream feedback. We introduce \method{}, a \textbf{S}elf-evolving \textbf{A}gentic \textbf{G}raph-memory \textbf{E}ngine that models graph memory as a dynamic long-term memory substrate. \method{} couples two roles: a memory writer that incrementally constructs structured graph memory from interaction histories, and a Graph Foundation Model-based memory reader to perform retrieval and provide feedback to the memory writer. We provide rigorous theoretical analyses supporting the effectiveness of carefully designed architectural components and the framework. Across multi-hop QA, open-domain retrieval, domain-specific review QA, and long-term agent-memory benchmarks, SAGE improves evidence recovery, answer grounding, and retrieval efficiency: after two self-evolution rounds, it achieves the best average rank on multi-hop QA; in zero-shot open-domain transfer, it reaches EM/F1 of 46.20/54.21 and 53.00/66.06 on NQ and PopQA. Most importantly, when used as a replacement memory-access backend for existing memory systems, \method{} improves LongMemEval overall accuracy by 7.4 to 9.6 absolute points over the native MemOS, Memobase, and Zep backends after one evolution round, while paired HaluMem experiments improve most operation-level memory metrics and reduce QA hallucination. These results suggest that self-evolving, structure-aware graph memory is a promising way for robust long-horizon language agents. Our code is available \href{https://anonymous.4open.science/r/Unified-Representation-A9D9/}{here}.
\end{abstract}

\section{Introduction}
\label{sec:introduction}

As large language models evolve from single-turn question-answering systems into general-purpose agents for multi-turn dialogue, personalized assistance, multi-agent collaboration, and open-environment exploration, the system bottleneck is shifting from whether a model can answer within the current context to whether it can accumulate, organize, invoke, and update memory over longer time scales. Memory is a core system capability that determines whether Agents can achieve long-term consistency, personalized adaptation, cross-turn reasoning, and self-improvement.\textbf{ Memory is to Agents what parameters are to foundation models} ~\citep{park2023generative,zhong2024memorybank,packer2024memgpt,wu2025humanmemorysurvey,yang2026graphmemorysurvey}. Recent memory benchmarks have made this bottleneck explicit, evaluating agents on ultra-long conversational consistency, multi-session reasoning, temporal reasoning, knowledge updating, selective forgetting, abstention, and hallucination control~\citep{maharana2024locomo,wu2024longmemeval,hu2025memoryagentbench,chen2025halumem,li2026locomoplus}.

In engineering practice, RAG has become the dominant non-parametric interface for extending language models with external memory, alleviating the static nature of parametric knowledge and the limited size of context windows~\citep{lewis2020retrieval}. Yet standard RAG usually retrieves independent text chunks, whereas long-term agent memory often requires recovering evidence distributed across entities, events, aliases, temporal constraints, and multi-hop dependencies. GraphRAG takes an important step by organizing documents, entities, relations, and summaries as graphs, making cross-document dependencies and reasoning paths more explicit~\citep{edge2024local,gutierrez2024hipporag}. 
However, for long-horizon agents, graph structure should not merely serve as an external retrieval index. In this work, we study \emph{agent graph memory} as a coupled write--read--update problem. Given interaction histories or external documents, a memory writer should construct an \textbf{evolving graph} whose nodes and edges encode entities, episodes, documents, aliases, temporal constraints, and cross-fragment relations. Given a query, a memory reader should not simply expand from a few matched entities; it should return a \textbf{compact, verifiable evidence chain}. The retrieval outcome should further \textbf{provide feedback about what the graph lacks}. In other words, the graph is not only built before retrieval and searched afterward; it is the \textbf{working substrate through which memory is written, read, corrected, and self-improved}. Around this goal, we identify the following three core challenges.

\begin{figure}[t]
    \centering
    \includegraphics[width=\textwidth]{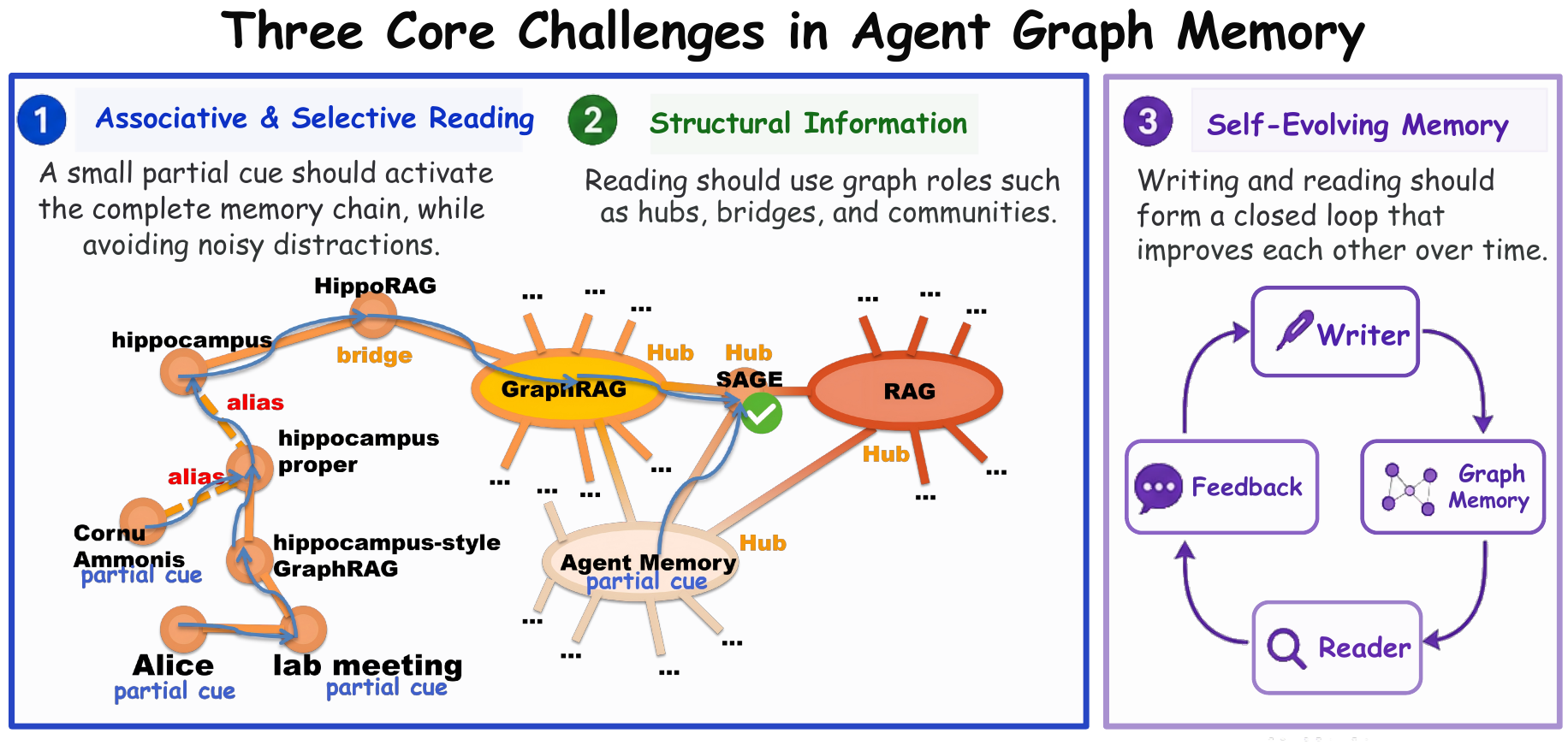}
    \vspace{-1.5em}
    \caption{
    Overview of the three core challenges in agent graph memory, illustrated with a concrete example. Given the query, \textit{\textbf{“Alice mentioned a work in last week’s lab meeting that seemed to be inspired by the Cornu Ammonis. Among works in the same field as that work, are there any that can also help with agent memory? Give one example,”}} the memory reader must address associative and selective reading by expanding sparse partial cues into the correct evidence chain while avoiding noisy distractors. It must then exploit structural information in the memory graph, such as aliases, bridges, and hubs, to traverse from Cornu Ammonis to SAGE. Self-evolving memory highlights the closed loop between writing and reading.
    }
    \label{fig:agent_graph_memory_challenges}
    \vspace{-1.5em}
\end{figure}

\noindent\textbf{Challenge I: Agent memory requires global associative reading from fragmented cues.}
The first challenge is not merely to retrieve text that is semantically similar to the query, but to reconstruct a complete reasoning chain from sparse, fragmented, and sometimes indirect cues. In long-term agent memory, a query may mention only an episodic clue, an alias, or a distant conceptual hint, while the answer depends on intermediate entities that are not explicitly named. Standard vector retrieval tends to return locally similar snippets, and many graph-based retrieval methods start propagation from a small set of query-matched anchor entities. However, if these anchors only cover a local subgraph, the necessary bridge nodes may lie outside the activated region, leaving the evidence chain disconnected even after graph propagation. Thus, agent memory reading should not commit too early to a small set of partial cues~\citep{trivedi2023interleaving,gutierrez2024hipporag}.

\noindent\textbf{Challenge II: Agent memory requires learned structural use rather than fixed structural expansion.}
The second challenge is that graph structure should not be used only as a fixed index after graph construction. Many GraphRAG-style systems exploit structure through pre-built communities, paths, graph indexes, or heuristic expansion rules, but once the graph is constructed, the role of structure is largely fixed: a hub remains broadly expanded, a bridge may be missed if it is not reached by the initial anchors, and noisy shortcuts may be treated similarly to useful evidence edges. This is insufficient for agent memory, where the graph itself is continuously updated by new interactions and where the same topological pattern may have different meanings across domains. A structure-aware reader should therefore learn how structural roles affect retrieval~\citep{edge2024local,gutierrez2024hipporag,liu2025graphfoundation,luo2025gfm}.

The example in Figure~\ref{fig:agent_graph_memory_challenges} illustrates both challenges. Given the query, the explicit cues are only \emph{Alice}, \emph{lab meeting}, \emph{Cornu Ammonis}, and \emph{agent memory}. A retrieval system that only anchors on the most query-matched nodes may retrieve the meeting note or the biological cue, but still fail to connect them to \emph{hippocampus}, \emph{HippoRAG}, \emph{GraphRAG}, and finally \method{}. This is the associative-reading challenge: the system must piece together a long chain from scattered cues. At the same time, the correct path depends on structural roles: \emph{HippoRAG} is a bridge, \emph{GraphRAG} and \emph{RAG} are hubs that must be controlled rather than blindly expanded, and the edge from \emph{GraphRAG} to \method{} is critical for reaching the final answer. This is the structural-information challenge: the reader must use graph topology in a learned and selective way, not simply propagate uniformly over a fixed graph.

\noindent\textbf{Challenge III: Existing methods mostly optimize retrieval trajectories, but rarely optimize the self-evolution of the memory system itself.}
Existing RAG and GraphRAG systems often assume that the external memory graph or knowledge base is already available, so the main problem becomes how to retrieve from it. For long-term agents, however, writing is itself part of the memory problem. Conversely, retrieval failures provide useful signals about what the memory graph lacks. For example, if the reader repeatedly needs to traverse from \emph{Cornu Ammonis} to \emph{hippocampus-style GraphRAG} and then to the GraphRAG literature, the memory system should gradually add or strengthen useful structural links, such as a more direct edge from \emph{hippocampus-style GraphRAG} to \emph{GraphRAG}. Thus, a true agent memory system should not only optimize retrieval trajectories; it should optimize the memory graph itself through a closed loop in which better reading exposes writing deficiencies, and better writing makes future reading more accurate, selective, and efficient~\citep{chen2025halumem}.

To address these challenges, we propose \method{}, a \textbf{\textit{S}}elf-evolving \textbf{\textit{A}}gentic \textbf{\textit{G}}raph-memory \textbf{\textit{E}}ngine. Unlike GraphRAG systems that mainly use graphs as retrieval middleware, \method{} treats the graph as a dynamic long-term memory object. It couples two mutually reinforcing components: a memory writer that incrementally constructs structured graph memory, and a Graph Foundation Model-based memory reader that retrieves evidence and provides feedback to the writer. \method{} directly targets the three challenges above: it recovers long reasoning chains from fragmented cues, learns how to use structural roles rather than propagate uniformly, and continuously improves the graph memory itself for future queries.

\section{Related Work}
\label{sec:related_work}

\paragraph{Retrieval-Augmented Generation and GraphRAG.}
Retrieval-Augmented Generation (RAG) provides a non-parametric interface for language models by retrieving external evidence before generation~\citep{lewis2020retrieval}. Many variants further improve retrieval timing, reasoning interaction, adaptive policies, and hierarchical organization~\citep{jiang2023active,trivedi2023interleaving,asai2024selfrag,jeong2024adaptive,sarthi2024raptor}. GraphRAG enables structured retrieval over cross-document dependencies and multi-hop evidence paths~\citep{edge2024local,he2024g,guo2024lightrag,li2025simple,wang2025proprag,xu2025noderag,zhao2025e2graphrag,zhang2025survey,gutierrez2024hipporag,gutierrez2025rag,luo2025gfm}. Another line of work improves retrieval by optimizing retrieval trajectories, including interleaved retrieval and reasoning, self-reflective retrieval, adaptive retrieval, multi-agent RAG, and reinforcement-learning-based query rewriting~\citep{trivedi2023interleaving,asai2024selfrag,jeong2024adaptive,chen2025marlrag,cha2025rlqr,tsang2025autograph}.

\paragraph{Agent Memory.}
Agent memory studies how LLM-based agents store, update, retrieve, and use past experiences~\citep{park2023generative,zhong2024memorybank,packer2024memgpt,chhikara2025mem0,xu2025amem,rasmussen2025zep,kang2025memoryos,zhang2025gmemory,wu2025sgmem,zhang2025assomem,huang2025licomemory,yue2026hypermem}. Recent surveys also highlight the importance of human-inspired and graph-based memory mechanisms for LLM agents~\citep{wu2025humanmemorysurvey,yang2026graphmemorysurvey}. Meanwhile, memory benchmarks evaluate long-term consistency, event reasoning, multi-session reasoning, temporal reasoning, knowledge updating, selective forgetting, abstention, and hallucination control~\citep{maharana2024locomo,wu2024longmemeval,hu2025memoryagentbench,chen2025halumem,li2026locomoplus}. 

\paragraph{Graph Foundation Models.}
Graph Foundation Models (GFMs) aim to learn transferable graph representations through large-scale pretraining, allowing models to reuse structural priors and semantic patterns across graphs, tasks, and domains~\citep{liu2025graphfoundation}. Representative early works include GCC, GPT-GNN, and GraphCL, which learn transferable graph representations through cross-network contrast, generative graph pretraining, and graph augmentation based contrastive learning~\citep{qiu2020gcc,hu2020gptgnn,you2020graphcl,yu2025samgpt}.

\section{Preliminary}
\label{sec:preliminary}

We consider a knowledge-intensive memory sample \(x=(q,\mathcal{D},\mathcal{D}^{+},y)\), where \(q\) denotes the query, \(\mathcal{D}=\{d_i\}_{i=1}^{N}\) is the set of candidate historical memory fragments, \(\mathcal{D}^{+}\subseteq\mathcal{D}\) is the gold evidence set supporting the answer, and \(y\) is the ground-truth answer. Depending on the benchmark, a memory fragment may correspond to a document, review snippet, or interaction/session unit. The writer is modeled as a structured policy: at step \(h\), given state \(s_h\), it samples a writing action \(a_h\sim\pi_{\theta}(\cdot\mid s_h)\) and updates the partial graph as \(\mathcal{G}_{h+1}=\mathcal{G}_{h}\oplus a_h\). The memory reader \(\mathcal{R}_{\phi}\) performs query-conditioned propagation over the graph, obtains entity relevance scores \(\mathbf{s}_{E}=f_{\phi}(q,\mathcal{G})\in\mathbb{R}^{|\mathcal{V}_{E}|}\), and then projects them into memory-fragment scores. The reader finally outputs \((\widehat{\mathcal{D}}_{k},\widehat{\mathcal{G}}_{q},\Pi_q)=\mathcal{R}_{\phi}(q,\mathcal{G},\mathbf{M})\), where \(\widehat{\mathcal{D}}_{k}=\mathrm{TopK}_{d\in\mathcal{D}}(\mathbf{s}_{D}(d))\), \(\widehat{\mathcal{G}}_{q}\) is the query-activated subgraph, and \(\Pi_q\) denotes optional relational paths. The generation model then produces the answer \(\widehat{y}=\mathrm{LLM}(q,\widehat{\mathcal{D}}_{k},\Pi_q)\).


\definecolor{SAGEBlue}{HTML}{FFB6C1}
\definecolor{SAGEBlueLight}{HTML}{FFF5F8}
\definecolor{SAGEViolet}{HTML}{FFB6C1}
\definecolor{SAGEVioletLight}{HTML}{FFF5F8}

\section{Method}
\label{sec:method}

\begin{figure}[t]
    \centering
    \includegraphics[width=\textwidth]{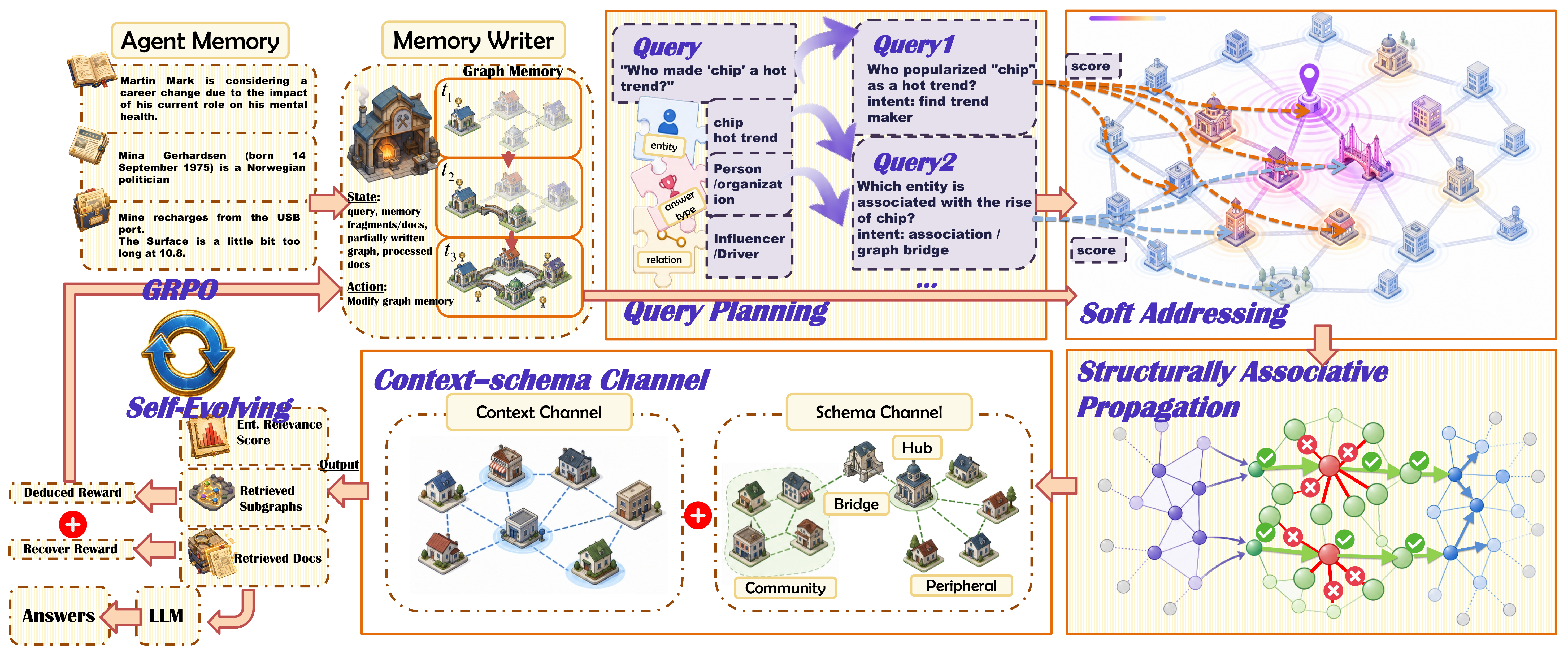}
    \vspace{-1.2em}
    \caption{
    Overall pipeline of the proposed \textbf{SAGE}.
    The memory writer incrementally constructs and updates graph memory from observations through state-conditioned writing actions, and receives rewards from downstream memory use.
    The resulting retrieval feedback closes the loop between writing and reading, enabling graph memory to improve over time.
    }
    \label{fig:memory_writer_pipeline}
    \vspace{-1.2em}
\end{figure}


At a high level, our method builds a self-evolving graph memory pipeline (Figure~\ref{fig:memory_writer_pipeline}). The memory writer \(\mathcal{W}_{\theta}\) first transforms the query and candidate historical memory fragments into a heterogeneous graph memory \(\mathcal{G}\). The memory reader \(\mathcal{R}_{\phi}\) then performs query-conditioned activation over \(\mathcal{G}\): it softly locates query-relevant entities, propagates evidence signals through relational structures, and projects the activated entity-level information back to memory fragments.

\subsection{Memory Writer: Graph Memory Writing via Reading Feedback}
\label{subsec:memory_writer}

\paragraph{Policy-based writing.}
The writer is modeled as a sequential decision-making policy. At step $t$, the state is defined as
\(s_t=\big(q,\contextset,\kg_{t-1},\mathcal{D}^{\mathrm{proc}}_{t-1}\big)\), where $\kg_{t-1}$ is the partially written graph, and $\mathcal{D}^{\mathrm{proc}}_{t-1}$ denotes the set of processed documents. The action $a_t$ contains entity-relation triples $(u,r,v)$ together with their source anchors $(u,\texttt{source},d)$. Detailed implementation information is provided in Appendix~\ref{app:writer_details}.

\paragraph{Reader-aware Writing Reward.}
The writer's reward stems from the task utility of its written graph after being accessed by the memory reader. Given the current graph $\kg$, the frozen reader returns the evidence $P_k(q,\kg)$. Inspired by~\citep{tsang2025autograph}, we employ two complementary types of rewards. The first category measures whether the graph is sufficient as a knowledge carrier to support the derivation of the answer: \(r_{\mathrm{ded}}(q,y,\kg)
    =
    \mathbb{I}
    \left[
        \operatorname{Judge}
        \big(q,y\mid P_k(q,\kg)\big)
        =
        \texttt{Yes}
    \right]\). The second category measures whether the graph can serve as a knowledge index to recover the supporting text: \(r_{\mathrm{rec}}(q,\support,\kg)
    =
    \frac{|P_k(q,\kg)\cap\support|}{|\support|},
    r_{\mathrm{pre}}(q,\support,\kg)
    =
    \frac{|P_k(q,\kg)\cap\support|}{|P_k(q,\kg)|}\). Here, $r_{\mathrm{rec}}$ encourages the coverage of necessary evidence, while $r_{\mathrm{pre}}$ penalizes the expansion of irrelevant evidence. To align with end-to-end question answering, we also use an answer-level auxiliary reward \(r_{\mathrm{ans}}(q,y,\kg)
    =
    \max_{y'\in\mathcal{Y}(y)}
    \mathrm{F1}\big(\hat y,y'\big),
    \hat y=\operatorname{LLM}\big(q,P_k(q,\kg)\big)\), where $\mathcal{Y}(y)$ is the set of answer aliases. In practice, we adopt a hybrid task reward \(r_{\mathrm{task}}
=
\frac{
\alpha r_{\mathrm{rec}}
+
\beta r_{\mathrm{pre}}
+
\gamma r_{\mathrm{ded}}
+
\delta r_{\mathrm{ans}}
}{
\alpha+\beta+\gamma+\delta
}\).
   
Furthermore, to prevent the policy from inflating the graph size by stacking duplicate triples, we define a repetition rate: \(\rho_{\mathrm{rep}}(\kg)
    =
    \dfrac{
        |\triples(\kg)|-
        |\operatorname{uniq}(\triples(\kg))|
    }{
        |\triples(\kg)|
    }\)
and derive the trajectory return \(R(\tau)
    =
    r_{\mathrm{task}}(\tau)
    -
    \lambda_{\mathrm{rep}}\rho_{\mathrm{rep}}(\kg_\tau)
    +
    \lambda_{\mathrm{fmt}}
    \sum_{t=1}^{|\tau|}
    r^{\mathrm{fmt}}_t\) . This directly addresses the issue revealed by works such as HaluMem: errors in memory systems often do not emerge only at the answering stage, but are already written during the extraction and updating phases~\citep{chen2025halumem}. We employ the standard clipped GRPO objective~\citep{shao2024deepseekmath} to update the writer.

\subsection{Memory Reader: Memory Retrieval Based on Graph Foundation Model}
\label{subsec:memory_reader}

The memory reader must operate stably over graph memory that is continuously updated by the writer. Dense retrievers mainly learn query--document semantic matching and thus struggle to exploit entity roles, bridge paths, and cross-community dependencies, while conventional GNN retrievers are often tied to fixed graph distributions and generalize poorly across domains, users, and evolution stages. We therefore adopt a Graph Foundation Model (GFM) as the memory reader, whose multi-graph pre-training enables transferable structural priors and lightweight calibration on new graphs~\citep{luo2025gfm,zhang2025survey}. Formally, the memory reader outputs an entity distribution, a document distribution, and an optional retrieval subgraph \(\retriever(q,\kg,\contextset)
    =
    \big(
        p_\phi(e\mid q,\kg),
        p_\phi(d\mid q,\kg,\contextset),
        \kg_q
    \big)\). Here, $p_\phi(e\mid q,\kg)$ represents the entity memory activated by the query, $p_\phi(d\mid q,\kg,\contextset)$ denotes the final retrieved textual evidence, and $\kg_q$ provides an interpretable retrieval path. To obtain a compact and query-aligned activated subgraph, we further introduce a lightweight query-conditioned subgraph selector; implementation details are provided in Appendix~\ref{app:selector_regularizer}.

\paragraph{Cognition-inspired Structured Query Planning.} 
When humans extract long-term memories, the brain often automatically generates multi-dimensional retrieval cues to anchor the target based on only a vague final intention. Inspired by this, we no longer treat the natural language query as a single retrieval command. Instead, we introduce a planning function $\mathcal{P}_\omega$ to simulate the cue reconstruction process of the human brain before awakening memory, decomposing the initial query into a set of rich associative probes: \(\mathcal{P}_\omega(q)
    =
    \Big(
        \mathcal{E}_{\mathrm{exp}},
        \mathcal{A},
        \mathcal{C}_{\mathrm{rel}},
        \mathcal{C}_{\mathrm{hard}},
        \tau,
        \{(\tilde q_m,\alpha_m,t_m)\}_{m=1}^{M}
    \Big)\). Detailed definitions of the notation, additional information, and the concrete prompt templates and output schema are provided in Appendix~\ref{app:rewriter_details}. This multi-path concurrent awakening method effectively overcomes the "tip-of-the-tongue phenomenon" (i.e., difficulties in alias alignment or missing bridging entities) and naturally stitches together forgotten implicit relationships~\citep{trivedi2023interleaving,asai2024selfrag,wu2025sgmem,zhang2025assomem}.

\paragraph{Soft Addressing and Pre-activation of Memory Fragments.}
Cognitive neuroscience reveals that human memory retrieval involves not only the extraction of perfectly matching information but also the instinctive awakening of peripherally related memories through \textit{Semantic Priming}. And to address the first challenge, we treat the calculation of the query-conditioned entry score $s_e(q)$ as a comprehensive assessment of the stimulus intensity across different \textit{Memory Engrams}:
\begin{align}
    s_e(q)
    =
    &\lambda_1\,\operatorname{Exact}(e,\mathcal{E}_{\mathrm{exp}})
    +\lambda_2\,\operatorname{Alias}(e,\mathcal{A})
    +\lambda_3\max_{m\le M}
    \cos\big(
        \operatorname{Emb}(\operatorname{desc}(e)),
        \operatorname{Emb}(\tilde q_m)
    \big)
    \nonumber\\
    &+
    \lambda_4\,\operatorname{Type}(e,\tau)
    +\lambda_5\,\operatorname{Cons}(e,\mathcal{C}_{\mathrm{hard}})
    +\lambda_6
    \sum_{\xi\in\operatorname{NER}(q)}
    \operatorname{EL}(e\mid \xi).
    \label{eq:soft_address}
\end{align}
Subsequently, the system employs a Softmax function with a temperature coefficient $T_0$ to simulate the brain's limited Attention Allocation mechanism during retrieval. This normalizes the multi-dimensional stimulus signals to form the initial activation distribution of the memory atlas \(p_0(e\mid q)
    =
    \frac{\exp(s_e(q)/T_0)}
    {\sum_{v\in\V_E}\exp(s_v(q)/T_0)}\). Based on this distribution, we define the initial state of the memory nodes as \(\hh_e^{(0)}
    =
    \big(p_0(e\mid q)\big)^\eta W_q\operatorname{Emb}(q)
    +
    W_x \xx_e\). In this process, $\xx_e$ acts as the solidified \textbf{long-term memory} (static representation of entities) in the brain, while the query vector adjusted by the cognitive recall degree $p_0(e\mid q)$ represents the current \textbf{working memory} (task context).

\paragraph{Synapse-inspired Structurally Conditioned Associative Propagation.}
To address the second challenge while avoiding indiscriminate diffusion, we introduce edge-level vector structural gating in the GFM. The node-level structural features, edge-pair structural features, and graph-level summary are defined as:
\begin{align}
    \phi(v)
    &=
    \big[
        \log(1+d_v),
        c_v,
        \kappa_v,
        \bar d_{\mathcal{N}(v)}
    \big],
    \label{eq:node_struct}\\
    \psi(u,v)
    &=
    \big[
        |d_u-d_v|,
        |\mathcal{N}(u)\cap\mathcal{N}(v)|,
        \operatorname{Jaccard}(\mathcal{N}(u),\mathcal{N}(v))
    \big],
    \label{eq:edge_struct}\\
    \rr_{\kg}
    &=
    \big[
        \operatorname{mean}_{v\in\V_E}\phi(v);
        \operatorname{std}_{v\in\V_E}\phi(v);
        \operatorname{dens}(\kg)
    \big].
    \label{eq:graph_struct}
\end{align}
Detailed definitions and normalization procedures are provided in Appendix~\ref{app:structure_features_detailed}. The edge structural context for the $l$-th layer is \(\zz_{uv}^{(l)}
    =
    \big[
        E_n^{(l)}(\phi(u));
        E_n^{(l)}(\phi(v));
        E_p^{(l)}(\psi(u,v));
        E_g^{(l)}(\rr_\kg)
    \big]\), which generates the vector gating \(\gv_{uv}^{(l)}
    =
    \one
    +
    \delta
    \tanh
    \big(
        \operatorname{MLP}_g^{(l)}(\zz_{uv}^{(l)})
    \big)\). Let $\eta_{uv}$ be the normalized adjacency weight with self-loops; the message and node updates are \(\mv_{u\to v}^{(l)}=
    \eta_{uv}
    \,
    \gv_{uv}^{(l)}
    \odot
    W_m^{(l)}
    \hh_u^{(l-1)}, \hh_v^{(l)}
    =
    \operatorname{LayerNorm}
    \left(
        \hh_v^{(l-1)}
        +
        \operatorname{PReLU}
        \left(
            \bv^{(l)}
            +
            \sum_{u\in\mathcal{N}(v)}
            \mv_{u\to v}^{(l)}
        \right)
    \right)\). Unlike traditional heuristic path expansion, PPR walks, or community summarization~\citep{edge2024local,guo2024lightrag,wang2025proprag}, the system here can actively perform \textbf{Inhibition} of non-specific generalized memories (suppressing hub edges), keenly capture and preserve \textbf{long-distance associations} across different cognitive clusters (lateral thinking/bridge edge preservation), and undergo \textbf{Habituation} (weakening redundant edges) toward highly repetitive local information, much like the human brain.

Traditional query-dependent GNNs or PPR-style expansion can perform multi-hop propagation along graph structures. But the key issue is not simply to expand the propagation range, but to preserve the advantage of query-relevant evidence signal over distractor noise under a limited top-$k$ budget. Proposition~\ref{prop:method_reader_theory}(i) summarizes this signal--budget view: soft addressing improves the initial evidence activation, structural gating preserves bridge/evidence paths while suppressing noisy neighborhoods, and controlled entity-to-document projection converts the entity-level advantage into more efficient document-level retrieval. Complete definitions, assumptions, and proofs are provided in Appendix~\ref{sec:theory_signal_noise_budget}.

\paragraph{Target Graph Calibration and Cross-graph Structural Priors.}
Human memory, on one hand, reorganizes cues based on the current context, while on the other, it retains relatively stable structured recall habits. In our self-evolving graph memory, each $\kg$ generated by the writer per round alters the local topology and noise distribution; therefore, the reader cannot rely solely on propagation patterns from a fixed graph. 

Since the writer continuously changes the memory graph, the reader must simultaneously adapt to the current target graph and preserve cross-graph structural priors. This is precisely why we introduce the context--schema decomposition. As summarized in Proposition~\ref{prop:method_reader_theory}(ii), the schema channel provides a transferable structural prior, while the context channel corrects the target-graph residual induced by the current writer, current domain, entity granularity, and local noise. The complete theoretical motivation is provided in Appendix~\ref{app:theory_schema_prior} and Appendix~\ref{app:theory_graph_shift}.


First, a feature prompt vector $\pv_f$ is used for a lightweight calibration of the query-activated input \(\tilde\hh_e^{(0)}
    =
    \pv_f\odot \hh_e^{(0)}\). The contextual calibration channel performs gated propagation on the current graph $\kg$: \(\Hh_{\mathrm{ctx}}
    =
    F_{\mathrm{gate}}
    (\tilde\Hh^{(0)},\kg;\Theta_{\mathrm{gate}})\). Here, $\Hh_{\mathrm{ctx}}$ captures the immediate structural state within the current memory graph. Simultaneously, the schema prior channel maintains a set of cross-graph structural prompt bases $\{\Pv_j^{(l)}\}_{j=1}^{K}$, which are used to encode stable reading habits formed during multi-graph training: \(\omega_j^{(l)}
    =
    \operatorname{softmax}_j(\av^{(l)}/T_p),
    \Pv_{\mathrm{schema}}^{(l)}
    =
    \sum_{j=1}^{K}
    \omega_j^{(l)}\Pv_j^{(l)}\). Propagation is executed based on these schema prompts to obtain: \(\Hh_{\mathrm{sch}}
    =
    F_{\mathrm{prompt}}
    \big(
        \tilde\Hh^{(0)},
        \kg;
        \{\Pv_{\mathrm{schema}}^{(l)}\}_{l=1}^{L}
    \big)\). The final entity representation is jointly determined by the current context and the long-term schema: \(\Hh(q,\kg)
    =
    \Hh_{\mathrm{ctx}}
    +
    \beta_{\mathrm{sch}}\Hh_{\mathrm{sch}}\). Here, $\Hh_{\mathrm{ctx}}$ is analogous to a context-dependent immediate recall state, responsible for adapting to the specific graph structure generated by the current writer; $\Hh_{\mathrm{sch}}$ is akin to a memory schema formed across experiences, retaining the ability to recognize stable patterns such as bridge nodes, community boundaries, core--periphery structures, and noise short-circuits.

\paragraph{Reader Training.}
Reader training aims to learn cross-graph transferable retrieval biases through a two-stage procedure. First, we perform structural contrastive pre-training on multiple augmented graph views. Then, in the supervised fine-tuning stage, we align these transferable capabilities with question-driven evidence retrieval by training the reader to identify and rank supporting entities for each query using weighted classification and multi-positive ranking objectives. Implementation details are provided in Appendices~\ref{app:pretraining_details} and~\ref{app:finetune_objective_detailed}.

\paragraph{Writer-Reader Self-evolution.}
\label{subsec:self_evolution_algorithm}
 To address the third challenge, we propose a self-evolution framework. Each of our self-evolution iterations consists of two phases. First, we fix the reader and train the writer using its retrieval results as rewards. Subsequently, we use the updated writer to generate new graphs and continue training the reader. The overall procedure is detailed in Algorithm~\ref{alg:self_evolving_memory}.

\begin{sagepropbox}
{Theoretical consequences of SAGE}
{prop:method_reader_theory}
{sagesnrbox}

\begin{enumerate}
\item[\textnormal{(i)}] \textbf{Signal--budget efficiency.}
    Soft addressing, structural gating, and controlled entity-to-document projection jointly improve evidence signal over distractor noise, thereby reducing the top-$k$ budget needed for evidence coverage.
    \par\hfill
    \hyperref[thm:realistic_snr]{\textnormal{\sffamily\small[$\rhd$ SNR proof]}}\;
    \hyperref[thm:realistic_signal_noise_budget]{\textnormal{\sffamily\small[$\rhd$ Budget proof]}}

\item[\textnormal{(ii)}] \textbf{Context--schema decomposition.}
    The reader combines transferable structural priors with target-graph calibration, so adaptation only needs to correct graph-specific residuals.
    \par\hfill
    \hyperref[app:theory_schema_prior]{\textnormal{\sffamily\small[$\rhd$ Proof]}}

\item[\textnormal{(iii)}] \textbf{Evolution stability.}
    Under bounded graph drift, consecutive writer updates induce bounded document-score changes.
    \par\hfill
    \hyperref[app:theory_reader_stability]{\textnormal{\sffamily\small[$\rhd$ Proof]}}
\end{enumerate}
\end{sagepropbox}


From a theoretical perspective, this process can be interpreted as approximate coordinate improvement over a joint memory utility: the writer update improves the readability of the graph memory, while the reader update reduces writer-induced graph distribution shift and reward bias.


\begin{table}
\renewcommand{\arraystretch}{0.86}
\setlength{\tabcolsep}{2.3pt}
  \centering
  \vspace{-0.5cm}
  \caption{Open-domain retrieval and QA results on \texttt{NQ} and \texttt{PopQA}. We report passage-level Recall (\%) at top-2 and top-5 and end-to-end EM/F1 (\%) when comparable numbers are available in original papers or later works that reproduce/cite these methods. Best available results are in \textbf{bold} and runner-ups are \underline{underlined}. \textbf{\textcolor{red}{Only rows marked with \zsmark{} are our zero-shot transfer results; baseline rows are not marked as zero-shot.}}}
  \vspace{0.2cm}
  \resizebox{0.85\linewidth}{!}{
    \begin{tabular}{lcccccccc}
    \toprule
    \multicolumn{9}{c}{\cellcolor{red!8}\textbf{\textcolor{red}{Zero-shot setting applies only to \method{} on \texttt{NQ} and \texttt{PopQA}.}}} \\
    \midrule
    \textbf{Dataset}
    & \multicolumn{4}{c}{\texttt{NQ}}
    & \multicolumn{4}{c}{\texttt{PopQA}} \\
    \cmidrule(r{1mm}){1-1}
    \cmidrule(l{0.5mm}r{1mm}){2-5}
    \cmidrule(l{0.5mm}){6-9}
    \textbf{Method}
    & R@2$_\textsf{P}$ & R@5$_\textsf{P}$ & EM & F1
    & R@2$_\textsf{P}$ & R@5$_\textsf{P}$ & EM & F1 \\
    \cmidrule(r{1mm}){1-1}
    \cmidrule(l{0.5mm}r{1mm}){2-5}
    \cmidrule(l{0.5mm}){6-9}

    \texttt{BM25}
    \scalebox{0.68}{(\mycite{robertson1994some}~\textit{SIGIR'94})}
    & 34.\scalebox{0.75}{7}
    & 51.\scalebox{0.75}{1}
    & 24.\scalebox{0.75}{5}
    & 42.\scalebox{0.75}{13}
    & 36.\scalebox{0.75}{3}
    & 55.\scalebox{0.75}{7}
    & 42.\scalebox{0.75}{4}
    & 52.\scalebox{0.75}{60} \\

    \texttt{Contriever}
    \scalebox{0.68}{(\mycite{izacard2022unsupervised}~\textit{TMLR'22})}
    & 29.\scalebox{0.75}{1}\litA
    & 54.\scalebox{0.75}{6}\litA
    & 45.\scalebox{0.75}{0}\litA
    & 58.\scalebox{0.75}{9}\litA
    & 27.\scalebox{0.75}{0}\litA
    & 43.\scalebox{0.75}{2}\litA
    & 41.\scalebox{0.75}{6}\litA
    & 53.\scalebox{0.75}{1}\litA \\

    \texttt{GTR}
    \scalebox{0.68}{(\mycite{ni2022large}~\textit{EMNLP'22})}
    & 35.\scalebox{0.75}{0}\litA
    & 63.\scalebox{0.75}{4}\litA
    & \underline{~45.\scalebox{0.75}{5}\litA~}
    & \underline{~59.\scalebox{0.75}{9}\litA~}
    & 40.\scalebox{0.75}{1}\litA
    & 49.\scalebox{0.75}{4}\litA
    & 43.\scalebox{0.75}{2}\litA
    & 56.\scalebox{0.75}{2}\litA \\

    \texttt{ColBERTv2}
    \scalebox{0.68}{(\mycite{santhanam2022colbertv2}~\textit{NAACL'22})}
    & 35.\scalebox{0.75}{7}
    & 53.\scalebox{0.75}{3}
    & 25.\scalebox{0.75}{9}
    & 42.\scalebox{0.75}{91}
    & \underline{~41.\scalebox{0.75}{2}~}
    & \underline{~62.\scalebox{0.75}{4}~}
    & \underline{~46.\scalebox{0.75}{7}~}
    & \underline{~58.\scalebox{0.75}{05}~} \\

    \texttt{RAPTOR}
    \scalebox{0.68}{(\mycite{sarthi2024raptor}~\textit{ICLR'24})}
    & \underline{~40.\scalebox{0.75}{3}\litA~}
    & \underline{~68.\scalebox{0.75}{3}\litA~}
    & 36.\scalebox{0.75}{9}\litA
    & 50.\scalebox{0.75}{7}\litA
    & 40.\scalebox{0.75}{2}\litA
    & 48.\scalebox{0.75}{7}\litA
    & 43.\scalebox{0.75}{1}\litA
    & 56.\scalebox{0.75}{2}\litA \\

    \texttt{Proposition}
    \scalebox{0.68}{(\mycite{chen2024dense}~\textit{EMNLP'24})}
    & 33.\scalebox{0.75}{1}\litC
    & 62.\scalebox{0.75}{2}\litC
    & \NA
    & \NA
    & \NA
    & \NA
    & \NA
    & \NA \\

    \cmidrule(r{1mm}){1-1}
    \cmidrule(l{0.5mm}r{1mm}){2-5}
    \cmidrule(l{0.5mm}){6-9}

    \texttt{HippoRAG}
    \scalebox{0.68}{(\mycite{gutierrez2024hipporag}~\textit{NeurIPS'24})}
    & 21.\scalebox{0.75}{3}\litA
    & 44.\scalebox{0.75}{4}\litA
    & 43.\scalebox{0.75}{0}\litA
    & 55.\scalebox{0.75}{3}\litA
    & 40.\scalebox{0.75}{0}\litA
    & 53.\scalebox{0.75}{8}\litA
    & 42.\scalebox{0.75}{7}\litA
    & 55.\scalebox{0.75}{9}\litA \\

    \texttt{HippoRAG} \texttt{2}
    \scalebox{0.68}{(\mycite{gutierrez2025rag}~\textit{ICML'25})}
    & 35.\scalebox{0.75}{3}
    & 53.\scalebox{0.75}{3}
    & 25.\scalebox{0.75}{7}
    & 42.\scalebox{0.75}{87}
    & 39.\scalebox{0.75}{6}
    & 61.\scalebox{0.75}{7}
    & 46.\scalebox{0.75}{1}
    & 57.\scalebox{0.75}{18} \\

    \texttt{PropRAG}
    \scalebox{0.68}{(\mycite{wang2025proprag}~\textit{EMNLP'25})}
    & \NA
    & \textbf{77.\scalebox{0.75}{9}\litB}
    & \NA
    & \textbf{62.\scalebox{0.75}{5}\litB}
    & \NA
    & 56.\scalebox{0.75}{2}\litB
    & \NA
    & 56.\scalebox{0.75}{4}\litB \\

    \cmidrule(r{1mm}){1-1}
    \cmidrule(l{0.5mm}r{1mm}){2-5}
    \cmidrule(l{0.5mm}){6-9}

    \textbf{\method{} (ours)\zsmark}
    & \textbf{49.\scalebox{0.75}{5}\zsmark}
    & 63.\scalebox{0.75}{0}\zsmark
    & \textbf{46.\scalebox{0.75}{2}\zsmark}
    & 54.\scalebox{0.75}{21}\zsmark
    & \textbf{54.\scalebox{0.75}{9}\zsmark}
    & \textbf{64.\scalebox{0.75}{6}\zsmark}
    & \textbf{53.\scalebox{0.75}{0}\zsmark}
    & \textbf{66.\scalebox{0.75}{06}\zsmark} \\

    \bottomrule
    \end{tabular}%
  }

  {\footnotesize\noindent
  \textsuperscript{\ensuremath{\dagger}}
  Values are from the passage-level retrieval and QA evaluation in
  \mycite{gutierrez2025rag}.
  \textsuperscript{\ensuremath{\ddagger}}
  Values are from the Recall@5 and F1 tables in
  \mycite{wang2025proprag}; Recall@2 and EM are not reported there.
  \textsuperscript{\ensuremath{\star}}
  Values are from the reproduced single-step retrieval table in
  \mycite{cooprag2025}; PopQA and QA EM/F1 are not reported there.
  \textsuperscript{\ensuremath{\circledast}}
  Zero-shot denotes no \texttt{NQ}/\texttt{PopQA}-specific gradient updates,
  tuning, or self-evolution.}
  \vspace{-1cm}

  \label{tab:res_retrieve_nq_popqa}%
\end{table}We provide the full coordinate-improvement result, the surrogate reward bias bound, and the analysis of single-sided update bottlenecks in Appendix~\ref{app:theory_self_evolution}. In addition, Proposition~\ref{prop:method_reader_theory}(iii) shows that although each writer update changes the graph structure in self-evolving memory, the reader output does not oscillate arbitrarily with graph evolution. We provide detailed training, inference, memory, and selector-regularizer complexity analyses in Appendix~\ref{app:complexity2}.

\section{Experiments}
\label{sec:experiments}

This section presents an experimental evaluation centered around four research questions (RQs).
\textit{\textbf{RQ1}}: whether \method{} can bring consistent benefits in tasks such as multi-hop QA and open-domain transfer;

\textit{\textbf{RQ2}}: whether \method{} is capable of handling long-term conversation history, knowledge updates, and memory hallucination;

\textit{\textbf{RQ3}}: whether the writer-reader closed loop truly yields self-evolution benefits;

\textit{\textbf{RQ4}}: further analysis of where and how the performance gains come from specific designs.

%
%

\begin{table}
\renewcommand{\arraystretch}{0.85}
\setlength{\tabcolsep}{3.4pt}
\centering

\vspace{-0.8cm}
\caption{Performance of representative memory systems on \texttt{LongMemEval}. We report accuracy (\%) on single-session user (SS-U), single-session assistant (SS-A), merged single-session recall (SSR), single-session preference (SS-P), knowledge update (KU), temporal reasoning (TR), multi-session reasoning (MS), and overall performance. SSR is the sample-weighted average of SS-U and SS-A. Overall is the sample-weighted average over the six original task categories. Best numerical results are in \textbf{bold} and runner-ups are \underline{underlined}. The darker the cell, the better. Results from different reporting protocols are shown together for comparison and should not be treated as a single strict leaderboard.}
\label{tab:longmemeval_baselines}
\vspace{0.2cm}

\resizebox{0.9\linewidth}{!}{
\begin{tabular}{lcccccccc}
\toprule

\textbf{Dataset}
& \multicolumn{8}{c}{\texttt{LongMemEval-S / LongMemEval}} \\

\cmidrule(r{1mm}){1-1}
\cmidrule(l{0.5mm}r{1mm}){2-9}

\textbf{Method}
& SS-U
& SS-A
& SSR
& SS-P
& KU
& TR
& MS
& Overall \\

\cmidrule(r{1mm}){1-1}
\cmidrule(l{0.5mm}r{1mm}){2-9}


\multicolumn{9}{l}{
\textit{Other baselines: unified protocol in TiMem
(GPT-4o-mini, LLM-judge accuracy)}
} \\

\texttt{MemoryBank}
\scalebox{0.68}{(\mycite{zhong2024memorybank})}
& \cellcolor{myblue!2.0}{50.\scalebox{0.75}{0}}
& \cellcolor{myblue!2.0}{9.\scalebox{0.75}{8}}
& \cellcolor{myblue!2.0}{32.\scalebox{0.75}{1}}
& \cellcolor{myblue!2.0}{0.\scalebox{0.75}{0}}
& \cellcolor{myblue!2.0}{21.\scalebox{0.75}{8}}
& \cellcolor{myblue!2.0}{0.\scalebox{0.75}{0}}
& \cellcolor{myblue!2.0}{0.\scalebox{0.75}{0}}
& \cellcolor{myblue!2.0}{11.\scalebox{0.75}{5}} \\

\texttt{A-MEM}
\scalebox{0.68}{(\mycite{xu2025amem})}
& \cellcolor{myblue!49.5}{82.\scalebox{0.75}{9}}
& \cellcolor{myblue!62.7}{87.\scalebox{0.75}{5}}
& \cellcolor{myblue!57.6}{84.\scalebox{0.75}{9}}
& \cellcolor{myblue!31.7}{39.\scalebox{0.75}{3}}
& \cellcolor{myblue!53.1}{72.\scalebox{0.75}{8}}
& \cellcolor{myblue!33.0}{36.\scalebox{0.75}{1}}
& \cellcolor{myblue!40.7}{40.\scalebox{0.75}{3}}
& \cellcolor{myblue!45.7}{55.\scalebox{0.75}{4}} \\

\texttt{Mem0}
\scalebox{0.68}{(\mycite{chhikara2025mem0})}
& \cellcolor{myblue!66.0}{94.\scalebox{0.75}{3}}
& \cellcolor{myblue!34.8}{51.\scalebox{0.75}{8}}
& \cellcolor{myblue!47.6}{75.\scalebox{0.75}{4}}
& \cellcolor{myblue!39.7}{50.\scalebox{0.75}{0}}
& \cellcolor{myblue!59.0}{78.\scalebox{0.75}{7}}
& \cellcolor{myblue!44.3}{49.\scalebox{0.75}{2}}
& \cellcolor{myblue!65.6}{66.\scalebox{0.75}{2}}
& \cellcolor{myblue!55.3}{65.\scalebox{0.75}{0}} \\

\texttt{MemoryOS}
\scalebox{0.68}{(\mycite{kang2025memoryos})}
& \cellcolor{myblue!46.9}{81.\scalebox{0.75}{1}}
& \cellcolor{myblue!55.4}{78.\scalebox{0.75}{2}}
& \cellcolor{myblue!52.2}{79.\scalebox{0.75}{8}}
& \cellcolor{myblue!40.7}{51.\scalebox{0.75}{3}}
& \cellcolor{myblue!36.4}{56.\scalebox{0.75}{1}}
& \cellcolor{myblue!47.9}{53.\scalebox{0.75}{4}}
& \cellcolor{myblue!45.0}{44.\scalebox{0.75}{8}}
& \cellcolor{myblue!48.4}{58.\scalebox{0.75}{1}} \\

\texttt{TiMem}
\scalebox{0.68}{(\mycite{li2026timem})}
& \cellcolor{myblue!68.0}{95.\scalebox{0.75}{7}}
& \cellcolor{myblue!58.4}{82.\scalebox{0.75}{1}}
& \cellcolor{myblue!62.6}{89.\scalebox{0.75}{7}}
& \cellcolor{myblue!49.8}{63.\scalebox{0.75}{3}}
& \cellcolor{myblue!66.5}{
    \underline{~86.\scalebox{0.75}{2}~}}
& \cellcolor{myblue!60.8}{68.\scalebox{0.75}{4}}
& \cellcolor{myblue!70.0}{
    \textbf{70.\scalebox{0.75}{8}}}
& \cellcolor{myblue!67.1}{76.\scalebox{0.75}{9}} \\


\cmidrule(r{1mm}){1-1}
\cmidrule(l{0.5mm}r{1mm}){2-9}

\multicolumn{9}{l}{
\textit{Other baselines: MemOS evaluation suite
(short-answer prompt)}
} \\

\texttt{MIRIX}
\scalebox{0.68}{(\mycite{wang2025mirix})}
& \cellcolor{myblue!34.9}{72.\scalebox{0.75}{8}}
& \cellcolor{myblue!44.0}{63.\scalebox{0.75}{6}}
& \cellcolor{myblue!40.6}{68.\scalebox{0.75}{8}}
& \cellcolor{myblue!42.2}{53.\scalebox{0.75}{3}}
& \cellcolor{myblue!32.8}{52.\scalebox{0.75}{6}}
& \cellcolor{myblue!24.0}{25.\scalebox{0.75}{6}}
& \cellcolor{myblue!30.9}{30.\scalebox{0.75}{1}}
& \cellcolor{myblue!33.9}{43.\scalebox{0.75}{5}} \\

\texttt{Mem0}
\scalebox{0.68}{(\mycite{chhikara2025mem0})}
& \cellcolor{myblue!49.5}{82.\scalebox{0.75}{9}}
& \cellcolor{myblue!15.3}{26.\scalebox{0.75}{8}}
& \cellcolor{myblue!29.2}{57.\scalebox{0.75}{9}}
& \cellcolor{myblue!69.9}{
    \underline{~90.\scalebox{0.75}{0}~}}
& \cellcolor{myblue!47.0}{66.\scalebox{0.75}{7}}
& \cellcolor{myblue!64.1}{72.\scalebox{0.75}{2}}
& \cellcolor{myblue!62.6}{63.\scalebox{0.75}{1}}
& \cellcolor{myblue!56.7}{66.\scalebox{0.75}{4}} \\

\texttt{Supermemory}\footnotemark[2]
& \cellcolor{myblue!53.5}{85.\scalebox{0.75}{7}}
& \cellcolor{myblue!40.3}{58.\scalebox{0.75}{9}}
& \cellcolor{myblue!45.9}{73.\scalebox{0.75}{8}}
& \cellcolor{myblue!69.9}{
    \underline{~90.\scalebox{0.75}{0}~}}
& \cellcolor{myblue!35.3}{55.\scalebox{0.75}{1}}
& \cellcolor{myblue!40.2}{44.\scalebox{0.75}{4}}
& \cellcolor{myblue!52.5}{52.\scalebox{0.75}{6}}
& \cellcolor{myblue!48.7}{58.\scalebox{0.75}{4}} \\

\texttt{MemU}\footnotemark[3]
& \cellcolor{myblue!26.7}{67.\scalebox{0.75}{1}}
& \cellcolor{myblue!9.7}{19.\scalebox{0.75}{6}}
& \cellcolor{myblue!16.6}{46.\scalebox{0.75}{0}}
& \cellcolor{myblue!59.9}{76.\scalebox{0.75}{7}}
& \cellcolor{myblue!21.2}{41.\scalebox{0.75}{0}}
& \cellcolor{myblue!16.9}{17.\scalebox{0.75}{3}}
& \cellcolor{myblue!42.4}{42.\scalebox{0.75}{1}}
& \cellcolor{myblue!28.8}{38.\scalebox{0.75}{4}} \\


\cmidrule(r{1mm}){1-1}
\cmidrule(l{0.5mm}r{1mm}){2-9}

\multicolumn{9}{l}{
\textit{Matched baseline--\method{} comparisons}
} \\


\texttt{MemOS}
\scalebox{0.68}{(\mycite{li2025memos})}
& \cellcolor{myblue!65.1}{93.\scalebox{0.75}{7}}
& \cellcolor{myblue!47.4}{67.\scalebox{0.75}{9}}
& \cellcolor{myblue!54.7}{82.\scalebox{0.75}{2}}
& \cellcolor{myblue!40.3}{50.\scalebox{0.75}{7}}
& \cellcolor{myblue!57.0}{76.\scalebox{0.75}{7}}
& \cellcolor{myblue!58.0}{65.\scalebox{0.75}{1}}
& \cellcolor{myblue!58.5}{58.\scalebox{0.75}{8}}
& \cellcolor{myblue!58.9}{68.\scalebox{0.75}{7}} \\

\texttt{MemOS} + \method{}
& \cellcolor{myblue!70.0}{
    \textbf{97.\scalebox{0.75}{1}}}
& \cellcolor{myblue!68.2}{
    \underline{~94.\scalebox{0.75}{6}~}}
& \cellcolor{myblue!69.3}{
    \underline{~96.\scalebox{0.75}{0}~}}
& \cellcolor{myblue!64.9}{83.\scalebox{0.75}{3}}
& \cellcolor{myblue!50.8}{70.\scalebox{0.75}{5}}
& \cellcolor{myblue!63.4}{71.\scalebox{0.75}{4}}
& \cellcolor{myblue!65.6}{66.\scalebox{0.75}{2}}
& \cellcolor{myblue!67.0}{76.\scalebox{0.75}{8}} \\

\texttt{MemOS} + \method{} (1 round)
& \cellcolor{myblue!69.1}{
    \underline{~96.\scalebox{0.75}{5}~}}
& \cellcolor{myblue!70.0}{
    \textbf{96.\scalebox{0.75}{9}}}
& \cellcolor{myblue!70.0}{
    \textbf{96.\scalebox{0.75}{7}}}
& \cellcolor{myblue!66.3}{85.\scalebox{0.75}{2}}
& \cellcolor{myblue!53.7}{73.\scalebox{0.75}{4}}
& \cellcolor{myblue!63.8}{71.\scalebox{0.75}{9}}
& \cellcolor{myblue!67.5}{68.\scalebox{0.75}{2}}
& \cellcolor{myblue!68.4}{78.\scalebox{0.75}{2}} \\

\addlinespace[1.5pt]


\texttt{memobase}\footnotemark[1]
& \cellcolor{myblue!63.8}{92.\scalebox{0.75}{8}}
& \cellcolor{myblue!12.5}{23.\scalebox{0.75}{2}}
& \cellcolor{myblue!33.4}{61.\scalebox{0.75}{9}}
& \cellcolor{myblue!62.4}{80.\scalebox{0.75}{0}}
& \cellcolor{myblue!70.0}{
    \textbf{89.\scalebox{0.75}{7}}}
& \cellcolor{myblue!67.2}{75.\scalebox{0.75}{9}}
& \cellcolor{myblue!66.3}{66.\scalebox{0.75}{9}}
& \cellcolor{myblue!62.6}{72.\scalebox{0.75}{4}} \\

\texttt{Memobase} + \method{}
& \cellcolor{myblue!68.0}{95.\scalebox{0.75}{7}}
& \cellcolor{myblue!51.3}{72.\scalebox{0.75}{9}}
& \cellcolor{myblue!58.3}{85.\scalebox{0.75}{6}}
& \cellcolor{myblue!67.7}{87.\scalebox{0.75}{1}}
& \cellcolor{myblue!63.4}{83.\scalebox{0.75}{1}}
& \cellcolor{myblue!70.0}{
    \textbf{79.\scalebox{0.75}{1}}}
& \cellcolor{myblue!68.9}{
    \underline{~69.\scalebox{0.75}{7}~}}
& \cellcolor{myblue!69.5}{
    \underline{~79.\scalebox{0.75}{3}~}} \\

\texttt{Memobase} + \method{} (1 round)
& \cellcolor{myblue!67.7}{95.\scalebox{0.75}{5}}
& \cellcolor{myblue!52.7}{74.\scalebox{0.75}{8}}
& \cellcolor{myblue!59.1}{86.\scalebox{0.75}{3}}
& \cellcolor{myblue!70.0}{
    \textbf{90.\scalebox{0.75}{1}}}
& \cellcolor{myblue!66.3}{86.\scalebox{0.75}{0}}
& \cellcolor{myblue!69.6}{
    \underline{~78.\scalebox{0.75}{6}~}}
& \cellcolor{myblue!68.2}{68.\scalebox{0.75}{9}}
& \cellcolor{myblue!70.0}{
    \textbf{79.\scalebox{0.75}{8}}} \\

\addlinespace[1.5pt]


\texttt{Zep}
\scalebox{0.68}{(\mycite{rasmussen2025zep})}
& \cellcolor{myblue!63.9}{92.\scalebox{0.75}{9}}
& \cellcolor{myblue!52.9}{75.\scalebox{0.75}{0}}
& \cellcolor{myblue!57.6}{84.\scalebox{0.75}{9}}
& \cellcolor{myblue!42.2}{53.\scalebox{0.75}{3}}
& \cellcolor{myblue!54.7}{74.\scalebox{0.75}{4}}
& \cellcolor{myblue!48.5}{54.\scalebox{0.75}{1}}
& \cellcolor{myblue!47.5}{47.\scalebox{0.75}{4}}
& \cellcolor{myblue!53.5}{63.\scalebox{0.75}{2}} \\

\texttt{Zep} + \method{}
& \cellcolor{myblue!64.9}{93.\scalebox{0.75}{6}}
& \cellcolor{myblue!55.2}{78.\scalebox{0.75}{0}}
& \cellcolor{myblue!59.5}{86.\scalebox{0.75}{7}}
& \cellcolor{myblue!52.6}{67.\scalebox{0.75}{1}}
& \cellcolor{myblue!56.9}{76.\scalebox{0.75}{6}}
& \cellcolor{myblue!58.9}{66.\scalebox{0.75}{2}}
& \cellcolor{myblue!55.2}{55.\scalebox{0.75}{4}}
& \cellcolor{myblue!60.4}{70.\scalebox{0.75}{2}} \\

\texttt{Zep} + \method{} (1 round)
& \cellcolor{myblue!65.8}{94.\scalebox{0.75}{2}}
& \cellcolor{myblue!57.4}{80.\scalebox{0.75}{8}}
& \cellcolor{myblue!61.1}{88.\scalebox{0.75}{2}}
& \cellcolor{myblue!54.3}{69.\scalebox{0.75}{3}}
& \cellcolor{myblue!58.2}{77.\scalebox{0.75}{9}}
& \cellcolor{myblue!61.3}{69.\scalebox{0.75}{0}}
& \cellcolor{myblue!59.3}{59.\scalebox{0.75}{7}}
& \cellcolor{myblue!63.0}{72.\scalebox{0.75}{8}} \\

\bottomrule
\vspace{-2.5cm}
\end{tabular}%

}

\end{table}

\footnotetext[1]{%
\texttt{memobase}: \url{https://github.com/memodb-io/memobase}.}

\footnotetext[2]{%
\texttt{Supermemory}: \url{https://github.com/supermemoryai/supermemory}.}

\footnotetext[3]{%
\texttt{MemU}: \url{https://github.com/NevaMind-AI/memU}.}

\paragraph{Datasets.} \label{subsec:exp_setup}We evaluate \method{} on eight benchmarks spanning three evaluation scenarios. The first category consists of general QA benchmarks and three multi-hop QA benchmarks, including NQ, PopQA, HotpotQA, 2WikiMultiHopQA, and MuSiQue, used to examine whether the system can recover bridge entities across documents and combine evidence and reasoning paths. The second category focuses on a practical e-commerce application scenario, using a Review-Based Question Answering Task: AmazonQA, to assess its value in real e-commerce applications with real noisy reviews. The third category comprises long-term agent memory datasets, including LongMemEval and HaluMem, used to test information extraction from long interaction histories, multi-session reasoning, temporal reasoning, knowledge updating, abstention, and operation-level hallucination. Table~\ref{tab:dataset_grid} summarizes the details of each dataset. Further details on baselines and metrics can be found in Appendix~\ref{subsec:Baselines_Metrics}.

\subsection{End-to-End Effectiveness}
\label{subsec:end_to_end}

\paragraph{Multi-hop Question Answering} Table~\ref{tab:res_qa} reports the main results on general QA benchmarks and three multi-hop QA benchmarks. Table~\ref{tab:res_retrieve_multihop} reports the results of retrieval performance on multi-hop QA benchmarks. It is worth mentioning that even when we directly test on NQ and PopQA using a model trained only on MuSiQue, HotpotQA, and 2WikiMultiHopQA, SAGE remains competitive on NQ
and shows a clearer advantage on PopQA; detailed matched results are reported
in Table~\ref{tab:res_retrieve_nq_popqa}. Full protocol is given in Appendix~\ref{app:matched_protocol_details}.

\definecolor{myblue}{RGB}{56,132,255}
\providecommand{\placeholdercell}{}
\renewcommand{\placeholdercell}{\cellcolor{myblue!2}{--}}
\providecommand{\placeholderrank}{}
\renewcommand{\placeholderrank}{\cellcolor{myred!2}{--}}

\begin{table}[t]
\renewcommand{\arraystretch}{0.9}
\setlength{\tabcolsep}{4pt}
  \centering
  \vspace{-0.3cm}
  \caption{Results of multi-hop question answering (QA) performance. We report Exact Match (EM) and F1 score, both reported as percentages (\%). Best results are in \textbf{bold} and runner-ups are \underline{underlined}. The darker the cell, the better.}
  \vspace{0.1cm}
  \resizebox{\textwidth}{!}{%
    \begin{tabular}{lccccccr}
    \toprule
    \textbf{Dataset} & \multicolumn{2}{c}{\texttt{HotpotQA}}  & \multicolumn{2}{c}{\texttt{MuSiQue}}   & \multicolumn{2}{c}{\texttt{2WikiMultiHopQA}} & \multirow{2}[2]{*}{\textbf{Avg. Rank}} \\
\cmidrule(r{1mm}){1-1}  \cmidrule(l{0.5mm}r{1mm}){2-3} \cmidrule(l{0.5mm}r{1mm}){4-5} \cmidrule(l{0.5mm}r{1mm}){6-7} \textbf{Method} & ~~~EM~~ & ~~F1~~ & ~~~EM~~ & ~~F1~~ & ~~~EM~~ & ~~F1~~ & \\
    \cmidrule(r{1mm}){1-1}  \cmidrule(l{0.5mm}r{1mm}){2-3} \cmidrule(l{0.5mm}r{1mm}){4-5} \cmidrule(l{0.5mm}r{1mm}){6-7} \cmidrule(l{0.5mm}){8-8}
    \texttt{BM25}
\scalebox{0.68}{(\mycite{robertson1994some}~\textit{SIGIR'94})} & \cellcolor{myblue!19}{40.\scalebox{0.75}{0}} & \cellcolor{myblue!12.25}{53.\scalebox{0.75}{2}} & \cellcolor{myblue!18.27}{19.\scalebox{0.75}{5}} & \cellcolor{myblue!6.32}{23.\scalebox{0.75}{6}} & \cellcolor{myblue!32.07}{46.\scalebox{0.75}{9}} & \cellcolor{myblue!35.31}{57.\scalebox{0.75}{9}} & \cellcolor{myred!20.89}{15.\scalebox{0.75}{5}} \\
    \cmidrule(r{1mm}){1-1}  \cmidrule(l{0.5mm}r{1mm}){2-3} \cmidrule(l{0.5mm}r{1mm}){4-5} \cmidrule(l{0.5mm}r{1mm}){6-7} \cmidrule(l{0.5mm}){8-8}
    \texttt{Contriever} \scalebox{0.68}{(\mycite{izacard2022unsupervised}~\textit{TMLR'22})} & \cellcolor{myblue!6.25}{34.\scalebox{0.75}{9}} & \cellcolor{myblue!8.07}{51.\scalebox{0.75}{2}} & \cellcolor{myblue!9.73}{16.\scalebox{0.75}{3}} & \cellcolor{myblue!2}{21.\scalebox{0.75}{6}} & \cellcolor{myblue!2}{24.\scalebox{0.75}{3}} & \cellcolor{myblue!2}{33.\scalebox{0.75}{9}} & \cellcolor{myred!2}{20.\scalebox{0.75}{5}} \\
 \texttt{GTR} \scalebox{0.68}{(\mycite{ni2022large}~\textit{EMNLP'22})} & \cellcolor{myblue!3.5}{33.\scalebox{0.75}{8}} & \cellcolor{myblue!9.53}{51.\scalebox{0.75}{9}} & \cellcolor{myblue!6.53}{15.\scalebox{0.75}{1}} & \cellcolor{myblue!9.77}{25.\scalebox{0.75}{2}} & \cellcolor{myblue!14.51}{33.\scalebox{0.75}{7}} & \cellcolor{myblue!13.93}{42.\scalebox{0.75}{5}} & \cellcolor{myred!5.15}{19.\scalebox{0.75}{7}} \\
 \texttt{ColBERTv2} \scalebox{0.68}{(\mycite{santhanam2022colbertv2}~\textit{NAACL'22})} & \cellcolor{myblue!27.5}{43.\scalebox{0.75}{4}} & \cellcolor{myblue!21.67}{57.\scalebox{0.75}{7}} & \cellcolor{myblue!7.6}{15.\scalebox{0.75}{5}} & \cellcolor{myblue!12.36}{26.\scalebox{0.75}{4}} & \cellcolor{myblue!14.11}{33.\scalebox{0.75}{4}} & \cellcolor{myblue!15.04}{43.\scalebox{0.75}{3}} & \cellcolor{myred!13.96}{17.\scalebox{0.75}{3}} \\
 \texttt{RAPTOR} \scalebox{0.68}{(\mycite{sarthi2024raptor}~\textit{ICLR'24})} & \cellcolor{myblue!39.5}{48.\scalebox{0.75}{2}} & \cellcolor{myblue!24.81}{59.\scalebox{0.75}{2}} & \cellcolor{myblue!13.2}{17.\scalebox{0.75}{6}} & \cellcolor{myblue!17.76}{28.\scalebox{0.75}{9}} & \cellcolor{myblue!10.38}{30.\scalebox{0.75}{6}} & \cellcolor{myblue!13.24}{42.\scalebox{0.75}{0}} & \cellcolor{myred!19.94}{15.\scalebox{0.75}{8}} \\
    \cmidrule(r{1mm}){1-1}  \cmidrule(l{0.5mm}r{1mm}){2-3} \cmidrule(l{0.5mm}r{1mm}){4-5} \cmidrule(l{0.5mm}r{1mm}){6-7} \cmidrule(l{0.5mm}){8-8}
    \texttt{GraphRAG} \scalebox{0.68}{(\mycite{edge2024local}~\textit{arXiv'24})} & \cellcolor{myblue!7.25}{35.\scalebox{0.75}{3}} & \cellcolor{myblue!15.18}{54.\scalebox{0.75}{6}} & \cellcolor{myblue!2}{13.\scalebox{0.75}{4}} & \cellcolor{myblue!19.05}{29.\scalebox{0.75}{5}} & \cellcolor{myblue!7.32}{28.\scalebox{0.75}{3}} & \cellcolor{myblue!20.04}{46.\scalebox{0.75}{9}} & \cellcolor{myred!11.44}{18.\scalebox{0.75}{0}} \\
 \texttt{G-Retriever} \scalebox{0.68}{(\mycite{he2024g}~\textit{NeurIPS'24})} & \cellcolor{myblue!2}{33.\scalebox{0.75}{2}} & \cellcolor{myblue!6.18}{50.\scalebox{0.75}{3}} & \cellcolor{myblue!14.27}{18.\scalebox{0.75}{0}} & \cellcolor{myblue!11.28}{25.\scalebox{0.75}{9}} & \cellcolor{myblue!25.95}{42.\scalebox{0.75}{3}} & \cellcolor{myblue!18.24}{45.\scalebox{0.75}{6}} & \cellcolor{myred!10.81}{18.\scalebox{0.75}{2}} \\
 \texttt{LightRAG} \scalebox{0.68}{(\mycite{guo2024lightrag}~\textit{arXiv'24})} & \cellcolor{myblue!11}{36.\scalebox{0.75}{8}} & \cellcolor{myblue!2}{48.\scalebox{0.75}{3}} & \cellcolor{myblue!14.53}{18.\scalebox{0.75}{1}} & \cellcolor{myblue!14.74}{27.\scalebox{0.75}{5}} & \cellcolor{myblue!29.68}{45.\scalebox{0.75}{1}} & \cellcolor{myblue!23.65}{49.\scalebox{0.75}{5}} & \cellcolor{myred!16.48}{16.\scalebox{0.75}{7}} \\
 \texttt{HippoRAG} \scalebox{0.68}{(\mycite{gutierrez2024hipporag}~\textit{NeurIPS'24})} & \cellcolor{myblue!23.5}{41.\scalebox{0.75}{8}} & \cellcolor{myblue!16.02}{55.\scalebox{0.75}{0}} & \cellcolor{myblue!17.47}{19.\scalebox{0.75}{2}} & \cellcolor{myblue!19.7}{29.\scalebox{0.75}{8}} & \cellcolor{myblue!31.68}{46.\scalebox{0.75}{6}} & \cellcolor{myblue!37.53}{59.\scalebox{0.75}{5}} & \cellcolor{myred!27.19}{13.\scalebox{0.75}{8}} \\
 \texttt{HippoRAG} \texttt{2} \scalebox{0.68}{(\mycite{gutierrez2025rag}~\textit{ICML'25})} & \cellcolor{myblue!62.25}{\underline{~57.\scalebox{0.75}{3}~}} & \cellcolor{myblue!46.57}{69.\scalebox{0.75}{6}} & \cellcolor{myblue!56.67}{33.\scalebox{0.75}{9}} & \cellcolor{myblue!42.8}{40.\scalebox{0.75}{5}} & \cellcolor{myblue!70}{\textbf{75.\scalebox{0.75}{4}}} & \cellcolor{myblue!63.48}{78.\scalebox{0.75}{2}} & \cellcolor{myred!61.81}{4.\scalebox{0.75}{7}} \\
 \texttt{SubgraphRAG} \scalebox{0.68}{(\mycite{li2025simple}~\textit{ICLR'25})} & \cellcolor{myblue!50.5}{52.\scalebox{0.75}{6}} & \cellcolor{myblue!38.62}{65.\scalebox{0.75}{8}} & \cellcolor{myblue!54.27}{33.\scalebox{0.75}{0}} & \cellcolor{myblue!39.99}{39.\scalebox{0.75}{2}} & \cellcolor{myblue!64.14}{71.\scalebox{0.75}{0}} & \cellcolor{myblue!60.7}{76.\scalebox{0.75}{2}} & \cellcolor{myred!51.11}{7.\scalebox{0.75}{5}} \\
 \texttt{PropRAG} \scalebox{0.68}{(\mycite{wang2025proprag}~\textit{EMNLP'25})} & \cellcolor{myblue!62}{57.\scalebox{0.75}{2}} & \cellcolor{myblue!47.61}{70.\scalebox{0.75}{1}} & \cellcolor{myblue!70}{\textbf{38.\scalebox{0.75}{9}}} & \cellcolor{myblue!44.53}{41.\scalebox{0.75}{3}} & \cellcolor{myblue!66.81}{73.\scalebox{0.75}{0}} & \cellcolor{myblue!65.84}{79.\scalebox{0.75}{9}} & \cellcolor{myred!64.96}{3.\scalebox{0.75}{8}} \\
 \texttt{GFM-RAG} \scalebox{0.68}{(\mycite{luo2025gfm}~\textit{NeurIPS'25})} & \cellcolor{myblue!48}{51.\scalebox{0.75}{6}} & \cellcolor{myblue!40.92}{66.\scalebox{0.75}{9}} & \cellcolor{myblue!46.8}{30.\scalebox{0.75}{2}} & \cellcolor{myblue!42.58}{40.\scalebox{0.75}{4}} & \cellcolor{myblue!62.55}{69.\scalebox{0.75}{8}} & \cellcolor{myblue!62.78}{77.\scalebox{0.75}{7}} & \cellcolor{myred!49.85}{7.\scalebox{0.75}{8}} \\
    \cmidrule(r{1mm}){1-1}  \cmidrule(l{0.5mm}r{1mm}){2-3} \cmidrule(l{0.5mm}r{1mm}){4-5} \cmidrule(l{0.5mm}r{1mm}){6-7} \cmidrule(l{0.5mm}){8-8}
    \texttt{FLARE} \scalebox{0.68}{(\mycite{jiang2023active}~\textit{EMNLP'23})} & \cellcolor{myblue!40.75}{48.\scalebox{0.75}{7}} & \cellcolor{myblue!27.74}{60.\scalebox{0.75}{6}} & \cellcolor{myblue!9.47}{16.\scalebox{0.75}{2}} & \cellcolor{myblue!16.68}{28.\scalebox{0.75}{4}} & \cellcolor{myblue!31.81}{46.\scalebox{0.75}{7}} & \cellcolor{myblue!45.71}{65.\scalebox{0.75}{4}} & \cellcolor{myred!31.59}{12.\scalebox{0.75}{7}} \\
 \texttt{Adaptive-RAG} \scalebox{0.68}{(\mycite{jeong2024adaptive}~\textit{NAACL'24})} & \cellcolor{myblue!32.75}{45.\scalebox{0.75}{5}} & \cellcolor{myblue!25.64}{59.\scalebox{0.75}{6}} & \cellcolor{myblue!3.07}{13.\scalebox{0.75}{8}} & \cellcolor{myblue!10.63}{25.\scalebox{0.75}{6}} & \cellcolor{myblue!34.74}{48.\scalebox{0.75}{9}} & \cellcolor{myblue!42.11}{62.\scalebox{0.75}{8}} & \cellcolor{myred!25.61}{14.\scalebox{0.75}{3}} \\
 \texttt{ColBERTv2} + \texttt{IRCoT} \scalebox{0.68}{(\mycite{trivedi2023interleaving}~\textit{ACL'23})} & \cellcolor{myblue!32.75}{45.\scalebox{0.75}{5}} & \cellcolor{myblue!23.13}{58.\scalebox{0.75}{4}} & \cellcolor{myblue!17.2}{19.\scalebox{0.75}{1}} & \cellcolor{myblue!21.21}{30.\scalebox{0.75}{5}} & \cellcolor{myblue!16.77}{35.\scalebox{0.75}{4}} & \cellcolor{myblue!17.54}{45.\scalebox{0.75}{1}} & \cellcolor{myred!24.98}{14.\scalebox{0.75}{4}} \\
 \texttt{HippoRAG} + \texttt{IRCoT} \scalebox{0.68}{(\mycite{trivedi2023interleaving}~\textit{ACL'23})} & \cellcolor{myblue!33.25}{45.\scalebox{0.75}{7}} & \cellcolor{myblue!24.81}{59.\scalebox{0.75}{2}} & \cellcolor{myblue!24.67}{21.\scalebox{0.75}{9}} & \cellcolor{myblue!27.26}{33.\scalebox{0.75}{3}} & \cellcolor{myblue!33.14}{47.\scalebox{0.75}{7}} & \cellcolor{myblue!41.97}{62.\scalebox{0.75}{7}} & \cellcolor{myred!36.94}{11.\scalebox{0.75}{3}} \\
 \texttt{GFM-RAG} + \texttt{IRCoT} \scalebox{0.68}{(\mycite{trivedi2023interleaving}~\textit{ACL'23})} & \cellcolor{myblue!59}{56.\scalebox{0.75}{0}} & \cellcolor{myblue!51.17}{71.\scalebox{0.75}{8}} & \cellcolor{myblue!63.87}{\underline{~36.\scalebox{0.75}{6}~}} & \cellcolor{myblue!61.58}{49.\scalebox{0.75}{2}} & \cellcolor{myblue!66.14}{72.\scalebox{0.75}{5}} & \cellcolor{myblue!67.09}{\underline{~80.\scalebox{0.75}{8}~}} & \cellcolor{myred!65.59}{3.\scalebox{0.75}{7}} \\
    \cmidrule(r{1mm}){1-1}  \cmidrule(l{0.5mm}r{1mm}){2-3} \cmidrule(l{0.5mm}r{1mm}){4-5} \cmidrule(l{0.5mm}r{1mm}){6-7} \cmidrule(l{0.5mm}){8-8}
    \textbf{\method{}} (ours) & \cellcolor{myblue!47.25}{51.\scalebox{0.75}{3}} & \cellcolor{myblue!53.05}{72.\scalebox{0.75}{7}} & \cellcolor{myblue!40.93}{28.\scalebox{0.75}{0}} & \cellcolor{myblue!57.48}{47.\scalebox{0.75}{3}} & \cellcolor{myblue!55.5}{64.\scalebox{0.75}{5}} & \cellcolor{myblue!59.45}{75.\scalebox{0.75}{3}} & \cellcolor{myred!51.11}{7.\scalebox{0.75}{5}} \\
    \textbf{ours +1 round } & \cellcolor{myblue!70}{\textbf{60.\scalebox{0.75}{4}}} & \cellcolor{myblue!54.52}{73.\scalebox{0.75}{4}} & \cellcolor{myblue!52.13}{32.\scalebox{0.75}{2}} & \cellcolor{myblue!68.27}{52.\scalebox{0.75}{3}} & \cellcolor{myblue!65.08}{71.\scalebox{0.75}{7}} & \cellcolor{myblue!65.28}{79.\scalebox{0.75}{5}} & \cellcolor{myred!63.7}{4.\scalebox{0.75}{2}} \\
    \textbf{ours +2 round } & \cellcolor{myblue!59.25}{56.\scalebox{0.75}{1}} & \cellcolor{myblue!70}{\textbf{80.\scalebox{0.75}{8}}} & \cellcolor{myblue!57.47}{34.\scalebox{0.75}{2}} & \cellcolor{myblue!70}{\textbf{53.\scalebox{0.75}{1}}} & \cellcolor{myblue!68.14}{\underline{~74.\scalebox{0.75}{0}~}} & \cellcolor{myblue!66.53}{80.\scalebox{0.75}{4}} & \cellcolor{myred!70}{\textbf{2.\scalebox{0.75}{5}}} \\
 \textbf{ours + \texttt{IRCoT} } & \cellcolor{myblue!49.75}{52.\scalebox{0.75}{3}} & \cellcolor{myblue!57.86}{\underline{~75.\scalebox{0.75}{0}~}} & \cellcolor{myblue!60.13}{35.\scalebox{0.75}{2}} & \cellcolor{myblue!68.92}{\underline{~52.\scalebox{0.75}{6}~}} & \cellcolor{myblue!65.74}{72.\scalebox{0.75}{2}} & \cellcolor{myblue!70}{\textbf{82.\scalebox{0.75}{9}}} & \cellcolor{myred!66.85}{\underline{~3.\scalebox{0.75}{3}~}} \\
    \bottomrule
    \vspace{-1.0cm}
    \end{tabular}%
  }
  
  \label{tab:res_qa}%
\end{table}

\paragraph{Domain-specific Memory}
\label{subsubsec:domain_results}Table~\ref{tab:amazonqa_fulltest_baselines} reports the results on AmazonQA. \method{} consistently outperforms the neural baseline R-Net across all metrics, indicating strong cross-task generalization. After training on AmazonQA, Ours achieves substantial gains. Overall, training and interaction rounds steadily enhance performance, while the zero-shot results demonstrate promising transfer ability.


\begin{table}
\renewcommand{\arraystretch}{0.86}
\setlength{\tabcolsep}{2.6pt}
\centering
\vspace{-0.5cm}
\caption{Retrieval efficiency comparison. We report retrieval time in seconds on \texttt{HotpotQA}, \texttt{MuSiQue}, and \texttt{2Wiki}. For Time, lower is better. Best results are in \textbf{bold} and runner-ups are \underline{underlined}. The darker the cell, the better.}
\vspace{0.2cm}
\resizebox{0.6\linewidth}{!}{%
\begin{tabular}{lccc}
\toprule
\textbf{Dataset}
  & \texttt{HotpotQA}
  & \texttt{MuSiQue}
  & \texttt{2Wiki} \\
\cmidrule(r{1mm}){1-1}
\cmidrule(l{0.5mm}r{1mm}){2-2}
\cmidrule(l{0.5mm}r{1mm}){3-3}
\cmidrule(l{0.5mm}){4-4}
\textbf{Method}
  & Time$\downarrow$
  & Time$\downarrow$
  & Time$\downarrow$ \\
\cmidrule(r{1mm}){1-1}
\cmidrule(l{0.5mm}r{1mm}){2-2}
\cmidrule(l{0.5mm}r{1mm}){3-3}
\cmidrule(l{0.5mm}){4-4}

\multicolumn{4}{l}{\textit{Single-step retrieval methods}} \\

\texttt{ColBERTv2}
  & \cellcolor{myred!68}{\underline{0.035}}
  & \cellcolor{myred!70}{\textbf{0.030}}
  & \cellcolor{myred!68}{\underline{0.029}} \\

\texttt{HippoRAG}
  & \cellcolor{myred!54}{0.255}
  & \cellcolor{myred!54}{0.251}
  & \cellcolor{myred!56}{0.158} \\

\texttt{LightRAG}
  & \cellcolor{myred!25}{0.861}
  & \cellcolor{myred!23}{1.109}
  & \cellcolor{myred!28}{0.911} \\

\texttt{GraphRAG} \scalebox{0.72}{(MS)}
  & \cellcolor{myred!5}{2.759}
  & \cellcolor{myred!3}{3.037}
  & \cellcolor{myred!20}{1.204} \\

\texttt{GFM-RAG}
  & \cellcolor{myred!63}{0.107}
  & \cellcolor{myred!63}{0.124}
  & \cellcolor{myred!64}{0.060} \\

\cmidrule(r{1mm}){1-1}
\cmidrule(l{0.5mm}r{1mm}){2-2}
\cmidrule(l{0.5mm}r{1mm}){3-3}
\cmidrule(l{0.5mm}){4-4}

\multicolumn{4}{l}{\textit{Iterative retrieval methods}} \\

\texttt{IRCoT} + \texttt{ColBERTv2}
  & \cellcolor{myred!18}{1.146}
  & \cellcolor{myred!20}{1.152}
  & \cellcolor{myred!10}{2.095} \\

\texttt{IRCoT} + \texttt{HippoRAG}
  & \cellcolor{myred!2}{3.162}
  & \cellcolor{myred!2}{3.104}
  & \cellcolor{myred!2}{3.441} \\

\cmidrule(r{1mm}){1-1}
\cmidrule(l{0.5mm}r{1mm}){2-2}
\cmidrule(l{0.5mm}r{1mm}){3-3}
\cmidrule(l{0.5mm}){4-4}

\textbf{\method{}}
  & \cellcolor{myred!70}{\textbf{0.032}}
  & \cellcolor{myred!68}{\underline{0.034}}
  & \cellcolor{myred!70}{\textbf{0.019}} \\

\bottomrule
\end{tabular}%
}
\label{tab:efficiency}
\vspace{-0.5cm}
\end{table}


\subsection{Long-term Agent Memory Evaluation}
\label{subsec:agent_memory}

The LongMemEval results are shown in Table~\ref{tab:longmemeval_baselines}. 
The HaluMem results are shown in Table~\ref{tab:halumem_medium_baselines}. A direct full-stack leaderboard is informative but does not cleanly isolate the contribution of a retrieval substrate: mature memory systems differ in memory extraction, consolidation and other engineering choices. We therefore perform a paired component-level intervention. For each host, we keep its non-retrieval memory pipeline and evaluation protocol fixed and replace its native memory-access/retrieval backend with SAGE. The \textsc{+SAGE (1 round)} variant additionally performs one SAGE writer-reader evolution round under the same host protocol.
Detailed information is documented in
Appendix~\ref{app:host_integration_details}.

On LongMemEval, all three matched hosts improve in overall accuracy after the
SAGE substitution and one evolution round. On HaluMem, the same intervention improves both upstream memory operations and
downstream QA.

\vspace{-0.2cm}


\subsection{Ablations, Failure Modes, and Efficiency}
\label{sec:exp_analysis}

\vspace{-0.2cm}
As shown in \Cref{tab:efficiency}, \method{} demonstrates a strong speed advantage. SAGE achieves the lowest retrieval-engine latency on HotpotQA and
2WikiMultiHopQA and is a close second to ColBERTv2 on MuSiQue, indicating strong potential for practical and large-scale deployment. To further analyze the interpretability of \method{}, we visualize the retrieved subgraph for a representative case, as shown in Figure~\ref{fig:retrieved_subgraph}. A detailed case study is provided in Appendix~\ref{sec:interpretation}. Detailed writer-side and reader-side ablations are provided in
Appendix~\ref{subsec:writer_analysis} and
Appendix~\ref{subsec:reader_ablation}. The coupling ablation in Table~\ref{tab:coupling_ablation} separates the contribution of the main writer-reader components. The category-level KU/TR/MS statistics and concrete examples are reported in Appendix~\ref{app:memory_failure_analysis}. Table~\ref{tab:efficiency} reports \emph{retrieval-engine} latency. A measured audit of training GPU-hours, LM call counts, query planning, NER, entity linking, retrieval, and answer-generation latency is provided in
Appendix~\ref{app:complexity2}.

\vspace{-0.2cm}

\section{Conclusion}
\vspace{-0.2cm}
We presented SAGE, a self-evolving agentic graph-memory engine that treats memory as a dynamic substrate for writing, reading, and continual improvement. Experiments show that SAGE improves evidence recovery, grounding, and retrieval efficiency, suggesting that self-evolving graph memory is a promising foundation for long-horizon language agents.

\newpage

\bibliographystyle{iclr2027_conference}
\bibliography{references}

\makeatletter
\@ifundefined{b@wu2024longmemeval}{}{%
  \expandafter\let\csname b@longmemeval2024\expandafter\endcsname
    \csname b@wu2024longmemeval\endcsname
}
\@ifundefined{b@chen2025halumem}{}{%
  \expandafter\let\csname b@halumem2025\expandafter\endcsname
    \csname b@chen2025halumem\endcsname
}
\makeatother

\newpage
\appendix
\begin{algorithm}[t]
\caption{Writer-Reader Self-evolution Training for \method{}}
\label{alg:self_evolving_memory}
\KwIn{
Training set $\mathcal{D}_{\mathrm{train}}$, writer $\pi_{\theta_0}$, GFM reader $f_{\phi_0}$, self-evolution iterations $T$
}
\KwOut{Trained writer $\pi_{\theta_T}$ and reader $f_{\phi_T}$}
\For{$t=0,\dots,T-1$}{
    \textcolor{gray}{\tcp{Writer update: fixed GFM reader as reward environment}}
    \For{each sample $x=(q,\contextset,\support,y) \in \mathcal{D}_{\mathrm{train}}$}{
        Sample $G$ graph construction trajectories $\{\tau_i\}_{i=1}^{G}$ from $\pi_{\theta_t}$\;
        \For{$i=1,\dots,G$}{
            Obtain graph $\kg_i$ and retrieve $P_k(q,\kg_i)$ using $f_{\phi_t}$\;
            Calculate return $R_i$\;
        }
        Update writer $\pi_{\theta_t}$\;
    }
    \textcolor{gray}{\tcp{Reader update: improved graphs as memory substrate}}
    Construct a set of graph memories $\{\kg_x\}$ for the training corpus using $\pi_{\theta_{t+1}}$\;
    Update GFM reader $\f_{\phi_t}$ on $\{\kg_x\}$\;
}
\end{algorithm}
\section{Additional Analysis of the Memory Writer}
\label{app:memory_writer_analysis}

This appendix provides a detailed analysis of the memory writer experiments in the main text.

\subsection{Reward Design and Writer Behavior}
\label{app:writer_reward_analysis}

In Table~\ref{tab:writer_main_results}, different RL rewards induce different writer behaviors. GFM-pretrained-only achieves Precision/Recall/Deducible of $0.838/0.818/0.510$, while GFM-finetuned achieves $0.824/0.813/0.512$, indicating that relying solely on supervised finetuning cannot stably improve the utility of graph memory for a frozen reader. This result is also consistent with our setup: the goal of the memory writer is not to reproduce a static graph format.

RL-Recall improves Precision and Recall to $0.889/0.835$, but Deducible drops to $0.502$. This shows that rewarding only supporting context coverage encourages the writer to store more locally relevant evidence, but does not necessarily lead to a complete multi-hop reasoning chain. RL-F1 further raises Recall to $0.881$, but Deducible is only $0.497$, again indicating a gap between retrieval matching quality and answer deducibility: the reader hitting the supporting contexts does not guarantee that these contexts are organized in a way sufficient to support answer reasoning. In contrast, RL-Deduce achieves $0.861/0.892/0.517$, showing that using answer deducibility directly as feedback can encourage the writer to focus more on bridging entities, cross-document relations, and answer-relevant causal or attribute paths.

RL-Hybrid achieves Precision and Recall of $\mathbf{0.902}$ and $\mathbf{0.917}$, respectively, representing improvements of $+0.064$ and $+0.099$ over pretrained-only, while Deducible reaches $0.522$. This indicates that hybrid rewards can mitigate the bias of a single reward: they both avoid the introduction of too much weakly relevant evidence caused by a pure recall reward and prevent a pure deducibility reward from overfavoring short paths or local answer clues. Hybrid + frozen answer API achieves the highest Deducible, at $\mathbf{0.526}$, but Precision and Recall drop to $0.832/0.874$. This suggests that stronger answer-side feedback can further improve reasoning usability, but it may also make the writer more conservative, writing only evidence directly related to the final answer and thereby sacrificing some supporting context coverage.

\subsection{Cross-domain Transfer}
\label{app:writer_transfer_analysis}

Table~\ref{tab:domain_transfer} shows that the base writer trained on HotpotQA/MuSiQue has a certain degree of transferability to new domains, but training on the target domain remains very important. On GRBench, Base$\rightarrow$GRBench achieves Precision/Recall/Deducible of $0.575/0.609/0.411$, while GRBench train$\rightarrow$val improves to $0.794/0.833/0.596$. This improvement indicates that the writing strategy learned for multi-hop QA can transfer to structured product or domain graph memory tasks, but the entity types, attribute relations, and evidence granularity in the target domain still need to be re-adapted.

On HaluMem and LongMemEval, cross-domain differences are even more pronounced. The Base$\rightarrow$HaluMem results are $0.230/0.448/0.299$, which improve to $0.312/0.708/0.438$ after training on the target domain; the Base$\rightarrow$LongMemEval results are $0.232/0.376/0.475$, which improve to $0.377/0.439/0.531$ after training on the target domain. These results indicate that memory writing in agent memory tasks requires not only extracting explicit facts, but also maintaining user preferences, temporal order, state updates, and long-term consistency. In traditional multi-hop QA, supporting contexts often form a relatively static set of evidence centered around a single question, whereas information in long-term memory tasks changes over time and involves personalization, conflict updates, and context dependence. Therefore, although reader-aware RL feedback can provide transferable writing principles, interaction feedback from the target domain remains crucial for achieving stable performance.

\subsection{Writing Protocol and Interaction Budget}
\label{app:writer_protocol_analysis}

Table~\ref{tab:protocol_ablation} shows that the writing protocol significantly changes the trade-off among Precision, Recall, and Deducible. The results for Tight=True are $0.836/0.806/0.515$; the results for Tight=False are $0.845/0.851/0.506$. After relaxing the protocol, Recall improves noticeably, indicating that the writer can write more potentially relevant evidence; however, Deducible declines, suggesting that the additional evidence also contains more noise, redundant facts, or weakly related local information. Although this content may increase the coverage of supporting context, it can dilute the reasoning chain that truly supports the answer.

Iterative writing further highlights the role of the interaction budget. For Iterative, 12 turns, tight, Precision/Recall/Deducible are $0.852/0.829/0.516$; after increasing to 20 turns, Recall reaches the highest value of $\mathbf{0.881}$, indicating that multi-turn reader feedback helps the writer complete cross-document bridging paths. For Iterative, 24 turns, loose, Precision and Deducible reach $\mathbf{0.863}$ and $\mathbf{0.531}$, respectively, but Recall falls back to $0.826$. This shows that more rounds of interaction are not simply “the longer, the better”: the benefit comes from the writer revising the graph structure based on reader feedback, whereas when the protocol is too loose or the writing space becomes too large, the additional content may alter the reader’s ranking, causing some gold supporting contexts to be pushed out of the top results.

\subsection{Reader-side Sensitivity}
\label{app:reader_side_sensitivity}

\begin{figure*}[t]
    \centering
    \includegraphics[width=0.96\textwidth]{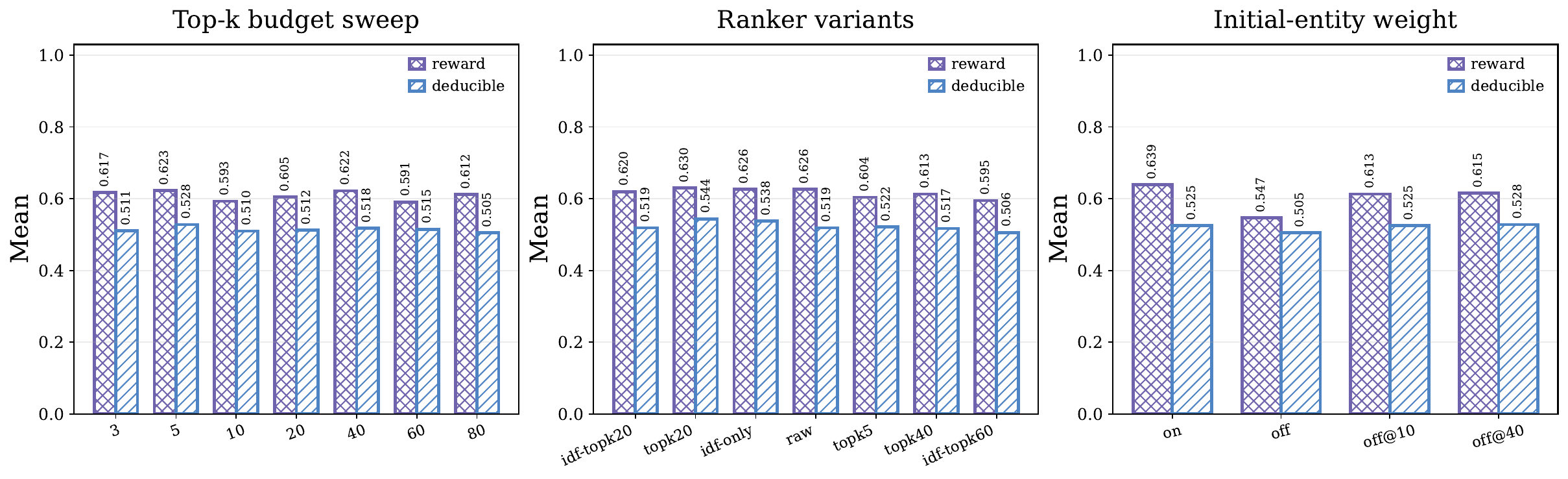}
    \caption{Freeze the sensitivity analysis on the reader side. The impact of the initial-entity weight on reward and Deducible is the most stable; the top-$k$ and ranker variants exhibit a non-monotonic budget–noise trade-off.}
    \label{fig:retriever_ablation}
\end{figure*}

Figure~\ref{fig:retriever_ablation} analyzes the impact of the frozen reader setting on writer training results. First, the top-$k$ budget sweep shows a non-monotonic trend: at $k=5$, reward and Deducible are $0.623/0.528$, the better setting in this group; at $k=40$, reward remains at $0.622$, but Deducible drops to $0.518$; at $k=60$, reward further declines to $0.591$. This indicates that expanding the retrieval budget does not necessarily lead to better reader feedback. Although a larger top-$k$ improves potential coverage, it also introduces more weakly related or redundant evidence, diluting the reasoning chain that truly supports the answer.

The ranker variants also reflect a similar coverage--noise trade-off. topk20 achieves the highest reward and Deducible, at $0.630/0.544$, respectively; idf-only and raw both have a reward of $0.626$, but their Deducible scores are $0.538$ and $0.519$, respectively; idf-topk60 declines to $0.595/0.506$. This shows that the reader ranker cannot rely solely on entity overlap or on expanding the candidate set, but instead must strike a balance among entity matching, semantic relevance, and contextual compactness. For the writer, an overly weak ranker makes it difficult for effective graph structure to be read out, while an overly broad candidate space amplifies the negative impact of noisy writes.

Initial-entity weight is the most stable factor among the three groups of reader-side settings. When initial-entity weight is enabled, reward reaches $0.639$ and Deducible is $0.525$; when it is disabled, reward drops significantly to $0.547$ and Deducible falls to $0.505$. Even when the budget is increased after disabling it, the rewards for off@10 and off@40 recover only to $0.613$ and $0.615$. This indicates that the initial entity anchor is crucial in multi-hop graph retrieval: it helps the reader enter the correct local subgraph from the question entity and expand along the bridging relations written by the writer to the evidence supporting the answer. Without this anchor, simply increasing the retrieval budget cannot fully compensate for the deviation in graph traversal direction.

\subsection{Training Stability and Regularization}
\label{app:writer_regularization_analysis}

\begin{figure*}[t]
    \centering
    \includegraphics[width=0.96\textwidth]{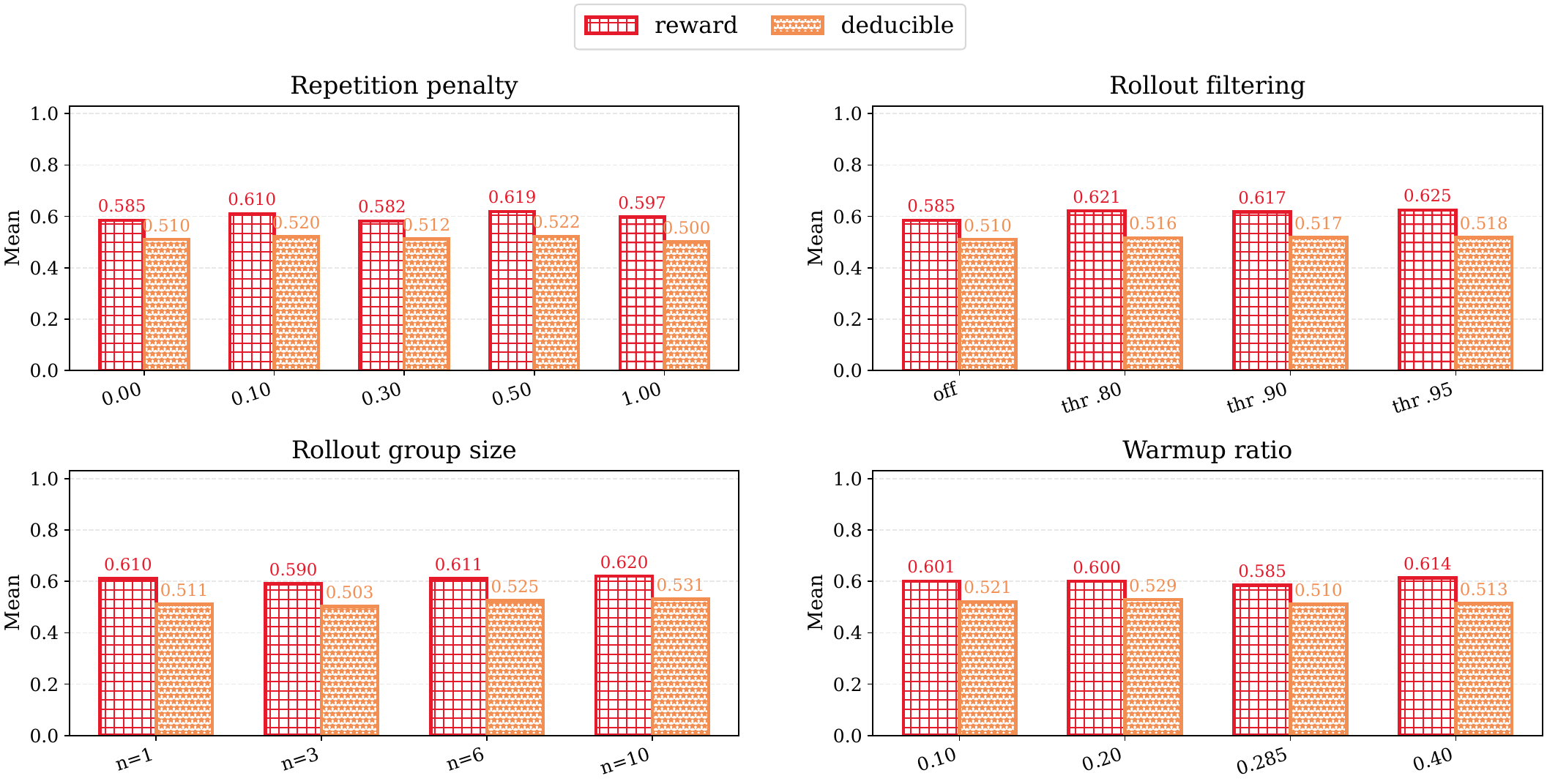}
    \caption{Training regularization and scaling analysis. A larger rollout group size and a moderate warmup ratio show a slight advantage in this batch of results, but the gains are smaller than the effects of reward design, the reader-side initial-entity weight, and the writing protocol.}
    \label{fig:regularization_scaling}
\end{figure*}

Figure~\ref{fig:regularization_scaling} shows the effects of training regularization and rollout settings on the writer. First, the repetition penalty affects both reward and Deducible, but the trend is not monotonic. Without any penalty, the result is $0.585/0.510$; with a penalty of $0.10$, it improves to $0.610/0.520$; with a penalty of $0.50$, it reaches the best result in this group at $0.619/0.522$; after further increasing it to $1.00$, it drops to $0.597/0.500$. This indicates that a moderate penalty on repeated triples can suppress redundant edges and cyclic expressions, but an overly strong penalty may limit the writer’s necessary restatement of key facts. Especially in multi-hop reasoning, the same bridging entity often needs to appear in multiple relational paths, so repetition is not always meaningless noise.

Rollout filtering brings consistent but limited gains. When filtering is disabled, reward/Deducible is $0.585/0.510$; after applying \texttt{thr\_80}, \texttt{thr\_90}, and \texttt{thr\_95}, reward increases to $0.621$, $0.617$, and $0.625$, respectively, while Deducible remains stable at $0.516$--$0.518$. This suggests that filtering out low-quality rollouts can reduce the interference of noisy trajectories with policy updates, allowing the writer to learn effective writing strategies more stably. However, the differences between thresholds are small, indicating that the main role of filtering is to remove obviously negative samples rather than determine the final performance ceiling.

Rollout group size and warmup ratio further affect training stability. As group size increases from $n=1$ to $n=10$, reward/Deducible rises from $0.610/0.511$ to $0.620/0.531$, indicating that a larger group size can provide more reliable relative preference estimates and help RL distinguish more accurately between effective and ineffective writing. The optimal warmup ratio appears at $0.20$, where Deducible reaches $0.529$; too little warmup may lead to unstable early policy updates, while too much warmup may delay the effect of the RL signal. Overall, these regularization and training scale settings can improve stability, but their gains are smaller than those from reward design, reader-side initial-entity weight, and the interaction protocol itself. This shows that the core improvement of the memory writer comes from reader-aware RL feedback: it forces the graph constructor to learn to preserve bridging entities, cross-document relations, and evidence chains that support answer derivation, while also reducing repetitive structures and irrelevant local facts.
\section{Signal-to-Noise Ratio and Retrieval Budget of Structurally Gated Propagation}
\label{sec:theory_signal_noise_budget}

This section analyzes the structural capability of the GFM memory reader in \method{} from the perspective of signal propagation and retrieval budget. Unlike graph-isomorphism expressivity analyses centered on $k$-WL, we focus on the following question: on noisy graph memories dynamically written by the memory writer, how do soft addressing, structurally gated propagation, context--schema dual-channel calibration, and entity-to-document projection jointly improve the ratio of query-relevant evidence signal to distractor noise, thereby reducing the top-$k$ retrieval budget required to achieve a given level of evidence coverage?

\subsection{Review of the SAGE-GFM Reader Formalization}
\label{subsec:theory_reader_formalization}

Given a sample $x=(q,D,D^+,y)$, where $q$ denotes the query, $D=\{d_i\}_{i=1}^N$ denotes the candidate memory fragments, and $D^+\subseteq D$ denotes the gold evidence set supporting the answer $y$, the memory writer constructs a heterogeneous graph
\begin{equation}
    G = W_\theta(q,D)
    = (V_E\cup V_D, E_{EE}\cup E_{ED}),
\end{equation}
where $V_E$ is the set of entity nodes, $V_D$ is the set of memory-fragment nodes, $E_{EE}$ denotes entity--entity relation edges, and $E_{ED}$ denotes entity--text-fragment anchoring edges. The GFM memory reader outputs an entity distribution, a document distribution, and an optional activated subgraph:
\begin{equation}
    f_\phi(q,G,D)
    = \big(p_\phi(e\mid q,G),\,p_\phi(d\mid q,G,D),\,G_q\big).
\end{equation}

The reader first uses query planning and soft addressing to generate query-conditioned initial activation for entities. Let $s_e(q)$ denote the entry score of entity $e$, which integrates multiple cues such as explicit entities, aliases, pseudo-query similarity, answer type, hard constraints, and entity linking. The initial activation distribution is then given by
\begin{equation}
    p_0(e\mid q)
    =
    \frac{\exp(s_e(q)/T_0)}{\sum_{v\in V_E}\exp(s_v(q)/T_0)} .
    \label{eq:theory_softmax_seed}
\end{equation}
The initial entity representation is written as
\begin{equation}
    h_e^{(0)}
    =
    \big(p_0(e\mid q)\big)^\eta W_q\operatorname{Emb}(q)
    + W_x x_e,
    \qquad 0\le \eta\le 1 .
    \label{eq:theory_initial_rep}
\end{equation}
The reader then constructs structural gates using node-level structural features, edge-pair structural features, and graph-level structural summaries. Specifically, let
\begin{align}
    \varphi(v)
    &= \big(\log(1+d_v),\,c_v,\,\kappa_v,\,\bar d_{\calN(v)}\big),\label{eq:theory_node_struct}\\
    \psi(u,v)
    &= \big(|d_u-d_v|,\,|\calN(u)\cap\calN(v)|,\,
    \operatorname{Jaccard}(\calN(u),\calN(v))\big),\label{eq:theory_edge_struct}\\
    r_G
    &= \big(\operatorname{mean}_{v\in V_E}\varphi(v);\,
    \operatorname{std}_{v\in V_E}\varphi(v);\,
    \operatorname{dens}(G)\big).
    \label{eq:theory_graph_struct}
\end{align}
The edge structural context at layer $l$ is
\begin{equation}
    z_{uv}^{(l)}
    =
    \big(E_n^{(l)}(\varphi(u));\,E_n^{(l)}(\varphi(v));\,
    E_p^{(l)}(\psi(u,v));\,E_g^{(l)}(r_G)\big),
\end{equation}
which generates the vector-valued gate
\begin{equation}
    g_{uv}^{(l)}
    =
    1+\delta\tanh\big(\operatorname{MLP}_g^{(l)}(z_{uv}^{(l)})\big).
    \label{eq:theory_vector_gate}
\end{equation}
Let $\eta_{uv}\ge 0$ denote the normalized adjacency weight with self-loops. The message and node update are
\begin{align}
    m_{u\to v}^{(l)}
    &=
    \eta_{uv}g_{uv}^{(l)}\odot W_m^{(l)}h_u^{(l-1)},
    \label{eq:theory_message}\\
    h_v^{(l)}
    &=
    \operatorname{LayerNorm}
    \left(
    h_v^{(l-1)}+
    \operatorname{PReLU}
    \left(
    b^{(l)}+\sum_{u\in\calN(v)}m_{u\to v}^{(l)}
    \right)
    \right).
    \label{eq:theory_update}
\end{align}
In addition, the reader combines a contextual calibration channel on the current graph with a cross-graph schema prior channel:
\begin{equation}
    H(q,G)=H_{\ctx}+\beta_{\sch}H_{\sch}.
    \label{eq:theory_ctx_schema}
\end{equation}

\subsection{Recoverable Evidence Region and Effective Signal-to-Noise Ratio}
\label{subsec:recoverable_region_snr}

\begin{definition}[Recoverable Evidence Region]
\label{def:recoverable_region}
Fix a query $q$ and the current memory graph $G$. Let $R_q\subseteq V_E$ denote the recoverable evidence region under query $q$, namely the set of entities jointly determined by the current graph structure, anchoring edges, and reader-reachable paths. This set contains nodes that support the answer, connect supporting documents, or serve as bridge entities. If the entity anchor set of document $d$ is denoted by $A(d)\subseteq V_E$, then the anchor coverage of the current graph over the gold evidence is defined as
\begin{equation}
    \rho_A
    =
    \frac{\left|\{d\in D^+: A(d)\cap R_q\neq\varnothing\}\right|}{|D^+|}.
    \label{eq:anchor_coverage}
\end{equation}
\end{definition}

\begin{definition}[Query-Relevant Scalar Activation]
\label{def:scalar_activation}
Let $r_q$ be the direction induced by the query representation or the final scoring head. The nonnegative query-relevant activation of node $v$ at layer $l$ is defined as
\begin{equation}
    a_v^{(l)}
    =
    \big[\langle r_q,h_v^{(l)}\rangle\big]_+,
    \label{eq:scalar_activation}
\end{equation}
where $[t]_+=\max\{t,0\}$. The evidence signal mass, noise mass, and effective signal-to-noise ratio at layer $l$ are respectively defined as
\begin{align}
    S_l &= \sum_{v\in R_q}a_v^{(l)},
    \label{eq:S_l}\\
    N_l &= \sum_{v\in V_E\setminus R_q}a_v^{(l)},
    \label{eq:N_l}\\
    \SNR_l &= \frac{S_l}{N_l},
    \label{eq:SNR_l}
\end{align}
with the convention that $\SNR_l=+\infty$ when $N_l=0$.
\end{definition}

\begin{remark}[Scope of the Scalar-Channel Analysis]
\label{rem:scalar_channel}
Equation~\eqref{eq:scalar_activation} does not assume that the full vector update in Eq.~\eqref{eq:theory_update} is a nonnegative linear recurrence in every coordinate. LayerNorm, PReLU, residual connections, and linear transformations may all change representation directions. We only analyze the query-relevant channel on which the final retrieval score depends, and absorb the additional effects caused by nonlinearities and directional shifts into a perturbation term. This avoids an overly strong coordinate-wise monotonicity assumption.
\end{remark}

\subsection{Aggregate Propagation Assumptions and Structural Gating Coefficients}
\label{subsec:aggregate_assumption}

Idealized analyses often assume that every evidence edge has a uniform lower gate bound $g_+$ and every noisy edge has a uniform upper gate bound $g_-$. However, in graph memories dynamically constructed by an LLM writer, edges may be missing, erroneous, or repeated; some evidence edges may be underestimated, while some distractor edges may receive high gates. We therefore adopt an aggregate propagation assumption.

\begin{assumption}[Query-Relevant Effective Propagation Operator]
\label{ass:effective_operator}
Fix a query $q$, the current graph memory $G$, and the reader representation at layer $l$. For each layer $l\in\{1,\ldots,L\}$, there exists a nonnegative matrix $T_l\in\R_{\ge 0}^{|V_E|\times |V_E|}$ and a nonnegative perturbation vector $\epsilon_l\in\R_{\ge 0}^{|V_E|}$ such that the query-relevant activation vector at layer $l$,
\[
    a^{(l)}
    =
    \big(a_v^{(l)}\big)_{v\in V_E},
\]
is controlled by the previous-layer activation $a^{(l-1)}$ in the following coordinate-wise sense:
\begin{equation}
    a^{(l)}
    \preceq
    T_l a^{(l-1)}+\epsilon_l .
    \label{eq:effective_operator_upper}
\end{equation}
Here, $\preceq$ denotes coordinate-wise inequality. The operator $T_l$ denotes the effective propagation operator induced by the $l$-th layer on the query-relevant scalar channel. It absorbs the combined effects of normalized adjacency weights $\eta_{uv}$, structural gates $g_{uv}^{(l)}$, message projection $W_m^{(l)}$, context--schema representation composition, and final scoring-channel projection into a single nonnegative propagation kernel. In other words, $T_l(u,v)$ is the effective nonnegative contribution strength of the query-relevant activation of node $v$ at the previous layer to node $u$ at layer $l$.

The perturbation term $\epsilon_l$ absorbs residual effects that are difficult to exactly characterize by nonnegative linear propagation, including LayerNorm, PReLU, residual connections, vector-direction rotation, scoring-channel mismatch, and finite-parameter approximation error.

Furthermore, we only require the propagation process to preserve effective signal in the evidence region in an aggregate sense. We do not require the structural gate to perfectly distinguish every evidence edge from every noisy edge. The operator $T_l$ may allow some evidence edges to be underestimated and some distractor edges to be overestimated; the aggregate propagation coefficients defined below only characterize the overall effect of these local errors on the evidence and noise regions.
\end{assumption}

Let
\[
    \bar R_q = V_E\setminus R_q .
\]
Partition $T_l$ according to the node sets $R_q$ and $\bar R_q$:
\begin{equation}
    T_l
    =
    \begin{pmatrix}
        T_{RR}^{(l)} & T_{R\bar R}^{(l)}\\
        T_{\bar R R}^{(l)} & T_{\bar R\bar R}^{(l)}
    \end{pmatrix}.
    \label{eq:T_l_block}
\end{equation}
Here, $T_{RR}^{(l)}$ denotes effective propagation within the evidence region, $T_{\bar R\bar R}^{(l)}$ denotes effective propagation within the noise region, $T_{\bar R R}^{(l)}$ denotes leakage propagation from the evidence region to the noise region, and $T_{R\bar R}^{(l)}$ denotes propagation from the noise region to the evidence region.

\begin{definition}[Aggregate Propagation Coefficients]
\label{def:aggregate_coefficients}
Given the effective propagation operator $T_l$ at layer $l$ and its block decomposition in Eq.~\eqref{eq:T_l_block}, define the three aggregate propagation coefficients $A_l$, $B_l$, and $C_l$ as follows.

$A_l$ is the evidence-retention coefficient. It characterizes the minimum fraction of total mass that remains inside the evidence region $R_q$ after any nonnegative evidence signal $x$ propagates one layer within $R_q$. Formally, $A_l$ is any nonnegative constant satisfying
\begin{equation}
    A_l
    \le
    \inf_{x\in\R_{\ge 0}^{|R_q|},\ \bbone^\top x>0}
    \frac{\bbone^\top T_{RR}^{(l)}x}{\bbone^\top x}.
    \label{eq:A_l_def}
\end{equation}
Equivalently, for any nonnegative evidence signal $x\in\R_{\ge 0}^{|R_q|}$ with $\bbone^\top x>0$,
\[
    \bbone^\top T_{RR}^{(l)}x
    \ge
    A_l\,\bbone^\top x .
\]

$B_l$ is the noise self-propagation coefficient. It characterizes the maximum extent to which any nonnegative noise signal $y$ can be retained or expanded after one layer of propagation inside the noise region $\bar R_q$. Formally, $B_l$ is any nonnegative constant satisfying
\begin{equation}
    B_l
    \ge
    \sup_{y\in\R_{\ge 0}^{|\bar R_q|},\ \bbone^\top y>0}
    \frac{\bbone^\top T_{\bar R\bar R}^{(l)}y}{\bbone^\top y}.
    \label{eq:B_l_def}
\end{equation}
Equivalently, for any nonnegative noise signal $y\in\R_{\ge 0}^{|\bar R_q|}$ with $\bbone^\top y>0$,
\[
    \bbone^\top T_{\bar R\bar R}^{(l)}y
    \le
    B_l\,\bbone^\top y .
\]

$C_l$ is the evidence-to-noise leakage coefficient. It characterizes the maximum fraction of an arbitrary nonnegative signal in the evidence region that can leak into the non-evidence region $\bar R_q$ after one layer of propagation. Formally, $C_l$ is any nonnegative constant satisfying
\begin{equation}
    C_l
    \ge
    \sup_{x\in\R_{\ge 0}^{|R_q|},\ \bbone^\top x>0}
    \frac{\bbone^\top T_{\bar R R}^{(l)}x}{\bbone^\top x}.
    \label{eq:C_l_def}
\end{equation}
Equivalently, for any nonnegative evidence signal $x\in\R_{\ge 0}^{|R_q|}$ with $\bbone^\top x>0$,
\[
    \bbone^\top T_{\bar R R}^{(l)}x
    \le
    C_l\,\bbone^\top x .
\]

Finally, let
\begin{equation}
    \xi_l
    =
    \bbone^\top \epsilon_{l,\bar R}
    \label{eq:xi_l_def}
\end{equation}
denote the total perturbation mass injected into the noise region $\bar R_q$ at layer $l$ by nonlinearities, normalization, representation-direction shifts, and approximation errors. Here, $\epsilon_{l,\bar R}$ denotes the restriction of the perturbation vector $\epsilon_l$ to $\bar R_q$.
\end{definition}

\begin{lemma}[Aggregate Propagation Recurrence]
\label{lem:aggregate_recursion}
Under Assumption~\ref{ass:effective_operator} and Definition~\ref{def:aggregate_coefficients}, if $S_{l-1}>0$ and $N_{l-1}\ge 0$, then layer $l$ satisfies
\begin{align}
    S_l &\ge A_l S_{l-1},
    \label{eq:signal_recursion}\\
    N_l &\le B_l N_{l-1}+C_l S_{l-1}+\xi_l.
    \label{eq:noise_recursion}
\end{align}
\end{lemma}

\begin{proof}
Let $a_R^{(l-1)}$ and $a_{\bar R}^{(l-1)}$ be the restrictions of $a^{(l-1)}$ to $R_q$ and $\bar R_q$, respectively. By the definition of the evidence-retention coefficient in Eq.~\eqref{eq:A_l_def}, the total mass retained within the evidence region through $R_q\to R_q$ propagation is at least
\begin{equation}
    \bbone^\top T_{RR}^{(l)}a_R^{(l-1)}
    \ge A_l \bbone^\top a_R^{(l-1)}
    = A_l S_{l-1},
\end{equation}
and thus $S_l\ge A_lS_{l-1}$.

On the other hand, the layer-$l$ mass in the noise region can be upper-bounded by three terms: noise self-propagation, evidence leakage, and perturbation:
\begin{equation}
    N_l
    \le
    \bbone^\top T_{\bar R\bar R}^{(l)}a_{\bar R}^{(l-1)}
    +
    \bbone^\top T_{\bar R R}^{(l)}a_R^{(l-1)}
    +
    \bbone^\top \epsilon_{l,\bar R}.
\end{equation}
Using Eqs.~\eqref{eq:B_l_def}, \eqref{eq:C_l_def}, and \eqref{eq:xi_l_def}, we obtain
\begin{equation}
    N_l
    \le
    B_l N_{l-1}+C_lS_{l-1}+\xi_l.
\end{equation}
This proves the lemma.
\end{proof}

\subsection{Realistic Aggregate Signal-to-Noise Ratio Bound}
\label{subsec:aggregate_snr_bound}

\begin{theorem}[Realistic Aggregate SNR Bound]
\label{thm:realistic_snr}
Assume that for all $l=1,\ldots,L$, there exist $A_l>0$, $B_l\ge 0$, $C_l\ge 0$, and $\xi_l\ge 0$ such that the recurrences in Eqs.~\eqref{eq:signal_recursion}--\eqref{eq:noise_recursion} hold. Let
\begin{equation}
    Q_l = \SNR_l^{-1}=\frac{N_l}{S_l}.
\end{equation}
Then
\begin{equation}
    Q_L
    \le
    \left(\prod_{l=1}^{L}\frac{B_l}{A_l}\right)Q_0
    +
    \sum_{i=1}^{L}
    \left(
    \frac{C_i}{A_i}
    +
    \frac{\xi_i}{A_iS_{i-1}}
    \right)
    \prod_{t=i+1}^{L}\frac{B_t}{A_t} .
\label{eq:Q_bound}
\end{equation}
Equivalently, if the right-hand side is finite, then
\begin{equation}
    \SNR_L
    \ge
    \left[
    \left(\prod_{l=1}^{L}\frac{B_l}{A_l}\right)\SNR_0^{-1}
    +
    \sum_{i=1}^{L}
    \left(
    \frac{C_i}{A_i}
    +
    \frac{\xi_i}{A_iS_{i-1}}
    \right)
    \prod_{t=i+1}^{L}\frac{B_t}{A_t}
    \right]^{-1} .
\label{eq:SNR_bound_realistic}
\end{equation}
The empty product is defined as $1$.
\end{theorem}

\begin{proof}
By Lemma~\ref{lem:aggregate_recursion}, for any $l$,
\begin{equation}
    S_l\ge A_lS_{l-1},
    \qquad
    N_l\le B_lN_{l-1}+C_lS_{l-1}+\xi_l.
\end{equation}
Therefore,
\begin{align}
    Q_l
    =\frac{N_l}{S_l}
    &\le
    \frac{B_lN_{l-1}+C_lS_{l-1}+\xi_l}{A_lS_{l-1}} \\
    &=
    \frac{B_l}{A_l}Q_{l-1}
    +
    \frac{C_l}{A_l}
    +
    \frac{\xi_l}{A_lS_{l-1}}.
\end{align}
Let
\begin{equation}
    r_l=\frac{B_l}{A_l},
    \qquad
    d_l=\frac{C_l}{A_l}+\frac{\xi_l}{A_lS_{l-1}}.
\end{equation}
Then
\begin{equation}
    Q_l\le r_lQ_{l-1}+d_l.
\end{equation}
Expanding this first-order nonhomogeneous recurrence yields
\begin{equation}
    Q_L
    \le
    \left(\prod_{l=1}^{L}r_l\right)Q_0
    +
    \sum_{i=1}^{L}d_i\prod_{t=i+1}^{L}r_t.
\end{equation}
Substituting back $r_l$ and $d_l$ proves Eq.~\eqref{eq:Q_bound}. Since $\SNR_L=Q_L^{-1}$, Eq.~\eqref{eq:SNR_bound_realistic} follows.
\end{proof}

\begin{corollary}[Layer-Homogeneous Case]
\label{cor:homogeneous_snr}
If $A_l=A>0$, $B_l=B\ge 0$, $C_l=C\ge 0$, and $\xi_l=0$, then
\begin{equation}
\SNR_L
    \ge
    \left[
    \left(\frac BA\right)^L\SNR_0^{-1}
    +
    \frac CA\sum_{i=0}^{L-1}\left(\frac BA\right)^i
    \right]^{-1}.
\label{eq:homogeneous_snr_bound}
\end{equation}
If further $C=0$, then
\begin{equation}
    \SNR_L\ge \left(\frac AB\right)^L\SNR_0 .
\label{eq:no_leak_exp_snr}
\end{equation}
\end{corollary}

\begin{proof}
Substituting $A_l=A$, $B_l=B$, $C_l=C$, and $\xi_l=0$ into Theorem~\ref{thm:realistic_snr} and simplifying the resulting geometric series gives the result.
\end{proof}

\begin{corollary}[Ideal Edge-Wise Gating as a Special Case]
\label{cor:ideal_gate_special_case}
Suppose there exist constants $g_+>g_->0$, $\alpha_+>0$, $\alpha_->0$, $g_0\ge0$, and $\lambda_{\leak}\ge0$ such that the effective retention inside the evidence region is $A=g_+\alpha_+$, the self-propagation inside the noise region is $B=g_-\alpha_-$, the evidence-to-noise leakage is $C=g_0\lambda_{\leak}$, and $\xi_l=0$. Then Theorem~\ref{thm:realistic_snr} reduces to
\begin{equation}
    \SNR_L
    \ge
    \left[
    \left(\frac{g_-\alpha_-}{g_+\alpha_+}\right)^L\SNR_0^{-1}
    +
    \frac{g_0\lambda_{\leak}}{g_+\alpha_+}
    \sum_{i=0}^{L-1}
    \left(\frac{g_-\alpha_-}{g_+\alpha_+}\right)^i
    \right]^{-1}.
\end{equation}
If $\lambda_{\leak}=0$, then
\begin{equation}
    \SNR_L
    \ge
    \left(\frac{g_+\alpha_+}{g_-\alpha_-}\right)^L\SNR_0.
\end{equation}
\end{corollary}

\subsection{Document Retrieval Budget}
\label{subsec:from_snr_to_budget}

The final retrieval targets of the SAGE-GFM reader are memory fragments or documents. Therefore, we need to convert the entity-level SNR bound into a document-level top-$k$ budget bound. Let the final document score be $S_D(d)$, and let the top-$k$ retrieval result be
\begin{equation}
    P_k(q,G)=\TopK_{d\in D}S_D(d).
\end{equation}

\begin{definition}[$\rho$-Coverage Retrieval Budget]
\label{def:rho_budget}
Given $0<\rho\le \rho_A$, let
\begin{equation}
    m_\rho=\lceil \rho |D^+|\rceil .
\end{equation}
The minimum top-$k$ budget required to achieve $\rho$-level gold evidence coverage is defined as
\begin{equation}
    \calB_\rho(q,G)
    =
    \min\left\{k:\ |P_k(q,G)\cap D^+|\ge m_\rho\right\}.
    \label{eq:rho_budget_def}
\end{equation}
Let $\tau_\rho^+$ denote the $m_\rho$-th largest score among gold evidence documents, namely the gold score threshold required to achieve $\rho$-coverage.
\end{definition}

\begin{lemma}[Quantile Retrieval Budget Bound]
\label{lem:quantile_budget}
Let the total score mass of distractor documents be
\begin{equation}
    M_L^- = \sum_{d\in D\setminus D^+}S_D(d).
\end{equation}
If $\tau_\rho^+>0$, then
\begin{equation}
\calB_\rho(q,G)
    \le
    m_\rho+\frac{M_L^-}{\tau_\rho^+}.
\label{eq:quantile_budget_bound}
\end{equation}
\end{lemma}

\begin{proof}
Define the set of distractor documents whose scores are not lower than $\tau_\rho^+$ as
\begin{equation}
    \calN_\rho
    =
    \{d\in D\setminus D^+: S_D(d)\ge \tau_\rho^+\}.
\end{equation}
For any $d\in\calN_\rho$, we have $S_D(d)\ge\tau_\rho^+$, and hence
\begin{equation}
    |\calN_\rho|\tau_\rho^+
    \le
    \sum_{d\in\calN_\rho}S_D(d)
    \le
    M_L^-.
\end{equation}
Thus $|\calN_\rho|\le M_L^-/\tau_\rho^+$. To ensure that the top-$k$ results contain at least $m_\rho$ gold evidence documents, it suffices to include these $m_\rho$ gold documents and all distractors whose scores are not lower than the threshold $\tau_\rho^+$. Therefore,
\begin{equation}
    \calB_\rho(q,G)
    \le
    m_\rho+|\calN_\rho|
    \le
    m_\rho+\frac{M_L^-}{\tau_\rho^+}.
\end{equation}
This proves the lemma.
\end{proof}

To use entity-level SNR for document-level retrieval, we need to control the noise expansion introduced by entity-to-document projection.

\begin{assumption}[Projection Noise and Gold Score Concentration]
\label{ass:projection}
There exist constants $K_A\ge0$, $\zeta_A\ge0$, and $c_\rho\in(0,1]$ such that the final document scores satisfy
\begin{align}
    M_L^- &\le K_A N_L+\zeta_A,
    \label{eq:projection_noise}\\
    \tau_\rho^+ &\ge \frac{c_\rho}{m_\rho}S_L.
    \label{eq:gold_concentration}
\end{align}
Here, $K_A$ is the noise expansion factor of entity-to-document projection, $\zeta_A$ denotes the projection residual caused by incorrect anchors, missing anchors, or additional text-similarity terms, and $c_\rho$ measures whether the evidence signal is effectively distributed over at least $m_\rho$ gold documents.
\end{assumption}

\begin{theorem}[Realistic Signal--Noise--Budget Bound]
\label{thm:realistic_signal_noise_budget}
Under the conditions of Theorem~\ref{thm:realistic_snr}, further assume that Assumption~\ref{ass:projection} holds. Then
\begin{equation}
    \calB_\rho(q,G)
    \le
    m_\rho
    +
    \frac{m_\rho K_A}{c_\rho}\SNR_L^{-1}
    +
    \frac{m_\rho\zeta_A}{c_\rho S_L}.
\label{eq:budget_snr_bound}
\end{equation}
Substituting Theorem~\ref{thm:realistic_snr} further yields the explicit upper bound
\begin{equation}
\begin{aligned}
    \calB_\rho(q,G)
    \le
    &\ m_\rho
    +
    \frac{m_\rho K_A}{c_\rho}
    \Bigg[
    \left(\prod_{l=1}^{L}\frac{B_l}{A_l}\right)\SNR_0^{-1} \\
    &\quad+
    \sum_{i=1}^{L}
    \left(
    \frac{C_i}{A_i}
    +
    \frac{\xi_i}{A_iS_{i-1}}
    \right)
    \prod_{t=i+1}^{L}\frac{B_t}{A_t}
    \Bigg]
    +
    \frac{m_\rho\zeta_A}{c_\rho S_L} .
\end{aligned}
\label{eq:budget_explicit_bound}
\end{equation}
\end{theorem}

\begin{proof}
By Lemma~\ref{lem:quantile_budget},
\begin{equation}
    \calB_\rho(q,G)
    \le
    m_\rho+\frac{M_L^-}{\tau_\rho^+}.
\end{equation}
By Assumption~\ref{ass:projection},
\begin{equation}
    \frac{M_L^-}{\tau_\rho^+}
    \le
    \frac{K_A N_L+\zeta_A}{(c_\rho/m_\rho)S_L}
    =
    \frac{m_\rho K_A}{c_\rho}\frac{N_L}{S_L}
    +
    \frac{m_\rho\zeta_A}{c_\rho S_L}.
\end{equation}
Since $N_L/S_L=\SNR_L^{-1}$, Eq.~\eqref{eq:budget_snr_bound} follows. Substituting the upper bound on $\SNR_L^{-1}$ from Theorem~\ref{thm:realistic_snr} into Eq.~\eqref{eq:budget_snr_bound} gives Eq.~\eqref{eq:budget_explicit_bound}.
\end{proof}

\begin{corollary}[Full Evidence Recovery Budget]
\label{cor:full_recovery_budget}
If $\rho=1$ and $\rho_A=1$, then $m_\rho=|D^+|$, and $\calB_\rho(q,G)$ reduces to the full evidence recovery budget. In this case, Theorem~\ref{thm:realistic_signal_noise_budget} provides an upper bound on the top-$k$ budget required to recover all gold evidence.
\end{corollary}

\subsection{Interpretation of the Theoretical Bound for the SAGE Design}
\label{subsec:design_interpretation}

Theorem~\ref{thm:realistic_snr} and Theorem~\ref{thm:realistic_signal_noise_budget} unify four reader design factors under the same retrieval-budget upper bound. To avoid relying only on intuitive discussion, we provide several direct monotonicity propositions.

\begin{proposition}[Monotonicity of the Budget Bound]
\label{prop:monotonicity}
Fix $m_\rho$, $K_A$, $c_\rho$, $\zeta_A$, $S_L$, and $\SNR_0$, and define
\begin{equation}
    \Gamma_L
    =
    \left(\prod_{l=1}^{L}\frac{B_l}{A_l}\right)\SNR_0^{-1}
    +
    \sum_{i=1}^{L}
    \left(
    \frac{C_i}{A_i}
    +
    \frac{\xi_i}{A_iS_{i-1}}
    \right)
    \prod_{t=i+1}^{L}\frac{B_t}{A_t}.
\end{equation}
Then the budget upper bound
\begin{equation}
    U_\rho
    =
    m_\rho+
    \frac{m_\rho K_A}{c_\rho}\Gamma_L+
    \frac{m_\rho\zeta_A}{c_\rho S_L}
\end{equation}
is monotonically nondecreasing in $\Gamma_L$, $K_A$, and $\zeta_A$, and monotonically nonincreasing in $c_\rho$ and $S_L$. If all other terms in the products are fixed, decreasing any of $B_l/A_l$, $C_l/A_l$, or $\xi_l/(A_lS_{l-1})$ cannot increase $U_\rho$.
\end{proposition}

\begin{proof}
The partial derivatives of $U_\rho$ with respect to $\Gamma_L$, $K_A$, and $\zeta_A$ are nonnegative, while the partial derivatives with respect to $c_\rho$ and $S_L$ are nonpositive. Moreover, $\Gamma_L$ is a nonnegative linear or multiplicative combination of $B_l/A_l$, $C_l/A_l$, and $\xi_l/(A_lS_{l-1})$. When the other terms are fixed, decreasing any such nonnegative term cannot increase $\Gamma_L$. The proposition follows.
\end{proof}

\begin{proposition}[Effect of Soft Addressing]
\label{prop:soft_addressing}
If soft addressing increases the initial evidence signal $S_0$ and decreases the initial noise mass $N_0$, thereby increasing $\SNR_0$, then, with all other coefficients fixed, the budget upper bound in Theorem~\ref{thm:realistic_signal_noise_budget} does not increase. In particular, explicit entities, aliases, pseudo-queries, type constraints, hard constraints, and entity-linking signals in query planning improve the final budget bound whenever they increase $S_0/N_0$ in an aggregate sense.
\end{proposition}

\begin{proposition}[Aggregate Advantage of Structural Gating]
\label{prop:structural_gating_advantage}
Compared with a reader without structural gating, suppose the structurally gated reader satisfies
\begin{equation}
    \frac{B_l^{\mathrm{gate}}}{A_l^{\mathrm{gate}}}
    \le
    \frac{B_l^{\mathrm{plain}}}{A_l^{\mathrm{plain}}},
    \qquad
    \frac{C_l^{\mathrm{gate}}}{A_l^{\mathrm{gate}}}
    \le
    \frac{C_l^{\mathrm{plain}}}{A_l^{\mathrm{plain}}},
    \qquad
    \frac{\xi_l^{\mathrm{gate}}}{A_l^{\mathrm{gate}}S_{l-1}^{\mathrm{gate}}}
    \le
    \frac{\xi_l^{\mathrm{plain}}}{A_l^{\mathrm{plain}}S_{l-1}^{\mathrm{plain}}}.
\end{equation}
Then the budget upper bound of the structurally gated reader is no larger than that of the ungated reader.
\end{proposition}

\begin{proposition}[Stability Interpretation of the Context--Schema Dual Channel]
\label{prop:ctx_schema_stability}
Let $\delta_l$ denote the failure probability of the aggregate recurrences in Eqs.~\eqref{eq:signal_recursion}--\eqref{eq:noise_recursion} at layer $l$. If the schema prior channel reduces the variance of cross-graph structural-role estimation and the context calibration channel reduces the current-graph adaptation error, so that $\delta_l$ decreases to $\delta_l'$ with $\delta_l'\le \delta_l$, then the probability lower bound under which Theorem~\ref{thm:realistic_snr} and Theorem~\ref{thm:realistic_signal_noise_budget} simultaneously hold improves from \(1-\sum_{l=1}^{L}\delta_l\) to \(1-\sum_{l=1}^{L}\delta_l'\).
\end{proposition}

The core quantities derived above are
\begin{equation}
    \SNR_L^{-1}
    \le
    \left(\prod_{l=1}^{L}\frac{B_l}{A_l}\right)\SNR_0^{-1}
    +
    \sum_{i=1}^{L}
    \left(
    \frac{C_i}{A_i}
    +
    \frac{\xi_i}{A_iS_{i-1}}
    \right)
    \prod_{t=i+1}^{L}\frac{B_t}{A_t},
    \label{eq:final_snr_inverse}
\end{equation}
and
\begin{equation}
    \calB_\rho(q,G)
    \le
    m_\rho
    +
    \frac{m_\rho K_A}{c_\rho}\SNR_L^{-1}
    +
    \frac{m_\rho\zeta_A}{c_\rho S_L}.
    \label{eq:final_budget_bound}
\end{equation}
Equation~\eqref{eq:final_snr_inverse} shows that soft addressing reduces the amount of noise that subsequent propagation must overcome by improving the initial $\SNR_0$; structural gating improves aggregate evidence retention and noise suppression by increasing $A_l$ and decreasing $B_l$ and $C_l$; the context--schema dual channel makes these aggregate inequalities more stable on dynamic graph memories by reducing cross-graph structural-role estimation error; and entity-to-document projection converts entity-level SNR into document-level budget efficiency by decreasing $K_A$ and $\zeta_A$ and increasing $c_\rho$. Equation~\eqref{eq:final_budget_bound} further shows that the advantage of SAGE-GFM does not rely on perfect edge-wise classification or zero-leakage assumptions. As long as evidence-retention dominance is achieved in an aggregate or high-probability sense, i.e., $B_l/A_l$ and $C_l/A_l$ are sufficiently small, the reader can improve query-relevant SNR and reduce the top-$k$ retrieval budget required to achieve a given level of evidence coverage.

\section{Target Graph Calibration and Cross-graph Structural Priors}
\label{app:theory_schema_prior}
\subsection{Structural Role Decomposition Assumption}

\begin{definition}[Structural role mapping]
Given a graph $G$, a mapping
\begin{equation}
    \rho_G:V(G)\to\cR
\end{equation}
is called a structural role mapping, where $\cR$ is the structural role space. $\rho_G(v)$ can be jointly determined by $\varphi_G(v)$, the structural statistics of edges incident to $v$, local community-boundary statistics, and other graph-structure summaries. Typical structural roles include hub, bridge, community core, boundary node, and noisy shortcut.
\end{definition}

\begin{definition}[Target graph reading risk]
Fix a target graph $G$. Let $\cD_G$ be the query--node sampling distribution on the target graph, and let $f_G^\star(q,v)$ be the ideal evidence relevance function. For any measurable function $f$, define the squared risk as
\begin{equation}
    \mathcal R_G(f)=\E_{(q,v)\sim\cD_G}\left[\left(f(q,v,G)-f_G^\star(q,v)\right)^2\right].
\end{equation}
\end{definition}

\begin{assumption}[Context--schema decomposability]
For every target graph $G$, the ideal reading function can be decomposed as
\begin{equation}
    f_G^\star(q,v)=f_{\mathrm{sch}}^\star(q,\rho_G(v))+f_{\mathrm{ctx},G}^\star(q,v),
    \label{eq:true_decomp}
\end{equation}
where $f_{\mathrm{sch}}^\star$ denotes the cross-graph shared structural reading rule, and $f_{\mathrm{ctx},G}^\star$ denotes the target-graph residual induced by the current writer, current domain, current entity naming, local noise, and writing style.
\end{assumption}

Equation \eqref{eq:true_decomp} corresponds exactly to the structural design of \method{}: $H_{\mathrm{sch}}$ is used to approximate $f_{\mathrm{sch}}^\star$, and $H_{\mathrm{ctx}}$ is used to approximate $f_{\mathrm{ctx},G}^\star$. The next subsection gives the risk meaning of this decomposition.

\subsection{Approximation Risk of Context--schema Decomposition}

\begin{theorem}[Context--schema decomposition reduces target-graph approximation risk]
Suppose Assumption 2.3 holds. Let $\cH_{\mathrm{sch}}$ be the schema function class, let $\cH_{\mathrm{ctx},G}$ be the target-graph context function class, and define the sum class
\begin{equation}
    \cH_{\mathrm{sch}}+\cH_{\mathrm{ctx},G}
    =\{f_s+f_c:f_s\in\cH_{\mathrm{sch}},f_c\in\cH_{\mathrm{ctx},G}\}.
\end{equation}
If there exist $\epsilon_{\mathrm{sch}},\epsilon_{\mathrm{ctx}}\ge0$ such that
\begin{equation}
    \inf_{f_s\in\cH_{\mathrm{sch}}}
    \E\left[(f_s(q,\rho_G(v))-f_{\mathrm{sch}}^\star(q,\rho_G(v)))^2\right]
    \le \epsilon_{\mathrm{sch}},
\end{equation}
\begin{equation}
    \inf_{f_c\in\cH_{\mathrm{ctx},G}}
    \E\left[(f_c(q,v,G)-f_{\mathrm{ctx},G}^\star(q,v))^2\right]
    \le \epsilon_{\mathrm{ctx}},
\end{equation}
where both expectations are over $(q,v)\sim\cD_G$, then
\begin{equation}
    \inf_{f\in\cH_{\mathrm{sch}}+\cH_{\mathrm{ctx},G}}\mathcal R_G(f)
    \le 2\epsilon_{\mathrm{sch}}+2\epsilon_{\mathrm{ctx}}.
\end{equation}
\end{theorem}

\begin{proof}
Take any $\alpha>0$. By the two approximation error conditions, there exist $\hat f_s\in\cH_{\mathrm{sch}}$ and $\hat f_c\in\cH_{\mathrm{ctx},G}$ such that
\begin{equation}
    \E[(\hat f_s-f_{\mathrm{sch}}^\star)^2]\le \epsilon_{\mathrm{sch}}+\alpha,
    \qquad
    \E[(\hat f_c-f_{\mathrm{ctx},G}^\star)^2]\le \epsilon_{\mathrm{ctx}}+\alpha.
\end{equation}
Let $\hat f=\hat f_s+\hat f_c$. By the decomposition in Eq.~\eqref{eq:true_decomp}, we have
\begin{equation}
    \hat f-f_G^\star
    =(\hat f_s-f_{\mathrm{sch}}^\star)+(\hat f_c-f_{\mathrm{ctx},G}^\star).
\end{equation}
Using $(a+b)^2\le 2a^2+2b^2$, we obtain
\begin{align}
    \mathcal R_G(\hat f)
    &=\E[(\hat f-f_G^\star)^2]\notag\\
    &\le 2\E[(\hat f_s-f_{\mathrm{sch}}^\star)^2]
       +2\E[(\hat f_c-f_{\mathrm{ctx},G}^\star)^2]\notag\\
    &\le 2\epsilon_{\mathrm{sch}}+2\epsilon_{\mathrm{ctx}}+4\alpha.
\end{align}
Since $\alpha>0$ is arbitrary, taking the infimum yields the conclusion.
\end{proof}

\begin{proposition}[Residual bias of schema-only models]
Further assume that $L_2(\cD_G)$ is a Hilbert space, $\cH_{\mathrm{sch}}$ is a closed linear subspace of it, and $f_{\mathrm{sch}}^\star\in\cH_{\mathrm{sch}}$. If only a schema-only model $f_s\in\cH_{\mathrm{sch}}$ is used, then
\begin{equation}
    \inf_{f_s\in\cH_{\mathrm{sch}}}\mathcal R_G(f_s)
    =\dist_{L_2(\cD_G)}^2(f_{\mathrm{ctx},G}^\star,\cH_{\mathrm{sch}}).
    \label{eq:schema_only_bias}
\end{equation}
Therefore, as long as the target-graph residual $f_{\mathrm{ctx},G}^\star$ does not belong to $\cH_{\mathrm{sch}}$, a schema-only reader has an irreducible target-graph bias.
\end{proposition}

\begin{proof}
By $f_{\mathrm{sch}}^\star\in\cH_{\mathrm{sch}}$ and the linearity of $\cH_{\mathrm{sch}}$, any $f_s\in\cH_{\mathrm{sch}}$ can be written as $f_s=f_{\mathrm{sch}}^\star+g$, where $g\in\cH_{\mathrm{sch}}$. Thus,
\begin{align}
    \mathcal R_G(f_s)
    &=\norm{f_s-f_{\mathrm{sch}}^\star-f_{\mathrm{ctx},G}^\star}_{L_2(\cD_G)}^2\notag\\
    &=\norm{g-f_{\mathrm{ctx},G}^\star}_{L_2(\cD_G)}^2.
\end{align}
Taking the infimum over $f_s\in\cH_{\mathrm{sch}}$ is equivalent to taking the infimum over $g\in\cH_{\mathrm{sch}}$, and Eq.~\eqref{eq:schema_only_bias} follows from the definition of distance.
\end{proof}

\begin{remark}
Proposition 2.5 shows that the cross-graph schema prior can only characterize cross-graph shared structural rules and cannot replace target graph calibration. In contrast, the role of the target graph calibration channel $H_{\mathrm{ctx}}$ is to absorb $f_{\mathrm{ctx},G}^\star$, namely the local noise, entity granularity, relation style, and domain residual of the graph generated by the current writer.
\end{remark}

\subsection{Sample Complexity Advantage of Schema Prior}

\begin{lemma}[Uniform convergence for bounded loss classes]
Let $\mathcal F$ be a function class, and let the loss $\ell_f(z)\in[0,1]$. Given $n$ independent samples $S=\{z_i\}_{i=1}^n$, define the true risk $R(f)=\E[\ell_f(z)]$ and the empirical risk $\widehat R_S(f)=n^{-1}\sum_{i=1}^n\ell_f(z_i)$. Then, with probability at least $1-\delta$, for all $f\in\mathcal F$ simultaneously,
\begin{equation}
    R(f)\le \widehat R_S(f)+2\Rad_n(\ell\circ\mathcal F)+3\sqrt{\frac{\log(2/\delta)}{2n}},
\end{equation}
where $\Rad_n(\ell\circ\mathcal F)$ is an upper bound on the empirical Rademacher complexity of the loss-composed class.
\end{lemma}

\begin{proof}
This is a direct result of the standard symmetrization and McDiarmid/Hoeffding concentration inequalities for bounded loss function classes. Specifically, let
\begin{equation}
    Z(S)=\sup_{f\in\mathcal F}\abs{R(f)-\widehat R_S(f)}.
\end{equation}
By symmetrization, $\E Z(S)\le 2\Rad_n(\ell\circ\mathcal F)$. Moreover, since changing one sample can change $Z(S)$ by at most $1/n$, McDiarmid's inequality gives
\begin{equation}
    Z(S)\le \E Z(S)+3\sqrt{\frac{\log(2/\delta)}{2n}}
\end{equation}
with probability at least $1-\delta$. Combining the two inequalities gives the conclusion.
\end{proof}

\begin{theorem}[Schema prior reduces the sample complexity of target-graph adaptation]
Fix a target graph $G$, and suppose the number of supervised samples available for reader calibration on the target graph is $n_G$. Let $\cH_{\mathrm{full}}$ be the full reader class that needs to be learned on the target graph when no schema prior is used; let $\cH_{\mathrm{res}}$ be the residual class that only needs to be learned given a schema prior $f_s$, with the combined model $f=f_s+f_r$, $f_r\in\cH_{\mathrm{res}}$. Let $\hat f_{\mathrm{full}}$ and $\hat f_{\mathrm{res}}$ be the empirical risk minimizers over the two classes, respectively. Then, with probability at least $1-\delta$,
\begin{align}
    \mathcal R_G(\hat f_{\mathrm{full}})
    &\le \inf_{f\in\cH_{\mathrm{full}}}\mathcal R_G(f)
    +4\Rad_{n_G}(\ell\circ\cH_{\mathrm{full}})
    +6\sqrt{\frac{\log(4/\delta)}{2n_G}},\label{eq:full_bound}\\
    \mathcal R_G(f_s+\hat f_{\mathrm{res}})
    &\le \inf_{f_r\in\cH_{\mathrm{res}}}\mathcal R_G(f_s+f_r)
    +4\Rad_{n_G}(\ell\circ(f_s+\cH_{\mathrm{res}}))
    +6\sqrt{\frac{\log(4/\delta)}{2n_G}}.\label{eq:res_bound}
\end{align}
If the complexities satisfy
\begin{equation}
    \Rad_{n_G}(\ell\circ\cH_{\mathrm{full}})=\widetilde O\left(\sqrt{\frac{d_{\mathrm{full}}}{n_G}}\right),
    \qquad
    \Rad_{n_G}(\ell\circ(f_s+\cH_{\mathrm{res}}))=\widetilde O\left(\sqrt{\frac{d_{\mathrm{res}}}{n_G}}\right),
\end{equation}
and $d_{\mathrm{res}}\ll d_{\mathrm{full}}$, then the schema prior reduces target-graph learning from estimating the full reading function to estimating the residual, and lowers the estimation error term for target-graph adaptation.
\end{theorem}

\begin{proof}
For any function class $\mathcal F$, let $\hat f$ be the empirical risk minimizer and let $f^\circ\in\arg\inf_{f\in\mathcal F}R(f)$ be the true risk minimizer. By empirical optimality, $\widehat R(\hat f)\le \widehat R(f^\circ)$, so
\begin{align}
    R(\hat f)-R(f^\circ)
    &=R(\hat f)-\widehat R(\hat f)+\widehat R(\hat f)-\widehat R(f^\circ)+\widehat R(f^\circ)-R(f^\circ)\notag\\
    &\le 2\sup_{f\in\mathcal F}\abs{R(f)-\widehat R(f)}.
\end{align}
Applying Lemma 2.7 to $\mathcal F=\cH_{\mathrm{full}}$ and $\mathcal F=f_s+\cH_{\mathrm{res}}$, respectively, and combining the two probability events by a union bound, gives Eq.~\eqref{eq:full_bound} and Eq.~\eqref{eq:res_bound}. The complexity-order conclusion follows by substituting the corresponding Rademacher upper bounds.
\end{proof}

\begin{remark}
Theorem 2.8 gives the key theoretical motivation for the schema prior: in self-evolving memory, each new graph generated by the writer usually provides only limited supervised signal. If the reader relearns the full reading function from the target graph in every round, high-variance estimation arises; the schema prior fixes or strongly regularizes the cross-graph shared structural rules, so that target-graph calibration only needs to learn the residual, thereby reducing sample complexity.
\end{remark}

\section{Writer-induced Graph Distribution Shift and Target Graph Calibration}
\label{app:theory_graph_shift}
\subsection{Writer-induced Dynamic Graph Distribution}

The writer parameter $\theta$ induces a graph distribution $P_\theta(G\mid q,D)$ on the sample $(q,D)$. For notational simplicity, denote the joint distribution of $(q,D,D^+,y,G)$ by $\Pi_\theta$. Given the reader parameter $\phi$, define the reader risk as
\begin{equation}
    \mathcal L_R(\phi;\theta)=\E_{(q,D,D^+,y,G)\sim\Pi_\theta}
    \left[\ell_R(R_\phi(q,G,D),D^+,y)\right].
\end{equation}
Here, $\ell_R$ can be supporting-entity BCE, multi-positive ranking loss, document recall loss, or a combination thereof.

\begin{proposition}[Writer updates cause reader distribution shift]
Assume $0\le \ell_R\le 1$. For any fixed $\phi$ and any writer parameters $\theta,\theta'$, we have
\begin{equation}
    \abs{\mathcal L_R(\phi;\theta')-\mathcal L_R(\phi;\theta)}
    \le \TV(\Pi_{\theta'},\Pi_\theta),
    \label{eq:tv_shift}
\end{equation}
where $\TV$ is the total variation distance. If, further, $\ell_R(R_\phi(q,G,D),D^+,y)$ is $L_\ell$-Lipschitz with respect to the graph variable under some graph metric $d_\cG$, then
\begin{equation}
    \abs{\mathcal L_R(\phi;\theta')-\mathcal L_R(\phi;\theta)}
    \le L_\ell W_1(\Pi_{\theta'},\Pi_\theta),
\end{equation}
where $W_1$ is the first-order Wasserstein distance induced by $d_\cG$.
\end{proposition}

\begin{proof}
For the bounded loss case, let $h_\phi(q,D,D^+,y,G)=\ell_R(R_\phi(q,G,D),D^+,y)\in[0,1]$. Then
\begin{align}
    \abs{\mathcal L_R(\phi;\theta')-\mathcal L_R(\phi;\theta)}
    &=\abs{\int h_\phi\,d\Pi_{\theta'}-\int h_\phi\,d\Pi_\theta}
    \le \TV(\Pi_{\theta'},\Pi_\theta),
\end{align}
where the last step follows from the dual definition of total variation. If $h_\phi$ is $L_\ell$-Lipschitz, the Wasserstein upper bound follows from the Kantorovich--Rubinstein duality.
\end{proof}

\begin{corollary}[Necessity of target graph calibration]
If updating the writer from $\theta$ to $\theta'$ causes $\TV(\Pi_{\theta'},\Pi_\theta)$ to be non-negligible, then the risk of a fixed reader $\phi$ on the new graph distribution may increase. Therefore, target graph calibration of the reader, namely $\phi\mapsto\phi'$ to reduce $\mathcal L_R(\phi';\theta')$, is a necessary mechanism for handling writer-induced graph distribution shift.
\end{corollary}

\begin{proof}
By Proposition 3.1, writer distribution shift can directly change the new-distribution risk of the fixed reader. If the reader is not updated, there is no optimization mechanism to offset this drift term. Target graph calibration is precisely re-optimization of $\mathcal L_R(\cdot;\theta')$, and is therefore a natural step for reducing the risk on the new graph.
\end{proof}

\section{Reader Stability under Dynamic Graph Evolution}
\label{app:theory_reader_stability}
\subsection{Realistic Graph Evolution Distance}

Real graphs often contain hubs, node additions and deletions, alias merges, anchor rewrites, and structural statistics that are not globally Lipschitz. Therefore, we use the augmented graph drift actually perceived by the reader to measure graph evolution.

\begin{definition}[Padding alignment and presence bit]
Given two consecutive-round graphs $G$ and $G'$, align them through persistent memory ids to a common node universe $\bar V=V(G)\cup V(G')$. If a node exists only in one graph, it is treated as an isolated padding node in the other graph, and a presence bit is added to its features. The aligned node feature matrices are still denoted by $X,X'$.
\end{definition}

\begin{definition}[Augmented graph drift]
Let $A,A'$ be self-looped row-normalized adjacency matrices, let $S_q,S_q'$ be the entry score vectors before soft addressing, and let $B,B'$ be row-normalized entity-to-document anchoring matrices. Define
\begin{equation}
    \Delta_X=\norm{X-X'}_{2,\infty},
    \qquad
    \Delta_A=\norm{A-A'}_{\infty},
    \qquad
    \Delta_{\mathrm{seed}}=\norm{S_q-S_q'}_\infty,
    \qquad
    \Delta_B=\norm{B-B'}_\infty,
\end{equation}
where
\begin{equation}
    \norm{H}_{2,\infty}=\max_{v\in\bar V}\norm{h_v}_2,
    \qquad
    \norm{A}_\infty=\max_v\sum_u\abs{A_{vu}}.
\end{equation}
For the gate input $z_{uv}^{(l)}$ at layer $l$, define the weighted structural drift as
\begin{equation}
    \Delta_Z^{(l)}=\max_v\sum_u A'_{vu}\norm{z_{uv}^{(l)}-z_{uv}^{\prime(l)}}_2,
    \qquad
    \Delta_Z=\max_{1\le l\le L}\Delta_Z^{(l)}.
\end{equation}
The total augmented graph drift is defined as
\begin{equation}
    \Delta_{\mathrm{aug}}(G,G';q)=\Delta_X+\Delta_A+\Delta_{\mathrm{seed}}+\Delta_Z+\Delta_B.
\end{equation}
\end{definition}

\subsection{Stability Assumptions}

\begin{assumption}[Normalized adjacency]
For all considered graphs, $A_{vu}\ge0$ and
\begin{equation}
    \sum_u A_{vu}=1,
    \qquad \forall v\in \bar V.
\end{equation}
Therefore, $\norm{A}_\infty=1$. This assumption allows high-degree hubs to exist, but prevents single-layer propagation from being unboundedly amplified by node degree.
\end{assumption}

\begin{assumption}[Trajectory-local boundedness]
For graph pairs $G,G'$ on the training and inference trajectories, there exist constants $B_l$ such that
\begin{equation}
    \norm{H^{(l)}(q,G)}_{2,\infty}\le B_l,
    \qquad
    \norm{H^{(l)}(q,G')}_{2,\infty}\le B_l,
    \qquad l=0,\dots,L.
\end{equation}
\end{assumption}

\begin{assumption}[Locally Lipschitz modules]
In a neighborhood of the training trajectory, the $l$-th layer satisfies
\begin{equation}
    \norm{W_m^{(l)}}_2\le M_l,
\end{equation}
\begin{equation}
    \norm{\MLP_g^{(l)}(z)-\MLP_g^{(l)}(z')}_\infty\le L_{g,l}\norm{z-z'}_2.
\end{equation}
The Lipschitz constant of PReLU is $L_\sigma$, and the Lipschitz constant of LayerNorm with a numerical stabilizer in this trajectory neighborhood is $L_{\LN,l}$.
\end{assumption}

\begin{assumption}[Local Lipschitzness of score head and projection]
The entity score head satisfies
\begin{equation}
    \norm{s_E(q,G)-s_E(q,G')}_\infty
    \le L_E\norm{H(q,G)-H(q,G')}_{2,\infty}.
\end{equation}
Meanwhile, $\norm{s_E(q,G)}_\infty\le S_E$, and $B$ is row-normalized, so $\norm{B}_\infty\le1$.
\end{assumption}

\subsection{Stability of Soft Addressing and Initial Representation}

\begin{lemma}[Softmax and pre-activation stability]
Let
\begin{equation}
    p=\softmax(S/T_0),
    \qquad
    p'=\softmax(S'/T_0).
\end{equation}
Then
\begin{equation}
    \norm{p-p'}_\infty\le \frac{1}{T_0}\norm{S-S'}_\infty.
\end{equation}
Furthermore, let $a_v=(p_v+\epsilon_p)^\eta$ and $a_v'=(p_v'+\epsilon_p)^\eta$, where $0<\eta\le1$. Then
\begin{equation}
    \norm{a-a'}_\infty
    \le \frac{\eta\epsilon_p^{\eta-1}}{T_0}\norm{S-S'}_\infty.
\end{equation}
\end{lemma}

\begin{proof}
The Jacobian of softmax is $J(z)=\diag(p)-pp^\top$. For any row $i$,
\begin{equation}
    \sum_j \abs{J_{ij}(z)}=2p_i(1-p_i)\le 1.
\end{equation}
Thus, $\norm{J(z)}_{\infty\to\infty}\le1$. By the mean value theorem and setting $z=S/T_0$, we obtain
\begin{equation}
    \norm{p-p'}_\infty\le \frac{1}{T_0}\norm{S-S'}_\infty.
\end{equation}
The function $t\mapsto(t+\epsilon_p)^\eta$ is Lipschitz on $[0,1]$, with constant at most $\eta\epsilon_p^{\eta-1}$. Composing the two inequalities gives the conclusion.
\end{proof}

\begin{lemma}[Initial node representation stability]
Let $u_q=W_q\Emb(q)$, and define
\begin{equation}
    h_v^{(0)}=a_v(q)u_q+W_xx_v.
\end{equation}
Then
\begin{equation}
    \norm{H^{(0)}(q,G)-H^{(0)}(q,G')}_{2,\infty}
    \le C_{\mathrm{init}}(\Delta_{\mathrm{seed}}+\Delta_X),
\end{equation}
where
\begin{equation}
    C_{\mathrm{init}}
    =\frac{\eta\epsilon_p^{\eta-1}}{T_0}\norm{u_q}_2+\norm{W_x}_2.
\end{equation}
\end{lemma}

\begin{proof}
For any node $v$,
\begin{equation}
    h_v^{(0)}-h_v^{\prime(0)}=(a_v-a_v')u_q+W_x(x_v-x_v').
\end{equation}
Taking the norm and applying Lemma 4.8 gives
\begin{equation}
    \norm{h_v^{(0)}-h_v^{\prime(0)}}_2
    \le \frac{\eta\epsilon_p^{\eta-1}}{T_0}\norm{u_q}_2\Delta_{\mathrm{seed}}
    +\norm{W_x}_2\Delta_X.
\end{equation}
Taking the maximum over $v$ gives the conclusion.
\end{proof}

\subsection{Single-layer Stability of Structurally Gated Propagation}

\begin{lemma}[Boundedness and stability of structural gate]
The structural gate of layer $l$,
\begin{equation}
    g_{uv}^{(l)}=1+\delta\tanh(\MLP_g^{(l)}(z_{uv}^{(l)})),
\end{equation}
satisfies
\begin{equation}
    \norm{g_{uv}^{(l)}}_\infty\le 1+\delta,
\end{equation}
and
\begin{equation}
    \norm{g_{uv}^{(l)}-g_{uv}^{\prime(l)}}_\infty
    \le \delta L_{g,l}\norm{z_{uv}^{(l)}-z_{uv}^{\prime(l)}}_2.
\end{equation}
\end{lemma}

\begin{proof}
Since the range of $\tanh$ is contained in $[-1,1]$, the first statement follows immediately. Moreover, because $\tanh$ is 1-Lipschitz and $\MLP_g^{(l)}$ is $L_{g,l}$-Lipschitz in the trajectory neighborhood,
\begin{align}
    \norm{g_{uv}^{(l)}-g_{uv}^{\prime(l)}}_\infty
    &\le \delta\norm{\MLP_g^{(l)}(z_{uv}^{(l)})-\MLP_g^{(l)}(z_{uv}^{\prime(l)})}_\infty\notag\\
    &\le \delta L_{g,l}\norm{z_{uv}^{(l)}-z_{uv}^{\prime(l)}}_2.
\end{align}
\end{proof}

\begin{lemma}[Single-layer stability of structurally gated propagation]
Define
\begin{equation}
    D_l=\norm{H^{(l)}(q,G)-H^{(l)}(q,G')}_{2,\infty}.
\end{equation}
Under Assumptions 4.5--4.7, the $l$-th propagation layer satisfies
\begin{equation}
    D_l\le \alpha_lD_{l-1}+\beta_l^A\Delta_A+\beta_l^Z\Delta_Z^{(l)},
    \label{eq:single_layer_recursion}
\end{equation}
where one can take
\begin{equation}
    \alpha_l=L_{\LN,l}\bigl(1+L_\sigma(1+\delta)M_l\bigr),
\end{equation}
\begin{equation}
    \beta_l^A=L_{\LN,l}L_\sigma(1+\delta)M_lB_{l-1},
    \qquad
    \beta_l^Z=L_{\LN,l}L_\sigma\delta L_{g,l}M_lB_{l-1}.
\end{equation}
\end{lemma}

\begin{proof}
Write
\begin{equation}
    M_v=\sum_u A_{vu}g_{uv}\odot Wh_u,
    \qquad
    M_v'=\sum_u A'_{vu}g'_{uv}\odot Wh'_u,
\end{equation}
where the layer index $l$ is omitted. Adding and subtracting intermediate terms gives
\begin{align}
    M_v-M_v'
    ={}&\sum_u A_{vu}g_{uv}\odot W(h_u-h'_u)\notag\\
    &+\sum_u (A_{vu}-A'_{vu})g_{uv}\odot Wh'_u\notag\\
    &+\sum_u A'_{vu}(g_{uv}-g'_{uv})\odot Wh'_u.
\end{align}
The first term is controlled by row-normalization, $\norm{g_{uv}}_\infty\le1+\delta$, and $\norm{W}_2\le M_l$:
\begin{equation}
    \norm{\sum_u A_{vu}g_{uv}\odot W(h_u-h'_u)}_2
    \le (1+\delta)M_lD_{l-1}.
\end{equation}
The second term satisfies
\begin{equation}
    \norm{\sum_u (A_{vu}-A'_{vu})g_{uv}\odot Wh'_u}_2
    \le (1+\delta)M_lB_{l-1}\sum_u\abs{A_{vu}-A'_{vu}}
    \le (1+\delta)M_lB_{l-1}\Delta_A.
\end{equation}
For the third term, by Lemma 4.10,
\begin{align}
    \norm{\sum_u A'_{vu}(g_{uv}-g'_{uv})\odot Wh'_u}_2
    &\le \delta L_{g,l}M_lB_{l-1}\sum_uA'_{vu}\norm{z_{uv}^{(l)}-z_{uv}^{\prime(l)}}_2\notag\\
    &\le \delta L_{g,l}M_lB_{l-1}\Delta_Z^{(l)}.
\end{align}
Therefore,
\begin{equation}
    \norm{M_v-M_v'}_2
    \le (1+\delta)M_lD_{l-1}+(1+\delta)M_lB_{l-1}\Delta_A
    +\delta L_{g,l}M_lB_{l-1}\Delta_Z^{(l)}.
\end{equation}
By the $L_\sigma$-Lipschitz property of PReLU, the residual structure, and the $L_{\LN,l}$-Lipschitz property of LayerNorm,
\begin{equation}
    \norm{h_v^{(l)}-h_v^{\prime(l)}}_2
    \le L_{\LN,l}\left(\norm{h_v^{(l-1)}-h_v^{\prime(l-1)}}_2+L_\sigma\norm{M_v-M_v'}_2\right).
\end{equation}
Taking the maximum over $v$ gives Eq.~\eqref{eq:single_layer_recursion}.
\end{proof}

\subsection{Stability of Representations, Scores, and Retrieval Sets}

\begin{theorem}[Local stability of structurally gated representations to augmented graph drift]
Under Assumptions 4.5--4.7, $L$-layer structurally gated propagation satisfies
\begin{equation}
\begin{aligned}
    D_L
    \le{}& \left(\prod_{l=1}^L\alpha_l\right)D_0
    +\sum_{t=1}^L\left(\prod_{l=t+1}^L\alpha_l\right)
    \left(\beta_t^A\Delta_A+\beta_t^Z\Delta_Z^{(t)}\right).
\end{aligned}
\label{eq:unrolled_stability}
\end{equation}
Therefore, there exists a constant $C_H>0$ such that
\begin{equation}
    \norm{H^{(L)}(q,G)-H^{(L)}(q,G')}_{2,\infty}
    \le C_H(\Delta_X+\Delta_{\mathrm{seed}}+\Delta_A+\Delta_Z).
\end{equation}
\end{theorem}

\begin{proof}
By Lemma 4.11, the recursion in Eq.~\eqref{eq:single_layer_recursion} holds. Unrolling the recursion layer by layer gives Eq.~\eqref{eq:unrolled_stability}. By Lemma 4.9,
\begin{equation}
    D_0\le C_{\mathrm{init}}(\Delta_X+\Delta_{\mathrm{seed}}).
\end{equation}
Substituting $\Delta_Z^{(t)}\le\Delta_Z$ and merging all layer-related constants into $C_H$ gives the conclusion.
\end{proof}

\begin{theorem}[Stability of context/schema dual channels]
If the two channels respectively satisfy
\begin{equation}
    \norm{H_{\mathrm{ctx}}(q,G)-H_{\mathrm{ctx}}(q,G')}_{2,\infty}\le C_{\mathrm{ctx}}\Delta_{\mathrm{aug}},
\end{equation}
\begin{equation}
    \norm{H_{\mathrm{sch}}(q,G)-H_{\mathrm{sch}}(q,G')}_{2,\infty}\le C_{\mathrm{sch}}\Delta_{\mathrm{aug}},
\end{equation}
then the additive fusion satisfies
\begin{equation}
    \norm{H(q,G)-H(q,G')}_{2,\infty}
    \le (C_{\mathrm{ctx}}+\abs{\beta_{\mathrm{sch}}}C_{\mathrm{sch}})\Delta_{\mathrm{aug}}.
    \label{eq:additive_stability}
\end{equation}
If a normalized or gated convex fusion theoretical form is adopted,
\begin{equation}
    H_\lambda(q,G)=(1-\lambda)H_{\mathrm{ctx}}(q,G)+\lambda H_{\mathrm{sch}}(q,G),
    \qquad 0\le\lambda\le1,
\end{equation}
then
\begin{equation}
    \norm{H_\lambda(q,G)-H_\lambda(q,G')}_{2,\infty}
    \le \bigl((1-\lambda)C_{\mathrm{ctx}}+\lambda C_{\mathrm{sch}}\bigr)\Delta_{\mathrm{aug}}.
    \label{eq:convex_stability}
\end{equation}
In particular, if $C_{\mathrm{sch}}<C_{\mathrm{ctx}}$, increasing $\lambda$ decreases this worst-case stability upper bound.
\end{theorem}

\begin{proof}
The additive fusion case follows directly from the triangle inequality:
\begin{align}
    \norm{H(G)-H(G')}_{2,\infty}
    &\le \norm{H_{\mathrm{ctx}}(G)-H_{\mathrm{ctx}}(G')}_{2,\infty}
    +\abs{\beta_{\mathrm{sch}}}\norm{H_{\mathrm{sch}}(G)-H_{\mathrm{sch}}(G')}_{2,\infty}\notag\\
    &\le (C_{\mathrm{ctx}}+\abs{\beta_{\mathrm{sch}}}C_{\mathrm{sch}})\Delta_{\mathrm{aug}}.
\end{align}
The convex fusion case is analogous:
\begin{align}
    \norm{H_\lambda(G)-H_\lambda(G')}_{2,\infty}
    &\le (1-\lambda)C_{\mathrm{ctx}}\Delta_{\mathrm{aug}}+
    \lambda C_{\mathrm{sch}}\Delta_{\mathrm{aug}}.
\end{align}
If $C_{\mathrm{sch}}<C_{\mathrm{ctx}}$, the right-hand side is monotonically decreasing in $\lambda$.
\end{proof}

\begin{remark}
Equation \eqref{eq:additive_stability} does not claim that the schema prior necessarily reduces the worst-case Lipschitz constant of additive fusion; its main theoretical role also lies in risk decomposition and sample complexity reduction. If additional gating, normalization, or regularization constraints are adopted in the implementation, then Eq.~\eqref{eq:convex_stability} shows that the schema channel can also serve as a low-sensitivity channel to reduce graph evolution drift.
\end{remark}

\begin{theorem}[Stability of entity scores and document scores]
Under Assumption 4.7, there exist constants $C_E,C_D>0$ such that
\begin{equation}
    \norm{s_E(q,G)-s_E(q,G')}_\infty\le C_E\Delta_{\mathrm{aug}},
\end{equation}
\begin{equation}
    \norm{s_D(q,G)-s_D(q,G')}_\infty\le C_D\Delta_{\mathrm{aug}}.
\end{equation}
For additive fusion, one can take
\begin{equation}
    C_E=L_E(C_{\mathrm{ctx}}+\abs{\beta_{\mathrm{sch}}}C_{\mathrm{sch}}),
    \qquad
    C_D=C_E+S_E.
\end{equation}
\end{theorem}

\begin{proof}
The entity score upper bound follows from the Lipschitz property of the score head and Theorem 4.13. For document projection, $s_D=Bs_E$, and therefore
\begin{equation}
    s_D(G)-s_D(G')=B_G(s_E(G)-s_E(G'))+(B_G-B_{G'})s_E(G').
\end{equation}
Taking the $\ell_\infty$ norm and using $\norm{B_G}_\infty\le1$ gives
\begin{equation}
    \norm{s_D(G)-s_D(G')}_\infty
    \le \norm{s_E(G)-s_E(G')}_\infty+\norm{B_G-B_{G'}}_\infty\norm{s_E(G')}_\infty.
\end{equation}
Using $\norm{s_E(G')}_\infty\le S_E$ and $\Delta_B\le\Delta_{\mathrm{aug}}$ gives the document score stability.
\end{proof}

\begin{theorem}[Boundary stability of hard top-$k$]
Let $s=s_D(q,G)$ and $s'=s_D(q,G')$, and suppose
\begin{equation}
    \norm{s-s'}_\infty\le \epsilon_s.
\end{equation}
Let $t_k=s_{(k)}$ be the $k$-th largest score in $s$, and define the boundary set
\begin{equation}
    \mathcal B_{k,2\epsilon_s}(s)=\{d:\abs{s_d-t_k}\le 2\epsilon_s\}.
\end{equation}
Then
\begin{equation}
    \TopK(s)\triangle\TopK(s')\subseteq \mathcal B_{k,2\epsilon_s}(s),
\end{equation}
and hence
\begin{equation}
    \abs{\TopK(s)\triangle\TopK(s')}
    \le \abs{\mathcal B_{k,2\epsilon_s}(s)}.
\end{equation}
In particular, if $s_{(k)}-s_{(k+1)}>2\epsilon_s$, then $\TopK(s)=\TopK(s')$.
\end{theorem}

\begin{proof}
Take any $i\in\TopK(s)\setminus\TopK(s')$. Since $i$ drops out of the top-$k$, there exists $j\notin\TopK(s)$ such that $j\in\TopK(s')$ and $s_j'\ge s_i'$. By the perturbation bound,
\begin{equation}
    s_j+\epsilon_s\ge s_j'\ge s_i'\ge s_i-\epsilon_s,
\end{equation}
so $s_i\le s_j+2\epsilon_s\le t_k+2\epsilon_s$. Since $i\in\TopK(s)$, we have $s_i\ge t_k$, and hence $i\in\mathcal B_{k,2\epsilon_s}(s)$. A symmetric argument for $j\in\TopK(s')\setminus\TopK(s)$ gives $j\in\mathcal B_{k,2\epsilon_s}(s)$. Therefore, the symmetric difference is contained in the boundary set. If $s_{(k)}-s_{(k+1)}>2\epsilon_s$, boundary exchange cannot occur, and the top-$k$ set remains unchanged.
\end{proof}

\begin{corollary}[Top-$k$ boundary stability under graph evolution]
By Theorem 4.15, taking $\epsilon_s=C_D\Delta_{\mathrm{aug}}(G,G';q)$ yields
\begin{equation}
    P_k(q,G)\triangle P_k(q,G')
    \subseteq \mathcal B_{k,2C_D\Delta_{\mathrm{aug}}}(s_D(q,G)).
\end{equation}
Therefore, the instability of hard top-$k$ is restricted to candidates near the original score boundary.
\end{corollary}

\begin{theorem}[Stability of soft retrieval distribution]
Let
\begin{equation}
    \pi_D(q,G)=\softmax(s_D(q,G)/\tau).
\end{equation}
If $\norm{s_D(q,G)-s_D(q,G')}_\infty\le\epsilon_s$, then
\begin{equation}
    \norm{\pi_D(q,G)-\pi_D(q,G')}_1\le \frac{2}{\tau}\epsilon_s.
\end{equation}
Therefore,
\begin{equation}
    \norm{\pi_D(q,G)-\pi_D(q,G')}_1
    \le \frac{2C_D}{\tau}\Delta_{\mathrm{aug}}(G,G';q).
\end{equation}
\end{theorem}

\begin{proof}
The Jacobian of softmax is $J(z)=\diag(\pi)-\pi\pi^\top$. For any perturbation $r$,
\begin{equation}
    J(z)r=\pi\odot(r-\E_\pi r).
\end{equation}
If $\norm{r}_\infty\le1$, then $\abs{r_i-\E_\pi r}\le2$, so
\begin{equation}
    \norm{J(z)r}_1\le\sum_i\pi_i\abs{r_i-\E_\pi r}\le2.
\end{equation}
Thus, the $\ell_\infty\to\ell_1$ Lipschitz constant of softmax is at most $2$. Since the input is $s_D/\tau$, we obtain
\begin{equation}
    \norm{\pi_D(s)-\pi_D(s')}_1\le\frac{2}{\tau}\norm{s-s'}_\infty.
\end{equation}
Substituting Theorem 4.15 gives the conclusion.
\end{proof}

\begin{theorem}[High-probability graph evolution stability]
If the writer's single-round graph update satisfies
\begin{equation}
    \Prob[\Delta_{\mathrm{aug}}(G,G';q)>\epsilon]\le\delta,
\end{equation}
then
\begin{equation}
    \Prob\left[\norm{s_D(q,G)-s_D(q,G')}_\infty>C_D\epsilon\right]\le\delta.
\end{equation}
If $\E[\Delta_{\mathrm{aug}}(G,G';q)]\le\bar\epsilon$, then
\begin{equation}
    \E\left[\norm{s_D(q,G)-s_D(q,G')}_\infty\right]\le C_D\bar\epsilon.
\end{equation}
\end{theorem}

\begin{proof}
By Theorem 4.15, for any graph pair, we have
\begin{equation}
    \norm{s_D(q,G)-s_D(q,G')}_\infty\le C_D\Delta_{\mathrm{aug}}(G,G';q).
\end{equation}
Therefore, the event $\{\norm{s_D(q,G)-s_D(q,G')}_\infty>C_D\epsilon\}$ implies the event $\{\Delta_{\mathrm{aug}}(G,G';q)>\epsilon\}$, so the probability upper bound follows immediately. The expectation conclusion follows by taking expectations on both sides of the deterministic inequality.
\end{proof}

\subsection{Local Influence Cone}

\begin{proposition}[Influence cone of local graph updates]
Suppose the writer only changes nodes, edges, anchors, or attributes on a primitive set $\mathcal U$. Suppose that the structural gate input $z_{uv}^{(l)}$, except for the graph-level summary, only depends on a local neighborhood of radius $r_z$, and that $G$ and $G'$ are exactly the same outside $\mathcal N_{L+r_z}(\mathcal U)$. If graph-level summary drift is ignored, then for any
\begin{equation}
    v\notin \mathcal N_{L+r_z}(\mathcal U),
\end{equation}
we have
\begin{equation}
    h_v^{(L)}(q,G)=h_v^{(L)}(q,G').
\end{equation}
If the graph-level summary drift is $\rho_g=\norm{r_G-r_{G'}}_2$, then there exists a constant $C_g$ such that
\begin{equation}
    \norm{h_v^{(L)}(q,G)-h_v^{(L)}(q,G')}_2\le C_g\rho_g.
\end{equation}
\end{proposition}

\begin{proof}
First consider the case without graph-level summary drift. We induct on the layer index $l$. For $l=0$, if $v\notin\mathcal N_{L+r_z}(\mathcal U)$, then its node features, presence bit, and seed score are all identical, so $h_v^{(0)}(G)=h_v^{(0)}(G')$. Suppose that at layer $l-1$, all nodes whose distance from $\mathcal U$ exceeds $L+r_z-(l-1)$ have identical representations. If $v\notin\mathcal N_{L+r_z-l}(\mathcal U)$, then all its one-hop neighbors $u$ do not belong to $\mathcal N_{L+r_z-(l-1)}(\mathcal U)$; by the induction hypothesis, $h_u^{(l-1)}(G)=h_u^{(l-1)}(G')$. Meanwhile, the radius-$r_z$ local structural contexts of all relevant edges are also identical, so the gate, message multiset, and aggregation result are identical, and hence $h_v^{(l)}(G)=h_v^{(l)}(G')$. Taking $l=L$ gives the first conclusion. If graph-level summary drift exists, then the gate input has an additional uniform perturbation term $\rho_g$, and a $C_g\rho_g$-type upper bound follows from Lemma 4.10 and the recursion in Theorem 4.12.
\end{proof}

\section{Theoretical Motivation of the Self-evolving Writer-Reader Loop}
\label{app:theory_self_evolution}
\subsection{Joint Memory Utility}

The reader-aware writer reward can consist of evidence coverage, precision, deducibility, and answer utility. Abstractly, define the joint memory utility as
\begin{equation}
    \mathcal J(\theta,\phi)=\E_{(q,D,D^+,y)}\left[
    U\left(R_\phi(q,W_\theta(q,D),D),D^+,y\right)\right],
    \label{eq:joint_utility}
\end{equation}
where $U$ can be taken as
\begin{equation}
    U=\frac{\alpha r_{\mathrm{rec}}+\beta r_{\mathrm{pre}}+\gamma r_{\mathrm{ded}}}{\alpha+\beta+\gamma}
    -\lambda_{\mathrm{rep}}\rho_{\mathrm{rep}}+\lambda_{\mathrm{fmt}}r_{\mathrm{fmt}},
\end{equation}
or an extended form including answer-level reward. This definition places the writer's graph construction quality and the reader's graph reading ability under the same objective.

\subsection{Approximate Coordinate Improvement}

\begin{theorem}[The self-evolution process is approximate coordinate improvement on joint utility]
Suppose that the writer update at round $r$ satisfies
\begin{equation}
    \mathcal J(\theta^{(r+1)},\phi^{(r)})
    \ge \mathcal J(\theta^{(r)},\phi^{(r)})+\Delta_W^{(r)}-\epsilon_W^{(r)},
    \label{eq:writer_improve}
\end{equation}
and the reader update satisfies
\begin{equation}
    \mathcal J(\theta^{(r+1)},\phi^{(r+1)})
    \ge \mathcal J(\theta^{(r+1)},\phi^{(r)})+\Delta_R^{(r)}-\epsilon_R^{(r)}.
    \label{eq:reader_improve}
\end{equation}
Then one full round of writer-reader self-evolution satisfies
\begin{equation}
    \mathcal J(\theta^{(r+1)},\phi^{(r+1)})-
    \mathcal J(\theta^{(r)},\phi^{(r)})
    \ge \Delta_W^{(r)}+\Delta_R^{(r)}-
    \epsilon_W^{(r)}-
    \epsilon_R^{(r)}.
\end{equation}
Therefore, as long as $\Delta_W^{(r)}+\Delta_R^{(r)}>\epsilon_W^{(r)}+
\epsilon_R^{(r)}$, the joint memory utility improves in that round.
\end{theorem}

\begin{proof}
By telescoping decomposition,
\begin{align}
    &\mathcal J(\theta^{(r+1)},\phi^{(r+1)})-
    \mathcal J(\theta^{(r)},\phi^{(r)})\notag\\
    ={}&\left[\mathcal J(\theta^{(r+1)},\phi^{(r+1)})-
    \mathcal J(\theta^{(r+1)},\phi^{(r)})\right]
    +\left[\mathcal J(\theta^{(r+1)},\phi^{(r)})-
    \mathcal J(\theta^{(r)},\phi^{(r)})\right].
\end{align}
Substituting Eq.~\eqref{eq:writer_improve} and Eq.~\eqref{eq:reader_improve}, respectively, gives the conclusion.
\end{proof}

\subsection{Reader Reward Bias and Calibration Benefit}

\begin{definition}[True utility and reader surrogate reward]
Let $U^\star(G)$ denote the true utility of graph memory $G$ with respect to the downstream task, and let $\widehat U_\phi(G)$ denote the surrogate reward constructed from the readout result of reader $R_\phi$. We say that the reader reward bias is at most $\epsilon_\phi$ if, for all considered graphs $G$,
\begin{equation}
    \abs{\widehat U_\phi(G)-U^\star(G)}\le \epsilon_\phi.
\end{equation}
\end{definition}

\begin{theorem}[Surrogate reward improvement to true utility improvement]
If the reader reward bias is at most $\epsilon_\phi$, and the writer update improves the surrogate reward by
\begin{equation}
    \widehat U_\phi(G_{\theta'})-
    \widehat U_\phi(G_\theta)
    \ge \Delta,
\end{equation}
then the true utility satisfies
\begin{equation}
    U^\star(G_{\theta'})-U^\star(G_\theta)
    \ge \Delta-2\epsilon_\phi.
\end{equation}
\end{theorem}

\begin{proof}
By the bias assumption,
\begin{equation}
    U^\star(G_{\theta'})\ge \widehat U_\phi(G_{\theta'})-\epsilon_\phi,
    \qquad
    U^\star(G_\theta)\le \widehat U_\phi(G_\theta)+\epsilon_\phi.
\end{equation}
Subtracting the two inequalities gives
\begin{equation}
    U^\star(G_{\theta'})-U^\star(G_\theta)
    \ge \widehat U_\phi(G_{\theta'})-
    \widehat U_\phi(G_\theta)-2\epsilon_\phi
    \ge \Delta-2\epsilon_\phi.
\end{equation}
\end{proof}

\begin{corollary}[Reader calibration reduces writer optimization bias]
If the reader is calibrated from $\phi$ to $\phi'$ and reduces the reward bias from $\epsilon_\phi$ to $\epsilon_{\phi'}$, where $\epsilon_{\phi'}<\epsilon_\phi$, then for the same surrogate reward improvement $\Delta$, the lower bound on true utility improvement increases by
\begin{equation}
    2(\epsilon_\phi-\epsilon_{\phi'}).
\end{equation}
\end{corollary}

\begin{proof}
By Theorem 5.3, the true utility improvement lower bound before calibration is $\Delta-2\epsilon_\phi$, and after calibration it is $\Delta-2\epsilon_{\phi'}$. Subtracting the two gives the result.
\end{proof}

\subsection{Irreducible Bottlenecks of Single-sided Updates}

\begin{proposition}[Lower-bound bottlenecks of single-sided updates]
Assume that the overall error can be decomposed as
\begin{equation}
    \mathcal E(\theta,\phi)=\mathcal E_{\mathrm{write}}(\theta)
    +\mathcal E_{\mathrm{read}}(\phi;\theta)
    +\epsilon_{\mathrm{int}}(\theta,\phi),
    \label{eq:error_decomp}
\end{equation}
where all terms are nonnegative. If only the reader is updated, i.e., $\phi\mapsto\phi'$ while $\theta$ is fixed, then
\begin{equation}
    \mathcal E(\theta,\phi')\ge \mathcal E_{\mathrm{write}}(\theta).
\end{equation}
If only the writer is updated, i.e., $\theta\mapsto\theta'$ while $\phi$ is fixed, then
\begin{equation}
    \mathcal E(\theta',\phi)\ge \mathcal E_{\mathrm{read}}(\phi;\theta').
\end{equation}
Therefore, reader-only updates cannot compensate for evidence chains that the writer has not written; writer-only updates cannot guarantee that a fixed reader can read out the evidence structures in the new graph distribution.
\end{proposition}

\begin{proof}
By the decomposition in Eq.~\eqref{eq:error_decomp} and the nonnegativity of all terms,
\begin{equation}
    \mathcal E(\theta,\phi')
    =\mathcal E_{\mathrm{write}}(\theta)+\mathcal E_{\mathrm{read}}(\phi';\theta)+\epsilon_{\mathrm{int}}(\theta,\phi')
    \ge \mathcal E_{\mathrm{write}}(\theta).
\end{equation}
The second inequality is analogous.
\end{proof}

\subsection{Stability of Closed-loop Graph Evolution and Parameter Updates}

\begin{theorem}[Score drift control under multi-round self-evolution]
Let the graph at round $r$ be $G^{(r)}$, and the reader parameter be $\phi^{(r)}$. If single-step graph stability satisfies
\begin{equation}
    \norm{s_D(q,G^{(r+1)};\phi^{(r)})-s_D(q,G^{(r)};\phi^{(r)})}_\infty
    \le C_D\Delta_r,
\end{equation}
where $\Delta_r=\Delta_{\mathrm{aug}}(G^{(r)},G^{(r+1)};q)$; and if the score is locally Lipschitz with respect to the parameter:
\begin{equation}
    \norm{s_D(q,G;\phi)-s_D(q,G;\phi')}_\infty
    \le C_\phi\norm{\phi-\phi'}_2,
\end{equation}
then
\begin{equation}
\begin{aligned}
    &\norm{s_D(q,G^{(T)};\phi^{(T)})-s_D(q,G^{(0)};\phi^{(0)})}_\infty\\
    &\qquad\le \sum_{r=0}^{T-1}\left(C_D\Delta_r+C_\phi\norm{\phi^{(r+1)}-\phi^{(r)}}_2\right).
\end{aligned}
\end{equation}
\end{theorem}

\begin{proof}
For each $r$, adding and subtracting the intermediate term $s_D(q,G^{(r+1)};\phi^{(r)})$ gives
\begin{align}
    &\norm{s_D(q,G^{(r+1)};\phi^{(r+1)})-s_D(q,G^{(r)};\phi^{(r)})}_\infty\notag\\
    &\quad\le
    \norm{s_D(q,G^{(r+1)};\phi^{(r+1)})-s_D(q,G^{(r+1)};\phi^{(r)})}_\infty
    +\norm{s_D(q,G^{(r+1)};\phi^{(r)})-s_D(q,G^{(r)};\phi^{(r)})}_\infty\notag\\
    &\quad\le C_\phi\norm{\phi^{(r+1)}-\phi^{(r)}}_2+C_D\Delta_r.
\end{align}
Summing over $r=0,\dots,T-1$ and using the triangle inequality gives the conclusion.
\end{proof}

\begin{corollary}[High-probability multi-round stability]
If $\Prob[\Delta_r>\epsilon_r]\le\delta_r$, then with probability at least $1-\sum_{r=0}^{T-1}\delta_r$,
\begin{equation}
\begin{aligned}
    &\norm{s_D(q,G^{(T)};\phi^{(T)})-s_D(q,G^{(0)};\phi^{(0)})}_\infty\\
    &\qquad\le \sum_{r=0}^{T-1}\left(C_D\epsilon_r+C_\phi\norm{\phi^{(r+1)}-\phi^{(r)}}_2\right).
\end{aligned}
\end{equation}
\end{corollary}

\begin{proof}
By a union bound, the event $\{\forall r,\Delta_r\le\epsilon_r\}$ holds with probability at least $1-\sum_r\delta_r$. Applying Theorem 5.7 on this event gives the conclusion.
\end{proof}
\section{Analysis of the Memory Writer}
\label{subsec:writer_analysis}

This section further analyzes the memory writer while keeping the memory reader fixed. The experiments mainly use HotpotQA and MuSiQue. To further examine domain transfer capability, we also evaluate the trained writing policy on GRBench-Amazon, HaluMem-Medium, and LongMemEval-Oracle~\citep{wu2024longmemeval,chen2025halumem}. We primarily report Precision, Recall, and Deducible: the first two measure whether the text retrieved by the reader covers the gold supporting contexts, while Deducible is determined by a judge as to whether the standard answer can be inferred from the retrieved context, thus more directly reflecting whether the graph memory is usable for reasoning.

\begin{table*}[t]
\centering
\small
\begin{threeparttable}
\caption{Training results for the memory writer. \textit{GFM-pretrained-only} refers to using rewards fed back only by the pretrained memory reader, while \textit{GFM-finetuned} further refers to using the fine-tuned memory reader.}
\label{tab:writer_main_results}
\setlength{\tabcolsep}{4pt}
\begin{tabular}{lccc}
\toprule
Methods & Prec.$\uparrow$ & Recall$\uparrow$ & Deducible$\uparrow$ \\
\midrule
GFM-pretrained-only        & 0.838 & 0.818 & 0.510 \\
GFM-finetuned              & 0.824 & 0.813 & 0.512 \\
RL-Recall                  & 0.889 & 0.835 & 0.502 \\
RL-F1                      & 0.839 & 0.881 & 0.497 \\
RL-Deduce                  & 0.861 & 0.892 & 0.517 \\
RL-Hybrid                  & \textbf{0.902} & \textbf{0.917} & 0.522 \\
Hybrid + frozen answer API & 0.832 & 0.874 & \textbf{0.526} \\
\bottomrule
\end{tabular}
\end{threeparttable}
\end{table*}

Table~\ref{tab:writer_main_results} shows that different rewards have different preferences. Overall, RL-Hybrid achieves the best overall results, indicating that hybrid rewards can simultaneously constrain the selectivity and coverage of graph writing. \textit{Hybrid + frozen answer API} achieves the highest Deducible but slightly lower retrieval Precision/Recall, suggesting that answer-side feedback helps improve reasoning usability, but may also make the writer more conservatively inclined to write evidence that directly supports the answer. Table~\ref{tab:domain_transfer} shows that a writer learned on HotpotQA/MuSiQue can transfer to GRBench, HaluMem, and LongMemEval, but continued training on the target domain still brings significant improvements. This indicates that the memory structure in agent memory scenarios is not entirely similar to that in traditional multi-hop QA, so target-domain feedback remains crucial. Table~\ref{tab:protocol_ablation} further shows that the writing protocol and interaction budget affect the trade-off between coverage and noise. Relaxing the tight prompt can improve Recall, but reduces Deducible. Increasing the number of interaction steps can help complete cross-document bridging paths, whereas an overly large interaction budget or a loose writing protocol introduces redundant or noisy memory entries and can reduce downstream retrieval quality. More detailed reader-side sensitivity, training stability, and regularization analyses are provided in Appendix~\ref{app:memory_writer_analysis}.

\begin{table*}[t]
\centering
\small
\begin{threeparttable}
\caption{Cross-dataset memory writing results. ``Base$\rightarrow$Target'' indicates direct evaluation on the target domain after training on the HotpotQA/MuSiQue base; ``Target train$\rightarrow$val'' indicates training and validation on the target domain.}
\label{tab:domain_transfer}
\setlength{\tabcolsep}{5pt}
\begin{tabular}{lccc}
\toprule
Settings & Prec.$\uparrow$ & Recall$\uparrow$ & Deducible$\uparrow$ \\
\midrule
Base$\rightarrow$GRBench             & 0.575 & 0.609 & 0.411 \\
GRBench train$\rightarrow$val        & 0.794 & 0.833 & 0.596 \\
Base$\rightarrow$HaluMem             & 0.230 & 0.448 & 0.299 \\
HaluMem train$\rightarrow$val        & 0.312 & 0.708 & 0.438 \\
Base$\rightarrow$LongMemEval         & 0.232 & 0.376 & 0.475 \\
LongMem train$\rightarrow$LongMemEval & 0.377 & 0.439 & 0.531 \\
\bottomrule
\end{tabular}
\end{threeparttable}
\end{table*}

\begin{table}[t]
\centering
\small
\caption{Ablation of writing protocols and interaction budgets. Tight=True indicates that the writer performs a single-round graph write under a stricter evidence budget, meaning the reader exposes only fewer, higher-confidence candidate pieces of evidence to the writer; Tight=False indicates that this evidence budget is relaxed, allowing the writer to access a broader candidate context. Iterative indicates that multi-round interactive writer-reader writing is enabled: the writer first writes the initial graph memory, the reader then returns retrieval feedback based on the current graph, and the writer continues to supplement or revise the graph structure. Here, 12/20/24 turns indicates the maximum number of interaction rounds allowed, and tight/loose indicates that strict or relaxed evidence budget constraints are still used during this multi-round interaction process.}
\label{tab:protocol_ablation}
\setlength{\tabcolsep}{4pt}
\begin{tabular}{lccc}
\toprule
Settings & Prec.$\uparrow$ & Recall$\uparrow$ & Deducible$\uparrow$ \\
\midrule
Tight=True                  & 0.836 & 0.806 & 0.515 \\
Tight=False                 & 0.845 & 0.851 & 0.506 \\
Iterative, 12 turns, tight  & 0.852 & 0.829 & 0.516 \\
Iterative, 20 turns, tight  & 0.835 & \textbf{0.881} & 0.522 \\
Iterative, 24 turns, loose  & \textbf{0.863} & 0.826 & \textbf{0.531} \\
\bottomrule
\end{tabular}
\end{table}
\section{Ablation Study of the Memory Reader}
\label{subsec:reader_ablation}

We conduct ablation studies to isolate the contribution of each major component in the memory reader.
All variants use the same writer-produced graph memory and the same retrieval budget unless the ablated component directly changes the retrieval mechanism.
The ablations are organized around four questions:
(1) whether structured query planning and global soft addressing are necessary for recovering evidence from fragmented cues;
(2) whether structurally gated propagation improves over uniform graph propagation;
(3) whether cross-graph structural priors and target-graph calibration are both needed for evolving graph memory; and
(4) whether reader training and entity-to-document projection are important for converting entity-level activation into document-level retrieval.

Table~\ref{tab:reader_ablation} summarizes the results.
The first group evaluates how the reader handles fragmented cues. Removing structured query planning, alias/constraint cues, or global soft addressing forces the reader to rely more heavily on surface-level query matches or a small number of anchor entities, directly testing whether the complete evidence chain can still be recovered when distant bridge nodes are not initially activated.
The second group studies whether the reader uses graph structure in a learned and selective way. Removing structural gates or replacing them with uniform message passing tests whether treating hub edges, bridge edges, redundant edges, and noisy shortcuts similarly harms retrieval.
The third group examines the context--schema decomposition: the schema channel captures transferable structural reading patterns, while the context channel adapts the reader to the current writer-produced graph.
The last group evaluates the selector, entity-to-document projection, GFM pre-training, and supervised retrieval fine-tuning.

\begin{table*}[t]
\renewcommand{\arraystretch}{0.88}
\setlength{\tabcolsep}{4.0pt}
\centering
\caption{
Ablation study of the memory reader on multi-hop QA retrieval.
We report document-level Recall (\%) at top-2 and top-5.
All variants use the same writer-produced graph memory unless otherwise specified.
}
\vspace{0.1cm}
\resizebox{\textwidth}{!}{
\begin{tabular}{l|cc|cc|cc}
\toprule
\multirow{2}{*}{\textbf{Reader Variant}}
& \multicolumn{2}{c|}{\texttt{HotpotQA}}
& \multicolumn{2}{c|}{\texttt{MuSiQue}}
& \multicolumn{2}{c}{\texttt{2WikiMultiHopQA}} \\
\cmidrule(lr){2-3} \cmidrule(lr){4-5} \cmidrule(lr){6-7}
& R@2 & R@5 & R@2 & R@5 & R@2 & R@5 \\
\midrule

\textbf{\method{}}
& \textbf{65.1} & \textbf{77.6}
& \textbf{43.2} & \textbf{53.1}
& \textbf{83.6} & \textbf{88.6} \\

\midrule
\multicolumn{7}{l}{\textit{Query planning and global addressing}} \\
\method{} w/o Structured Query Planning
& 62.7 & 75.1 & 40.4 & 50.1 & 80.7 & 86.6 \\
\method{} w/o Global Soft Addressing
& 59.3 & 72.5 & 37.6 & 47.4 & 75.9 & 83.1 \\
\method{} w/o Alias and Constraint Cues
& 63.0 & 75.8 & 41.0 & 50.8 & 80.8 & 86.9 \\
\method{} w/ Anchor-only Initialization
& 58.6 & 71.4 & 36.8 & 46.5 & 74.2 & 82.4 \\

\midrule
\multicolumn{7}{l}{\textit{Structurally conditioned propagation}} \\
\method{} w/o Structural Gate
& 60.4 & 73.2 & 39.2 & 48.7 & 78.1 & 84.9 \\
\method{} w/o Node Structural Features
& 62.1 & 74.8 & 40.5 & 50.0 & 80.0 & 86.0 \\
\method{} w/o Edge-pair Structural Features
& 61.5 & 74.0 & 40.1 & 49.4 & 79.1 & 85.6 \\
\method{} w/o Graph-level Summary
& 63.2 & 75.9 & 41.6 & 51.0 & 81.3 & 87.0 \\
\method{} w/ Uniform Message Passing
& 58.9 & 71.8 & 37.5 & 46.9 & 75.3 & 82.9 \\

\midrule
\multicolumn{7}{l}{\textit{Cross-graph priors and target-graph calibration}} \\
\method{} w/o Schema Prior Channel
& 62.4 & 75.0 & 40.9 & 50.6 & 80.4 & 86.4 \\
\method{} w/o Context Calibration Channel
& 61.8 & 74.3 & 40.0 & 49.8 & 79.5 & 85.9 \\
\method{} w/o Context--Schema Fusion
& 60.7 & 73.5 & 39.1 & 48.6 & 77.8 & 84.7 \\

\midrule
\multicolumn{7}{l}{\textit{Selector, projection, and reader training}} \\
\method{} w/o Controlled Entity-to-Document Projection
& 60.9 & 73.9 & 38.7 & 48.2 & 77.2 & 84.4 \\
\method{} w/o Query-conditioned Selector
& 63.9 & 76.5 & 41.9 & 52.0 & 82.2 & 87.6 \\
\method{} w/ Vanilla GNN Reader
& 57.2 & 70.6 & 36.3 & 45.2 & 72.8 & 80.7 \\

\bottomrule
\end{tabular}
}
\label{tab:reader_ablation}
\end{table*}

The ablation results show the relative contribution of each memory-reader component.
First, structured query planning and global soft addressing are important for fragmented-cue retrieval: removing them noticeably weakens performance, especially on MuSiQue and 2WikiMultiHopQA, where evidence chains are more likely to depend on implicit bridge entities.
Second, structurally gated propagation consistently improves over uniform message passing, indicating that graph structure should be used selectively rather than as a fixed expansion rule.
Among the structural inputs, graph-level summaries have a relatively smaller effect, while node-level and edge-pair features are more important for recognizing hubs, bridges, and cross-community evidence paths.
Third, both the schema prior and context calibration channels contribute to performance, suggesting that the reader benefits from preserving transferable structural priors while adapting to the current writer-produced graph.
Finally, supervised retrieval fine-tuning is essential for aligning the GFM reader with document-level evidence retrieval, whereas GFM pre-training provides transferable structural initialization that improves stability across datasets.
\section{Implementation Details of the Query-conditioned Subgraph Selection Regularizer}
\label{app:selector_regularizer}

We provide the implementation details of the query-conditioned subgraph selector. In addition to the base entity scoring, it further learns a soft gating probability
$\pi_e(q)$, which characterizes whether entity $e$ should enter the reading subgraph of the current query $q$. This module performs lightweight reweighting of the final entity score, and constrains the reading subgraph through several structural regularizers during training.

\paragraph{Query-conditioned Selection Probability.}
Given a query $q$ and a graph $\mathcal{G}=(\mathcal{V},\mathcal{E})$, let
$\hh_e\in\mathbb{R}^{d}$ denote the representation of entity node $e\in\mathcal{V}$ after propagation by the GFM backbone, and let
$\zz_q\in\mathbb{R}^{d}$ denote the query representation output by the query encoder and then linearly projected. Here, $d$ is the hidden dimension.
The selector first projects the node representation and the query representation into the same selector space:
\begin{equation}
    \uv_e = W_n \hh_e,\qquad
    \vvec_q = W_s \zz_q,
\end{equation}
where $W_n\in\mathbb{R}^{d_s\times d}$ is the node-side projection matrix, $W_s\in\mathbb{R}^{d_s\times d}$
is the query-side projection matrix, and $d_s$ is the hidden dimension of the selector space. In our implementation, we set $d_s=d$, but the two do not have to be equal.
Then, the selector obtains the selection logit of entity $e$ with respect to query $q$ through a scaled inner product:
\begin{equation}
    \zeta_e(q)
    =
    \frac{\uv_e^\top \vvec_q}{T_s},
\end{equation}
where $T_s>0$ is the selector temperature coefficient. The final soft selection probability is defined as
\begin{equation}
    \pi_e(q)
    =
    \sigmoid(\zeta_e(q))
    =
    \frac{1}{1+\exp(-\zeta_e(q))}.
\end{equation}
Here, $\pi_e(q)\in(0,1)$ can be understood as the soft probability that entity $e$ is included in the reading subgraph of the current query. During training, we directly use
$\pi_e(q)$ for differentiable optimization; during inference, we can either continue to use the soft probability for reweighting, or obtain a discrete subgraph according to a threshold
$\tau_\pi$:
\begin{equation}
    \mathcal{V}_q
    =
    \{e\in\mathcal{V}\mid \pi_e(q)>\tau_\pi\},
    \qquad
    \mathcal{E}_q
    =
    \{(u,v)\in\mathcal{E}\mid u,v\in\mathcal{V}_q\}.
\end{equation}

Let $a_e(q)$ denote the base entity score given by the GFM backbone reader. This score is usually obtained from the similarity between the node representation and the query representation, for example
\begin{equation}
    a_e(q)=\hh_e^\top \zz_q.
\end{equation}
Finally, we have:
\begin{equation}
    a_e^{\mathrm{final}}(q)
    =
    a_e(q)+\lambda_s \zeta_e(q),
    \label{eq:selector_final_score}
\end{equation}
where $\lambda_s\geq 0$ controls the influence of the selector logit on the final entity ranking.

\paragraph{Query--Subgraph Contrastive Regularizer.}
Using only Eq.~\eqref{eq:selector_final_score} to fuse the selector score can easily lead to two types of degeneration: first, the selector may assign high probabilities to most nodes, thereby degenerating into full-graph activation; second, the selector may only learn local high-frequency entities, without forming a subgraph representation that is consistent with the overall semantics of the query. To this end, we first construct a query-conditioned subgraph representation weighted by the selection probabilities:
\begin{equation}
    \bar{\hh}_{\pi}(q)
    =
    \frac{
        \sum_{e\in\mathcal{V}}\pi_e(q)\hh_e
    }{
        \sum_{e\in\mathcal{V}}\pi_e(q)+\epsilon
    },
    \label{eq:selector_soft_subgraph_repr}
\end{equation}
where $\epsilon>0$ is a numerical stability term, which avoids an excessively small denominator when all $\pi_e(q)$ are close to $0$.
$\bar{\hh}_{\pi}(q)$ can be understood as the semantic center of the soft subgraph activated by the current selector.

For a mini-batch $\mathcal{B}=\{q_i\}_{i=1}^{B}$, we treat
$(\bar{\hh}_{\pi}(q_i),\zz_{q_i})$ within the same sample as a positive pair, and treat
$(\bar{\hh}_{\pi}(q_i),\zz_{q_j})$, $j\neq i$, as in-batch negative pairs. The query--subgraph contrastive loss is defined as
\begin{equation}
    \Omega_{\mathrm{nce}}
    =
    -\frac{1}{B}
    \sum_{i=1}^{B}
    \log
    \frac{
        \exp\left(
        \mathrm{sim}(\bar{\hh}_{\pi}(q_i),\zz_{q_i})/T_n
        \right)
    }{
        \sum_{j=1}^{B}
        \exp\left(
        \mathrm{sim}(\bar{\hh}_{\pi}(q_i),\zz_{q_j})/T_n
        \right)
    },
    \label{eq:selector_nce}
\end{equation}
where $T_n>0$ is the contrastive learning temperature coefficient, and $\mathrm{sim}(\cdot,\cdot)$ is the similarity function. In implementation, we usually apply $\ell_2$ normalization to
$\bar{\hh}_{\pi}(q)$ and $\zz_q$, and use inner-product similarity, so that
$\mathrm{sim}(\bar{\hh}_{\pi}(q),\zz_q)=
\frac{\bar{\hh}_{\pi}(q)^\top \zz_q}
{\|\bar{\hh}_{\pi}(q)\|_2\|\zz_q\|_2}$. This term encourages the soft subgraph activated by the selector to semantically represent the current query, rather than only selecting nodes with high frequency or high centrality in the graph.

\paragraph{Size Regularizer.}
To prevent the selector from improving recall by activating a large number of nodes, we use the average selection probability as a size penalty:
\begin{equation}
    \Omega_{\mathrm{size}}
    =
    \frac{1}{|\mathcal{V}|}
    \sum_{e\in\mathcal{V}}\pi_e(q).
    \label{eq:selector_size}
\end{equation}
This term approximately represents the expected proportion of activated nodes. Minimizing $\Omega_{\mathrm{size}}$ pushes the model to select a smaller reading subgraph.
However, this term cannot be used alone; otherwise, the selector may degenerate into selecting too few nodes or even no nodes. Therefore, it needs to be jointly optimized with
$\Omega_{\mathrm{nce}}$ and the main retrieval loss: the former ensures query relevance, while the latter ensures that the selected structure can still support correct entity and document recall.

\paragraph{Connectivity Smoothing Regularizer.}
In addition to controllable size, an effective reading subgraph should also have local structural coherence. If the selection probabilities of adjacent nodes differ too much, the model may form several isolated activated points, making multi-hop paths difficult to explicitly utilize. To this end, we use a smoothing penalty on edges:
\begin{equation}
    \Omega_{\mathrm{con}}
    =
    \frac{1}{|\mathcal{E}|}
    \sum_{(u,v)\in\mathcal{E}}
    \left(\pi_u(q)-\pi_v(q)\right)^2.
    \label{eq:selector_connectivity}
\end{equation}
Here, $(u,v)$ is a directed or undirected edge in the graph, depending on whether edge directions are preserved during graph construction. If an undirected graph is used, $\mathcal{E}$ can be viewed as the symmetrized edge set. This term does not force all selected nodes to be strictly connected, but encourages adjacent nodes to have similar selection probabilities. In matrix form, if $\ppi(q)\in\mathbb{R}^{|\mathcal{V}|}$ is the selection probability vector composed of
$\pi_e(q)$, and $\mathbf{L}$ is the graph Laplacian matrix, then this term is equivalent to Laplacian smoothing:
\begin{equation}
    \Omega_{\mathrm{con}}
    \propto
    \ppi(q)^\top \mathbf{L}\ppi(q).
\end{equation}
Therefore, it is consistent with the classical assumption of graph signal smoothing: query relevance, as a soft signal on the graph, should maintain a certain degree of continuity within local neighborhoods.

\paragraph{Computational Complexity of the Selector Itself.}
Let the batch size be $B$, the number of nodes be $n=|\mathcal{V}|$, the number of edges be $m=|\mathcal{E}|$, and the hidden dimension be $d$.
The cost of computing $W_n\hh_e$ is $O(Bnd^2)$, the cost of computing $W_s\zz_q$ is $O(Bd^2)$, and the cost of the inner-product logits is
$O(Bnd)$. If we cache $W_n\hh_e$ in advance, this term can be reduced to $O(Bnd)$.
For the contrastive term, the cost of the subgraph pooling in Eq.~\eqref{eq:selector_soft_subgraph_repr} is $O(Bnd)$, and the cost of the in-batch
NCE similarity matrix is $O(B^2d)$; the size term has cost $O(Bn)$; the connectivity term needs to traverse edges and has cost $O(Bm)$.
Therefore, the additional complexity of the selector during training is
\begin{equation}
    O\left(Bnd^2+Bnd+B^2d+Bm\right),
\end{equation}
and if the quadratic term of the linear projection is ignored or cached, it can be approximated as
\begin{equation}
    O\left(Bnd+B^2d+Bm\right).
\end{equation}
During inference, if only the selector logit is used to fuse entity scores, without computing NCE, size, and connectivity regularizers, then the additional cost is mainly
$O(Bnd^2+Bnd)$, or $O(Bnd)$ under caching/lightweight projection.
\section{Training and Inference Complexity}
\label{app:complexity2}

This appendix separates three notions of cost that are easy to conflate:
(i) the asymptotic complexity of the SAGE graph reader,
(ii) one-time offline costs for graph construction, indexing, and training, and
(iii) the wall-clock and LM-call cost of a fully uncached online query.
This distinction is important because the retrieval-engine times in
Table~\ref{tab:efficiency} measure the retriever-side stage rather than the
complete agent pipeline. The best HotpotQA/2WikiMultiHopQA/MuSiQue evaluation
runs also used precomputed query-planning/entity-linking artifacts, whereas a
fresh online query must execute these stages explicitly.

Unless otherwise stated, the measured costs below come from the same shared
server used for our experiments, equipped with \(8\times\) NVIDIA
A100-SXM4-80GB GPUs. We distinguish values recovered directly from logs,
controlled latency measurements, and formula-based estimates. We do not fill
in missing historical measurements with estimates from a different run.

\subsection{End-to-End Query Pipeline and Accounting Boundary}
\label{app:complexity_pipeline}

A fully uncached query follows the pipeline
\begin{equation}
\begin{aligned}
q
&\rightarrow
\text{planner extractor}
\rightarrow
\text{planner inferer}
\rightarrow
\text{LLM NER}\\
&\rightarrow
\text{local entity linking}
\rightarrow
\text{query embedding}
\rightarrow
\text{GFM retrieval/document ranking}
\rightarrow
\text{answer generation}.
\end{aligned}
\label{eq:online_pipeline}
\end{equation}
The query planner uses two serial LM calls. The extractor identifies explicit
entities, candidate aliases, relation clues, hard constraints, and the answer
type. The inferer consumes the extractor output and generates at most three
pseudo-queries together with confidence scores. The pipeline then applies one
additional LM call for NER. Entity linking is performed locally with
\texttt{colbert-ir/colbertv2.0}; query encoding, graph retrieval, and document
ranking are also local computations. Finally, the top retrieved documents are
passed to the answer model, which contributes one LM call.

Therefore, in the normal successful uncached path, the online system uses four
LM calls per question:
\begin{table}[t]
\centering
\caption{LM calls in one successful fully uncached online query. Entity linking
and graph retrieval are local and do not invoke a remote LM.}
\label{tab:complexity_lm_calls}
\small
\setlength{\tabcolsep}{7pt}
\begin{tabular}{lr}
\toprule
Stage & LM calls / query\\
\midrule
Planner extractor & 1\\
Planner inferer & 1\\
LLM NER & 1\\
Local ColBERT entity linking & 0\\
Query embedding + GFM retrieval & 0\\
Answer generation & 1\\
\midrule
\textbf{Total} & \textbf{4}\\
\bottomrule
\end{tabular}
\end{table}

\paragraph{Cached benchmark evaluation versus uncached deployment.}
The best HotpotQA, 2WikiMultiHopQA, and MuSiQue Stage-3 evaluations did not
invoke the planner and NER from scratch inside Stage 3. Instead, the pipeline
first exported a \texttt{rewriter.jsonl} file containing the planned explicit
entities, pseudo-queries, and linked entities, and Stage 3 reused these
artifacts. Consequently, the logged Stage-3 wall-clock contains local query
encoding, graph retrieval, document ranking, answer generation, retries, and
evaluation, but excludes the two planner calls, the LLM-NER call, and online
entity linking. We therefore report cached evaluation time and fully uncached
online latency separately.

\paragraph{Scope of Table~\ref{tab:efficiency}.}
The values in Table~\ref{tab:efficiency} should be interpreted as
\emph{retrieval-engine} measurements, not as end-to-end agent latency.
In particular, they should not be compared directly with the uncached latency
numbers below, which additionally account for query planning, NER, entity
linking, and answer generation.

\subsection{Notation}
\label{app:complexity_notation}

For the asymptotic analysis, suppose that the graph
\(\mathcal{G}=(\mathcal{V},\mathcal{E})\) contains
\(n=|\mathcal{V}|\) entity nodes and \(m=|\mathcal{E}|\) entity-relation
edges. If self-loops are added during GCN propagation, we denote
\(\tilde m=m+n\). Let the hidden dimension be \(d\), the number of propagation
layers be \(L\), the batch size be \(B\), and the number of pseudo-queries be
\(M\). One original query together with \(M\) pseudo-queries therefore requires
at most \(M+1\) graph reads.

Let
\(\mathbf{M}\in\{0,1\}^{n\times N_D}\) denote the sparse entity--document
association matrix, where \(N_D\) is the number of documents and
\(\operatorname{nnz}(\mathbf{M})\) is the number of nonzero
entity--document links. Let \(K_e\) be the number of top entities used for
document projection, \(K\) the number of final retrieved documents, and
\(\bar f\) the average number of documents linked to the selected entities.

\subsection{Offline Structural Features and Indexing}
\label{app:complexity_offline}

Let the dimensions of node-, edge-, and graph-level structural features be
\(p_n\), \(p_e\), and \(p_g\), respectively. In our implementation these are
small constants.

Given an adjacency list, degrees and average neighbor degrees can be computed
in \(O(n+m)\) time. Clustering coefficients and common-neighbor features
require neighbor-set intersections, with cost
\begin{equation}
    O\left(
    n+m+
    \sum_{(u,v)\in\mathcal{E}}
    \min\{\deg(u),\deg(v)\}
    \right).
    \label{eq:structure_feature_cost}
\end{equation}
For sparse graphs or graphs with bounded average degree,
Eq.~\eqref{eq:structure_feature_cost} is close to \(O(n+m)\), while an
extremely dense worst case can reach \(O(n^3)\).
The structural features and the entity--document matrix can be precomputed and
cached, requiring
\begin{equation}
    O\left(
    np_n+mp_e+p_g+\operatorname{nnz}(\mathbf{M})
    \right)
\end{equation}
space. Because these quantities are query-independent, they can be reused
across queries and across repeated reader evaluations of the same candidate
graph in the self-evolution loop.

\subsection{Forward Complexity of the Structurally Gated GFM}
\label{app:complexity_gfm_forward}

We first consider one forward pass for one query on one graph. A standard GCN
layer consists of a node-wise linear transformation and sparse adjacency
aggregation. The corresponding cost is
\begin{equation}
    C_{\mathrm{plain}}
    =
    O\left(nd^2+\tilde m d\right).
    \label{eq:plain_gcn_complexity}
\end{equation}

A structurally gated layer additionally generates a vector gate for each edge:
\begin{equation}
    \mm_{u\rightarrow v}
    =
    \gg_{uv}\odot W\hh_u,
\end{equation}
where \(\hh_u\) is the source-node representation,
\(W\in\mathbb{R}^{d\times d}\) is the node transformation,
\(\gg_{uv}\in\mathbb{R}^{d}\) is the structural gate for edge \((u,v)\), and
\(\odot\) is element-wise multiplication.
Let \(d_g\) denote the structural-feature encoding dimension and \(h_g\) the
hidden width of the gating MLP. When the gate uses source-node structure,
target-node structure, edge-pair structure, and a graph-level summary, its
generation cost is
\begin{equation}
\begin{aligned}
    C_{\mathrm{gate}}
    =
    O\Big(
        &np_n d_g
        + mp_e d_g
        + p_g d_g\\
        &+m(4d_g h_g+h_g d)
    \Big).
\end{aligned}
\label{eq:gate_complexity}
\end{equation}
The first three terms encode node, edge, and graph structural features; the
last term is the per-edge gating MLP. If edge-pair features or graph summaries
are disabled, their terms can be removed. If \(d_g\) and \(h_g\) scale with
\(d\), gate generation is \(O(md^2)\) in the worst case. If the gate is a
small constant-width or low-rank module, its practical scaling approaches
\(O(md)\). A structurally gated layer therefore costs
\begin{equation}
    C_{\mathrm{gated}}
    =
    O\left(nd^2+\tilde m d+C_{\mathrm{gate}}\right).
    \label{eq:gated_layer_complexity}
\end{equation}

The implementation can combine a holistic structurally gated branch with a
specific prompt branch. Let
\(\rho_{\mathrm{plain}},\rho_{\mathrm{gated}}\in\{0,1\}\) indicate whether
the corresponding branches are enabled. For a batch of \(B\) queries,
\begin{equation}
    C_{\mathrm{enc}}(B)
    =
    O\left(
    BL
    \left[
    \rho_{\mathrm{plain}} C_{\mathrm{plain}}
    +
    \rho_{\mathrm{gated}} C_{\mathrm{gated}}
    \right]
    \right).
    \label{eq:encoder_batch_complexity}
\end{equation}
The factor \(B\) appears because the current implementation constructs
query-conditioned node inputs and performs graph encoding for each query.
Caching query-independent structural features reduces constants but does not
change the conservative upper bound in
Eq.~\eqref{eq:encoder_batch_complexity}.

For the simplified setting used in the main-text complexity discussion,
\(B=1\), \(M=0\), only one branch is active, and the gating network is treated
as lightweight. Equation~\eqref{eq:encoder_batch_complexity} then reduces to
\begin{equation}
    O\left(L(md+d^2n)\right),
\end{equation}
where \(md\) represents sparse edge-level message propagation/gating and
\(d^2n\) the node-wise linear projections.

\subsection{Entity Scoring, Selector Regularization, and Document Projection}
\label{app:complexity_scoring_projection}

After graph encoding, the entity score is
\begin{equation}
    a_e(q)=\hh_e^\top\zz_q.
\end{equation}
For a batch of \(B\) queries, scoring all \(n\) entities costs
\begin{equation}
    C_{\mathrm{score}}(B)=O(Bnd).
\end{equation}

If the query-conditioned subgraph selector in
Appendix~\ref{app:selector_regularizer} is enabled, its inference-stage
overhead is
\begin{equation}
    C_{\mathrm{sel,infer}}(B)
    =
    O(Bnd^2+Bnd),
\end{equation}
and can approach \(O(Bnd)\) when the node-side projection is cached or a
lightweight projection is used. During training, the NCE, size, and
connectivity regularizers add
\begin{equation}
    C_{\mathrm{sel,train}}(B)
    =
    O(Bnd^2+Bnd+B^2d+Bm),
\end{equation}
where \(B^2d\) is the in-batch query--subgraph contrastive term and \(Bm\)
corresponds to edge-level connectivity smoothing.

Entity scores are projected to documents through the sparse association
matrix \(\mathbf{M}\). Full sparse matrix multiplication costs
\begin{equation}
    C_{\mathrm{doc,full}}(B)
    =
    O\left(B\operatorname{nnz}(\mathbf{M})\right).
\end{equation}
The current \texttt{IDFWeightedRanker} follows this sparse-projection form.
If only the top \(K_e\) entities are projected through inverted lists, the
conceptual cost is
\begin{equation}
    C_{\mathrm{doc,top}K}(B)
    =
    O\left(Bn\log K_e + BK_e\bar f\right),
    \label{eq:topk_doc_projection}
\end{equation}
where the first term selects top entities and the second accesses their linked
documents. Final top-\(K\) ranking over all \(N_D\) documents costs
\(O(BN_D\log K)\); ranking only a candidate pool costs
\(O(BN_{\mathrm{cand}}\log K)\), where
\(N_{\mathrm{cand}}\ll N_D\).

\subsection{Training-Time Complexity}
\label{app:complexity_training}

The reader-training cost is dominated by GFM forward propagation,
entity-level retrieval loss, the optional selector regularizer, and
backpropagation. Let \(\kappa_{\mathrm{bw}}\) denote the constant-factor cost
of backpropagation relative to the forward computation. For BCE, ranking loss,
or ListCE over all entity scores,
\begin{equation}
    C_{\mathrm{loss}}(B)=O(Bn).
\end{equation}
Without the optional selector,
\begin{equation}
    C_{\mathrm{train}}
    =
    O\left(
    \kappa_{\mathrm{bw}}
    \left[
        C_{\mathrm{enc}}(B)
        +
        C_{\mathrm{score}}(B)
        +
        C_{\mathrm{loss}}(B)
    \right]
    \right).
    \label{eq:train_complexity_no_selector}
\end{equation}
With the selector,
\begin{equation}
    C_{\mathrm{train}}^{\mathrm{sel}}
    =
    O\left(
    \kappa_{\mathrm{bw}}
    \left[
        C_{\mathrm{enc}}(B)
        +
        C_{\mathrm{score}}(B)
        +
        C_{\mathrm{loss}}(B)
        +
        C_{\mathrm{sel,train}}(B)
    \right]
    \right).
    \label{eq:train_complexity_selector}
\end{equation}
If document-level supervision is added explicitly, document projection adds
either \(C_{\mathrm{doc,full}}(B)\) or
\(C_{\mathrm{doc,top}K}(B)\).

The GFM pretraining/fine-tuning computation itself uses cached graph,
entity--document, and embedding artifacts and therefore does not require
remote LM calls during each optimization step. LM use in raw graph
construction is a separate one-time preprocessing cost quantified below.

\subsection{Inference-Time Complexity}
\label{app:complexity_inference}

Let \(C_{\mathrm{prep}}(q)\) denote the query-preparation cost before graph
reading, including query planning, NER, entity linking, and query encoding.
The graph-reader portion for one original query and \(M\) pseudo-queries is
\begin{equation}
\begin{aligned}
    C_{\mathrm{reader}}(q)
    =
    &(M+1)
    \Big[
        C_{\mathrm{enc}}(1)
        +
        C_{\mathrm{score}}(1)
        +
        C_{\mathrm{sel,infer}}(1)
        +
        C_{\mathrm{doc}}(1)
    \Big]\\
    &+
    C_{\mathrm{fuse}}(M,K),
\end{aligned}
\label{eq:inference_complexity}
\end{equation}
where \(C_{\mathrm{doc}}(1)\) is either full sparse projection or top-\(K_e\)
inverted projection; \(C_{\mathrm{sel,infer}}(1)=0\) when the optional
selector is disabled. If every query retains \(K\) candidate documents, simple
weighted fusion costs
\begin{equation}
    C_{\mathrm{fuse}}(M,K)
    =
    O\left((M+1)K\log((M+1)K)\right).
\end{equation}

The complete uncached online cost additionally includes query preparation and
answer generation:
\begin{equation}
    C_{\mathrm{online}}(q)
    =
    C_{\mathrm{prep}}(q)
    +
    C_{\mathrm{reader}}(q)
    +
    C_{\mathrm{answer}}(q).
    \label{eq:complete_online_complexity}
\end{equation}
This decomposition is the one used in the latency audit below. In particular,
the small graph-reader latency reported in Table~\ref{tab:efficiency} should
not be substituted for \(C_{\mathrm{online}}(q)\).

\subsection{Space Complexity}
\label{app:complexity_memory}

Offline graph storage requires
\begin{equation}
    O\left(
    n+m+\operatorname{nnz}(\mathbf{M})+np_n+mp_e+p_g
    \right).
\end{equation}
During training, storing layer-wise node activations requires
\begin{equation}
    O(BLnd)
\end{equation}
space. Fully materializing all edge gates/messages would require \(O(md)\)
GPU memory. The implementation instead streams edges in chunks. For chunk
size \(c\), the peak edge-message memory becomes
\begin{equation}
    O(cd), \qquad c\ll m.
\end{equation}
If \texttt{return\_gate} is enabled and all edge gates are retained for
visualization, memory returns to \(O(md)\). Materializing all document scores
requires \(O(BN_D)\) space, while maintaining only a candidate heap can reduce
this to \(O(BK)\) or \(O(BK_e\bar f)\).

\subsection{Measured Training Compute}
\label{app:complexity_measured_training}

\paragraph{Audited pretraining checkpoint.}
The training server contains eight A100-SXM4-80GB GPUs, but the auditable
pretraining run used three A100-80GB GPUs. The corresponding configuration
uses a 6-layer GCN backbone with hidden dimension 256, GraphCL pretraining,
128 negative samples, Adam with learning rate \(5\times10^{-4}\), activation
checkpointing, and 273 training shards. The epoch-97 checkpoint contains
22,233,918 parameters.

The run is recoverable from two logged segments. The first segment ran for
18 h 26 min 43 s on three GPUs, corresponding to approximately 55.34
GPU-hours. The second resumed from epoch 94 and recomputed through epoch 97,
running for 46 min 17 s, or approximately 2.31 GPU-hours. Counting the
recomputed work, the total auditable cost to the epoch-97 checkpoint is
therefore approximately
\begin{equation}
    57.65\ \text{GPU-hours}
\end{equation}
and 19 h 13 min of aggregate wall-clock training time on three GPUs.

\begin{table}[t]
\centering
\caption{Audited compute for the recoverable epoch-97 GFM pretraining
checkpoint. GPU-hours include the resumed segment and its recomputation.}
\label{tab:complexity_training_compute}
\small
\setlength{\tabcolsep}{6pt}
\begin{tabular}{lr}
\toprule
Item & Value\\
\midrule
Hardware used by run & \(3\times\) A100-80GB\\
Backbone & 6-layer GCN\\
Hidden dimension & 256\\
Training shards & 273\\
Checkpoint parameters & 22,233,918\\
Logged wall-clock compute & \(\approx 19\) h 13 min\\
Logged GPU-hours & \(\approx 57.65\)\\
\bottomrule
\end{tabular}
\end{table}

\paragraph{Limitation of the historical fine-tuning record.}
The historical directory corresponding to the fine-tuning run that produced
the best reported HotpotQA/2WikiMultiHopQA/MuSiQue checkpoint is no longer
available. Its exact wall-clock time and peak-memory trace therefore cannot be
reconstructed. We intentionally do not substitute the duration of a different
fine-tuning attempt and present it as the cost of the best model. The
57.65 GPU-hour number above should consequently be interpreted as the
auditable pretraining cost, not as a claim about the complete historical
end-to-end training budget.

\paragraph{Configuration-matched fine-tuning reference.}
For scale, a later fine-tuning attempt with the closest recoverable
configuration used five A100-80GB GPUs, batch size 1, gradient accumulation 4,
100 batches per epoch, and the same 273 training shards. Its first epoch,
including validation and checkpointing, took approximately 3 h 20 min, or
16.68 GPU-hours. That run was terminated during epoch 2 and did not produce
the best reported checkpoint. We therefore report this number only as a
configuration-level reference and do not extrapolate it into an ``actual''
fine-tuning cost for the final model.

\subsection{Offline LM Cost for Raw Graph Construction}
\label{app:complexity_offline_lm}

Raw Stage-1 graph construction performs one LLM-NER call and one OpenIE call
per document, and one LLM-NER call per QA question. If no processed artifact is
cached, the one-time LM-call count is therefore
\begin{equation}
    N_{\mathrm{LM,Stage1}}
    =
    2N_{\mathrm{doc}}+N_{\mathrm{question}}.
    \label{eq:stage1_lm_calls}
\end{equation}
For the 273 multi-hop training shards used in the audit, the raw instance
counts imply the call budget in Table~\ref{tab:complexity_stage1_calls}.
These are \emph{rebuild estimates}: the audited GFM training run reused
processed artifacts and did not regenerate all of these calls.

\begin{table*}[t]
\centering
\caption{One-time Stage-1 LM-call estimate if the 273 training shards are
rebuilt from raw data without cached preprocessing. Document counts are raw
document instances before cross-shard deduplication.}
\label{tab:complexity_stage1_calls}
\small
\setlength{\tabcolsep}{7pt}
\begin{tabular}{lrrrr}
\toprule
Dataset & Shards & Document instances & QA questions & Stage-1 LM calls\\
\midrule
2WikiMultiHopQA & 168 & 966,907 & 167,454 & 2,101,268\\
HotpotQA & 89 & 864,849 & 88,447 & 1,818,145\\
MuSiQue & 16 & 318,724 & 15,938 & 653,386\\
\midrule
\textbf{Total} & \textbf{273} & \textbf{2,150,480}
& \textbf{271,839} & \textbf{4,572,799}\\
\bottomrule
\end{tabular}
\end{table*}

The 2,150,480 document instances contain substantial duplication across
shards. Normalizing and deduplicating by title--content pairs leaves
1,101,879 unique pairs. If such global deduplication is performed before both
document-level LM calls, the estimated Stage-1 budget decreases by
\begin{equation}
    2(2{,}150{,}480-1{,}101{,}879)
    =
    2{,}097{,}202
\end{equation}
calls, from 4,572,799 to 2,475,597 calls, a 45.9\% reduction. This identifies
cross-shard deduplication as a direct optimization for future raw rebuilds.

For reference, rebuilding the three 1,000-question evaluation corpora from raw
data would require 19,442 Stage-1 calls for HotpotQA
(\(9{,}221\) documents), 13,238 for 2WikiMultiHopQA
(\(6{,}119\) documents), and 24,312 for MuSiQue
(\(11{,}656\) documents). Again, these are one-time raw-rebuild costs rather
than costs repeated for each evaluation.

\subsection{Online LM Calls and Token Footprint}
\label{app:complexity_online_lm}

Under the successful uncached path in
Table~\ref{tab:complexity_lm_calls}, 1,000 queries require 4,000 successful
LM calls: 1,000 planner-extractor calls, 1,000 planner-inferer calls,
1,000 NER calls, and 1,000 answer-generation calls.

We additionally audited the prompt/response content that remained available
from the multi-hop evaluation pipeline. The answer prompts and responses can
be reconstructed directly from saved predictions. Planner outputs are only
partially preserved, and raw NER outputs are not fully preserved; therefore,
planner/NER token counts are lower-bound/proxy estimates rather than provider
billing records. Message-framing tokens are not included.

\begin{table*}[t]
\centering
\caption{Audited per-question token footprint. Planner-inferer input is a lower
bound because not all intermediate planner fields are retained. The answer
input contains the retrieved top-10 document context.}
\label{tab:complexity_token_per_query}
\small
\setlength{\tabcolsep}{5pt}
\begin{tabular}{lrrrrr}
\toprule
Dataset
& Planner extractor input
& Planner inferer input
& NER input
& Answer input
& Answer output\\
\midrule
HotpotQA & 106 & 124 & 82 & 2,234 & 105\\
2WikiMultiHopQA & 102 & 120 & 78 & 1,949 & 87\\
MuSiQue & 107 & 124 & 83 & 2,013 & 100\\
\bottomrule
\end{tabular}
\end{table*}

The answer context dominates the recoverable input-token footprint. Across the
three datasets, the top-10 answer prompt accounts for approximately
86--88\% of the audited input tokens. For roughly 1,000 successful queries,
the lower-bound four-call input totals are 2.545M tokens for HotpotQA,
2.248M for 2WikiMultiHopQA, and 2.327M for MuSiQue; the corresponding
recoverable output-token proxies are 0.296M, 0.267M, and 0.288M,
respectively. We report token volumes instead of a fixed dollar figure because
provider pricing can change over time.

\paragraph{Retries in the historical benchmark runs.}
The Stage-3 answer wrapper allowed up to five attempts and slept for 30 s
after an error. The saved logs contain 260 API-error events for HotpotQA,
223 for 2WikiMultiHopQA, and 156 for MuSiQue. With 999, 999, and 1,000 saved
answers, respectively, the logs imply at least 1,259, 1,222, and 1,156 answer
request attempts. The dominant recorded errors were HTTP 429 rate-limit
events while the three evaluations were running concurrently under a shared
API project. These retries affect wall-clock time, but rejected requests
should not automatically be interpreted as successful billed generations.

\subsection{Measured Stage-Wise Online Latency}
\label{app:complexity_measured_latency}

We performed a controlled latency audit using the first 20 saved HotpotQA
queries on the same shared server. Query planning was measured on 20 examples
(two serial LM calls per example), LLM NER on 20 examples, and local entity
linking on 20 examples. Answer generation was measured on the first 10 examples
using the original top-10 retrieved-document context. The entity-linking test
used an otherwise available A100 GPU and searched over 87,768 HotpotQA graph
entities. Because the server was shared and under nontrivial load, these values
describe the measured environment rather than an ideal dedicated deployment.

\begin{table}[t]
\centering
\caption{Measured latency of audited online stages. The graph
query-encoding/GFM/document-ranking stage was not separately timed in this
controlled audit and is therefore not included in the 13.22 s sum below.}
\label{tab:complexity_stage_latency}
\small
\setlength{\tabcolsep}{4pt}
\begin{tabular}{lrrrr}
\toprule
Stage & \(N\) & Mean (s) & Median (s) & P95 (s)\\
\midrule
Planner (2 serial LM calls) & 20 & 4.72 & 4.40 & 12.93\\
LLM NER & 20 & 1.52 & 1.30 & 2.63\\
Entity linking & 20 & 3.08 & 2.25 & 14.30\\
Answer generation & 10 & 5.55 & 5.27 & 10.22\\
\bottomrule
\end{tabular}
\end{table}

The two planner calls are necessarily serial in the current implementation:
the inferer consumes the extractor JSON, so they cannot be issued in parallel
without changing the planner. The measured median of the directly audited
stages is
\begin{equation}
    4.40 + 1.30 + 2.25 + 5.27
    =
    13.22\ \text{s}.
\end{equation}
This is \emph{not} the full end-to-end median. The controlled audit did not
separately measure Qwen3-Embedding-8B query encoding, GFM forward passes for
the primary/pseudo queries, document ranking, or retrieval-result fusion.
The only supported statement from these measurements is therefore
\begin{equation}
    T_{\mathrm{uncached,online}}
    >
    13.22\ \text{s}
    \quad\text{(median in the audited environment).}
\end{equation}

\paragraph{Entity-linking initialization.}
Each query supplied a mean of 7.85 mentions to ColBERT
(median \(7\), P95 \(14\)), with top-5 linked entities per mention.
The first construction of the HotpotQA entity index required 10.42 s to load
the ColBERT model and 441.65 s to build the index. When an existing index was
reused, model loading took 3.87 s and index reuse 2.82 s. The first searcher
query also loads search structures and CUDA components, explaining the
14.30 s P95/cold-query value in
Table~\ref{tab:complexity_stage_latency}. Index construction is an offline
startup cost, and a deployed service should build/reuse and warm the index
before serving online queries rather than charge this cost to every request.

\subsection{Observed Stage-3 Evaluation Wall-Clock}
\label{app:complexity_stage3_wallclock}

For completeness, Table~\ref{tab:complexity_stage3_runtime} reports the
wall-clock of the historical best Stage-3 multi-hop evaluations. Each run
contains 1,000 questions, uses \texttt{n\_threads=5} for concurrent answer
calls, supplies top-10 documents to the answer model, and reuses precomputed
planner/NER/entity-linking results.

\begin{table*}[t]
\centering
\caption{Observed Stage-3 evaluation wall-clock. These times are not
single-query uncached latencies: planner/NER/entity linking are precomputed,
five answer calls can run concurrently, and the runs include rate-limit
retries and evaluation overhead.}
\label{tab:complexity_stage3_runtime}
\small
\setlength{\tabcolsep}{7pt}
\begin{tabular}{lrrr}
\toprule
Dataset & Wall-clock & Wall-clock / question & Throughput (q/s)\\
\midrule
HotpotQA & 2 h 30 min 45 s & 9.05 s & 0.110\\
2WikiMultiHopQA & 2 h 00 min 24 s & 7.23 s & 0.138\\
MuSiQue & 2 h 30 min 40 s & 9.04 s & 0.111\\
\bottomrule
\end{tabular}
\end{table*}

These Stage-3 runs include model/data initialization, local query embedding,
primary and pseudo-query graph retrieval, document ranking, five-thread answer
generation, retry sleeps, and the final evaluator. They exclude the
precomputed two-call planner, LLM NER, and online entity-linking stages.
Moreover, the three datasets were evaluated concurrently and shared the same
API project's rate limit, so these wall-clock values should not be interpreted
as isolated single-user latency.

\subsection{Complexity Interpretation and Comparison}
\label{app:complexity_related_work}

Different retrieval architectures place their cost in different parts of the
pipeline. Dense single-step RAG primarily pays for query encoding and nearest-
neighbor search over a prebuilt index. Iterative RAG adds repeated retrieval
and, when generation/reasoning is interleaved, repeated LM calls. GraphRAG
systems can shift substantial work to offline graph construction, indexing, or
summary generation. SAGE likewise has non-negligible offline graph-construction
cost, but once graph artifacts are cached, its core reader uses one primary
graph read plus at most a small number of pseudo-query reads, followed by sparse
entity-to-document projection.

The practical implication for self-evolving memory is not that the entire SAGE
agent has millisecond end-to-end latency. Rather, the \emph{graph-reader
feedback path} is a reusable local computation once query preparation and
indices are available. This makes it possible to score candidate memory graphs
repeatedly without launching a new multi-round agentic search for every graph
candidate. At the same time, the measured audit shows that a fully uncached
online system is dominated by LM query planning, NER/entity preparation, and
answer generation, so the retriever-only times in Table~\ref{tab:efficiency}
must be interpreted within that narrower scope.

Finally, the paired long-term-memory host experiments in
Appendix~\ref{app:host_integration_details} are accuracy-oriented interventions.
The available logs do not contain a matched host-level end-to-end latency audit
for every Native/\(+\)SAGE configuration. We therefore do not infer host-level
latency improvements from the multi-hop reader timings reported here.
\section{Implementation Details of Structured Query Planning}
\label{app:rewriter_details}

\subsection{Detailed definitions of the notation and additional information}

 In  \(\mathcal{P}_\omega(q)
    =
    \Big(
        \mathcal{E}_{\mathrm{exp}},
        \mathcal{A},
        \mathcal{C}_{\mathrm{rel}},
        \mathcal{C}_{\mathrm{hard}},
        \tau,
        \{(\tilde q_m,\alpha_m,t_m)\}_{m=1}^{M}
    \Big)\), $\mathcal{E}_{\mathrm{exp}}$ acts as a direct anchor for memory (explicit entities); $\mathcal{A}$ maps the brain's multiple representational habits for the same concept (aliases); $\mathcal{C}_{\mathrm{rel}}$ simulates the relational network in semantic memory; $\mathcal{C}_{\mathrm{hard}}$ serves as the spatiotemporal and logical boundaries of episodic memory (such as hard constraints like time and location); $\tau$ presets the cognitive template of the target memory (answer type); and the pseudo-queries $\tilde q_m$ with confidence $\alpha_m$ and intent $t_m$ are analogous to the multiple exploratory recalls conducted in the human mind (Simulated Recall).

\subsection{Two-stage Planning: Extraction and Inference}

Natural-language questions often compress key retrieval cues into implicit relations, such as ``the birthplace of the author'', ``the publication year of the only mystery novel of a certain work'', or ``the death date of the father''. If the original question is sent as a whole to the entity linker, the system can easily hit only surface entities while missing bridge entities or answer-type constraints. Therefore, the query planner is defined as a structured function
\begin{equation}
    \mathcal{P}(q)=
    \left(
        \mathcal{E}_{\mathrm{exp}},
        \mathcal{A},
        \mathcal{C}_{\mathrm{rel}},
        \mathcal{C}_{\mathrm{hard}},
        \tau,
        \{(\tilde q_m,\alpha_m)\}_{m=1}^{M}
    \right).
\end{equation}
It consists of two stages: \emph{Extractor}~\ref{fig:case_1} extracts explicit entities, aliases, relation clues, hard constraints, and the answer type; \emph{Inferer}~\ref{fig:case_2} generates at most $M$ retrieval intents based on the extraction results.

\begin{center}
\begin{tcolorbox}[
    float*=t,
    width=1\textwidth,
    enhanced,
    colback=white,
    colframe=black,
    boxrule=0.7pt,
    rounded corners,
    fonttitle=\bfseries,
    title=Extractor prompt template.,
    boxsep=2pt,
    left=2pt,
    right=2pt,
    top=8pt,
    bottom=5pt,
    breakable
]
\begin{spacing}{0.88}
{
\small
\begin{Verbatim}
You are a retrieval planner for graph-based multi-hop QA.
Question:
{QUESTION}

Extract structured retrieval signals.
Return JSON only with keys:
{
  "explicit_entities": [string],
  "candidate_aliases": {"entity": [alias]},
  "relation_clues": [string],
  "constraints": {},
  "answer_type": "string"
}
Rules: keep entries short, avoid explanations, keep empty fields as [] or {}.
\end{Verbatim}
}
\end{spacing}
\captionof{figure}{Metadata of the case study from \texttt{HotpotQA}.}
\label{fig:case_1}
\vspace{-0.75cm}
\end{tcolorbox}
\end{center}

\begin{center}
\begin{tcolorbox}[
    float*=t,
    width=1\textwidth,
    enhanced,
    colback=white,
    colframe=black,
    boxrule=0.7pt,
    rounded corners,
    fonttitle=\bfseries,
    title=Inferer prompt template.,
    boxsep=2pt,
    left=2pt,
    right=2pt,
    top=8pt,
    bottom=5pt,
    breakable
]
\begin{spacing}{0.88}
{
\small
\begin{Verbatim}
You are a retrieval planner for graph-based multi-hop QA.
Question:
{QUESTION}

Structured extraction:
{EXTRACTOR_JSON}

Generate at most M retrieval intents that help locate:
- evidence directly supporting the target relation;
- bridge entities required for multi-hop reasoning;
- documents likely to contain the target attribute;
- evidence satisfying temporal, spatial, type, comparison or negation constraints;
- evidence using aliases or alternative mentions.
Return JSON only with keys:
{
  "pseudo_queries": [string],
  "rewriter_confidence": [number]
}
\end{Verbatim}
}
\end{spacing}
\captionof{figure}{Metadata of the case study from \texttt{HotpotQA}.}
\label{fig:case_2}
\vspace{-0.75cm}
\end{tcolorbox}
\end{center}

\section{Computation Details of Topological Structural Features}
\label{app:structure_features_detailed}

\subsection{Normalized Structural Graph}

Structural features are computed on an undirected, self-loop-free, binarized adjacency matrix $\A_s$:
\begin{equation}
    \A_s=\I[(\A+\A^\top)>0],
    \qquad
    \diag(\A_s)=0.
\end{equation}
This avoids drastic fluctuations in topological statistics caused by unstable relation-extraction directions. Message propagation can still use the original bidirectional edges or relation-aware graph; structural statistics are only used as gating conditions.

\subsection{Node-level Structural Features}

For node $v$, let $\mathcal{N}(v)=\{u:\A_{s,uv}=1\}$ and $d_v=|\mathcal{N}(v)|$. The node-level features are
\begin{equation}
    \phi(v)=
    \big[
    \log(1+d_v),
    c_v,
    \kappa_v,
    \bar d_{\mathcal{N}(v)}
    \big].
\end{equation}
The local clustering coefficient is
\begin{equation}
    c_v=
    \begin{cases}
    \dfrac{2T_v}{d_v(d_v-1)},&d_v\ge 2,\\[0.4em]
    0,&d_v<2,
    \end{cases}
\end{equation}
where $T_v$ is the number of undirected edges inside the neighborhood of $v$; $\kappa_v$ is the core number; and the average neighbor degree is
\begin{equation}
    \bar d_{\mathcal{N}(v)}=
    \begin{cases}
    \dfrac{1}{d_v}\sum_{u\in\mathcal{N}(v)}d_u,&d_v>0,\\[0.4em]
    0,&d_v=0.
    \end{cases}
\end{equation}
These quantities respectively characterize node frequency, local clustering, core/peripheral position, and neighborhood density. For RAG memory, they correspond to four common structural risks: over-propagation by high-frequency hubs, redundant diffusion inside clustered regions, ignored peripheral bridge entities, and scale mismatch between sparse and dense regions.

\subsection{Edge-pair Structural Features}

For an undirected structural edge $(u,v)$, the pairwise features are
\begin{equation}
    \psi(u,v)=
    \big[
    |d_u-d_v|,
    \operatorname{CN}(u,v),
    \operatorname{Jac}(u,v)
    \big],
\end{equation}
where
\begin{equation}
    \operatorname{CN}(u,v)=|\mathcal{N}(u)\cap\mathcal{N}(v)|,
    \qquad
    \operatorname{Jac}(u,v)=
    \frac{|\mathcal{N}(u)\cap\mathcal{N}(v)|}
    {|\mathcal{N}(u)\cup\mathcal{N}(v)|+\eps}.
\end{equation}
Degree difference reflects cross-level connections, while common neighbors and Jaccard reflect local community overlap. Based on these features, the gate can distinguish intra-community evidence aggregation edges from cross-community bridge edges.

\subsection{Graph-level Summary and Normalization}

The graph-level summary concatenates the mean, standard deviation, and density of node features:
\begin{equation}
    \rr_\kg=
    \big[
    \mean_{v\in\V}\phi(v);
    \std_{v\in\V}\phi(v);
    \dens(\kg)
    \big],
\end{equation}
where
\begin{equation}
    \dens(\kg)=
    \begin{cases}
    \dfrac{2m_s}{n(n-1)},&n\ge2,\\[0.4em]
    0,&n<2,
    \end{cases}
\end{equation}
$n=|\V|$, and $m_s$ is the number of undirected structural edges. To remove graph-size differences, node and edge features are z-scored within each graph, and the graph-level summary computes global mean and standard deviation over the set of training graphs:
\begin{equation}
    \bar\rr_\kg=\frac{\rr_\kg-\mu_r}{\sigma_r+\eps}.
\end{equation}
If the standard deviation of a certain dimension is close to zero, we only perform centering to avoid division by an unstable small value.

\subsection{Gating Input Encoding}

For each message edge $u\to v$, the structural gate reads the source node, target node, pairwise features, and graph-level summary:
\begin{align}
    \bar\phi(u)&=\operatorname{NormNode}(\phi(u)),
    &\bar\psi(u,v)&=\operatorname{NormPair}(\psi(u,v)),\\
    u_u^{(l)}&=E_n^{(l)}(\bar\phi(u)),
    &u_v^{(l)}&=E_n^{(l)}(\bar\phi(v)),\\
    v_{uv}^{(l)}&=E_p^{(l)}(\bar\psi(u,v)),
    &r_{\kg}^{(l)}&=E_g^{(l)}(\bar\rr_\kg).
\end{align}
The encoders $E_n,E_p,E_g$ are all two-layer MLPs. The concatenated gating input is
\begin{equation}
    \zz_{uv}^{(l)}=
    [\uu_u^{(l)};\uu_v^{(l)};\vv_{uv}^{(l)};\rr_\kg^{(l)}].
\end{equation}
The gate itself is a vector rather than a scalar:
\begin{equation}
    \gv_{uv}^{(l)}=
    \one+\delta\tanh\big(\MLP_g^{(l)}(\zz_{uv}^{(l)})\big),
    \qquad \delta=0.1.
    \label{eq:app_residual_gate}
\end{equation}
The last layer of the gating MLP is initialized to zero, so initially $\gv_{uv}^{(l)}=\one$. At the beginning of training, the model does not destroy the original propagation scale; the learned structural bias gradually emerges in a residual manner.

\subsection{Message Propagation with Normalized Weights}

Let $\tilde\E$ be the edge set after adding self-loops. Structural gates are used for non-self-loop edges, and unit gates are used for self-loops. The GCN normalization coefficient is
\begin{equation}
    \eta_{uv}=\frac{w_{uv}}{\sqrt{\tilde d_u\tilde d_v}},
    \qquad
    \tilde d_v=\sum_{u:(u,v)\in\tilde\E}w_{uv},
\end{equation}
where $w_{uv}$ defaults to $1$, but can also come from edge weights. The propagation at layer $l$ is
\begin{align}
    \mv_{u\to v}^{(l)}&=
    \eta_{uv}\,\gv_{uv}^{(l)}\odot W^{(l)}\hh_u^{(l-1)},\\
    \hh_v^{(l)}&=
    \sigma\left(\bv^{(l)}+
    \sum_{u:(u,v)\in\tilde\E}\mv_{u\to v}^{(l)}
    \right).
\end{align}
The multi-layer wrapper also contains inter-layer residuals: when $l>1$ and the dimensions are consistent,
\begin{equation}
    \Hh^{(l)}\leftarrow \Hh^{(l)}+\Hh^{(l-1)}.
\end{equation}
This residual and Eq.~\eqref{eq:app_residual_gate} form a dual stability mechanism: the former stabilizes deep propagation, while the latter stabilizes structural modulation.

\subsection{Chunked Gating and GPU Memory Complexity}

Explicitly storing all gates requires $O(|\E|d)$ GPU memory. For large graphs, gates are computed by edge chunks:
\begin{equation}
    \E=\bigcup_{b=1}^{B_e}\E_b,
    \qquad
    |\E_b|\le C_e.
\end{equation}
Each edge chunk sequentially executes
\begin{equation}
    \gv_b\rightarrow \mv_b\rightarrow \operatorname{scatter\_add}(\mv_b),
\end{equation}
and immediately releases the intermediate gate tensor. Online GPU memory is reduced from $O(|\E|d)$ to $O(C_e d)$, while the time complexity remains linear, $O(|\E|d)$. This is especially important for self-evolving memory, because the same reader needs to repeatedly evaluate candidate graphs produced by different writers.

\section{Pretraining Objective and Augmented Views}
\label{app:pretraining_details}

\subsection{GraphCL View Construction}

The goal of the pretraining stage is to learn cross-graph transferable structural--semantic propagation, rather than fitting specific question-answering labels. Given the original graph view $(\kg_0,X_0)$, we construct two augmented views $(\kg_1,X_1)$, $(\kg_2,X_2)$ and one negative feature view $(\kg_0,X^-)$. The augmentation types include edge perturbation, feature masking, node perturbation, and subgraph sampling; let the augmentation operators be $\mathcal{A}_1,\mathcal{A}_2$, then
\begin{equation}
    (\kg_j,X_j)=\mathcal{A}_j(\kg_0,X_0),\qquad j\in\{1,2\}.
\end{equation}
If structural gating is enabled, each view precomputes its own node structural features, edge-pair features, and graph-level summary; the negative feature view shares the base graph structure, but its node features are shuffled or replaced.

\subsection{Graph-level Contrastive Objective}

The encoder outputs four sets of node representations:
\begin{equation}
    H_0=f_\theta(X_0,\kg_0),\quad
    H_1=f_\theta(X_1,\kg_1),\quad
    H_2=f_\theta(X_2,\kg_2),\quad
    H^-=f_\theta(X^-,\kg_0).
\end{equation}
The graph readout of each augmented view is
\begin{equation}
    c_j=\sigmoid\left(\frac{1}{|\V_j|}\sum_{v\in\V_j}H_{j,v}\right),
    \qquad j\in\{1,2\}.
\end{equation}
The bilinear discriminator
\begin{equation}
    D(c,h)=h^\top W_Dc
\end{equation}
determines whether the node representation comes from the same graph semantics. The pretraining loss is
\begin{equation}
    \mathcal{L}_{\mathrm{GCL}}=
    \frac12\sum_{j=1}^{2}
    \left[
    \BCE\big(D(c_j,H_0),\one\big)
    +
    \BCE\big(D(c_j,H^-),\mathbf{0}\big)
    \right].
\end{equation}
When edge-level gating is enabled, traditional static structural prompts are neutralized into identity mappings to avoid scale confusion caused by two sets of structural modulations acting simultaneously; the structural bias is mainly carried by the target edge's $\gv_{uv}^{(l)}$.

\subsection{Feature Alignment Layer}

When the input dimensions produced by different graphs or different text encoders are consistent but their distributions have large shifts, the feature alignment layer can be enabled:
\begin{equation}
    \operatorname{Align}(x)=
    \Dropout\left(
    \LayerNorm\left(
    \PReLU(W_ax+b_a)
    \right)
    \right).
\end{equation}
$W_a$ is initialized as the identity matrix, and $b_a$ is initialized as zero. Therefore, this layer is initially an approximately identity transformation; after training, it absorbs inter-graph feature-scale differences without changing the core structure of the graph propagator.

\section{Supervised Fine-tuning Objective}
\label{app:finetune_objective_detailed}

\subsection{Entity-level Supervision}

For each question $q_b$, the data provide a supporting-entity mask $y_{b,e}\in\{0,1\}$. The model outputs entity logits $a_{b,e}$. The weighted BCE is defined as
\begin{equation}
    \mathcal{L}_{\mathrm{bce}}
    =\frac{1}{B}\sum_{b=1}^{B}
    \frac{\sum_{e}w_{b,e}\,\BCEWithLogits(a_{b,e},y_{b,e})}
    {\sum_e w_{b,e}+\eps}.
\end{equation}
Positive weights are uniformly normalized within the positive set; if the adversarial temperature $T_a$ is enabled for negative weights, they are computed by applying softmax to the current model scores:
\begin{equation}
    w_{b,e}^{-}=\frac{\exp(a_{b,e}/T_a)}{\sum_{v:y_{b,v}=0}\exp(a_{b,v}/T_a)},
    \qquad y_{b,e}=0.
\end{equation}
If $T_a=0$, the negative weights degenerate into a uniform distribution. This design makes training focus more on high-scoring hard negatives, rather than being dominated by a large number of obviously irrelevant entities.

\subsection{Multi-positive List Cross-Entropy}

Using only BCE treats each entity as an independent binary classification problem, lacking the constraint that ``supporting entities should collectively rank near the top of the same candidate list''. To this end, we introduce a multi-positive list loss. Let
\begin{equation}
    p_{b,e}=\frac{\sigmoid(a_{b,e})}{\sum_{v}\sigmoid(a_{b,v})+\eps}.
\end{equation}
If sample $b$ has at least one supporting entity, the list loss is
\begin{equation}
    \mathcal{L}_{\mathrm{list}}
    =-\frac{1}{|\mathcal{B}_+|}\sum_{b\in\mathcal{B}_+}
    \frac{1}{|Y_E(q_b)|}
    \sum_{e\in Y_E(q_b)}\log(p_{b,e}+\eps).
\end{equation}
Samples with empty supporting-entity sets are skipped. The final entity fine-tuning objective is
\begin{equation}
    \mathcal{L}_{\mathrm{ent}}
    =\lambda_{\mathrm{bce}}\mathcal{L}_{\mathrm{bce}}
    +\lambda_{\mathrm{list}}\mathcal{L}_{\mathrm{list}},
    \qquad
    (\lambda_{\mathrm{bce}},\lambda_{\mathrm{list}})=(0.3,0.7).
\end{equation}

\subsection{Optional Document-level Supervision}

If the training configuration provides a document-level loss, entity logits are first projected into document logits:
\begin{equation}
    \tilde S_b=a_b^\top\M,
\end{equation}
and then the same type of BCE or list loss is computed with the supporting-document mask $z_{b,i}$. This term is suitable for tasks where entity annotations are noisy but the document support set is reliable; if it is not enabled, training is entirely driven by the entity-level support set, and document ranking is obtained through projection only during inference or validation.
\section{Memory Writer Implementation Details}
\label{app:writer_details}

\subsection{The Markov Decision Process for Multi-turn Graph Construction}

Specifically, the training of our graph constructor is implemented through VeRL's multi-turn GRPO loop. The state machine of the interactor can be abstracted as a finite-horizon MDP:
\begin{equation}
    \mathcal{M}=(\Sset,\A,P,\reward,\rho_0,H).
\end{equation}
Given a sample $x$, at round $t$, the state can be written as
\begin{equation}
    \state_t=(q,\kg_t,\mathcal{D}^{\mathrm{proc}}_t,\mathcal{D}^{\mathrm{rem}}_t,\zeta_t),
\end{equation}
where $\kg_t$ is the current partially written graph, $\mathcal{D}^{\mathrm{proc}}_t$ and $\mathcal{D}^{\mathrm{rem}}_t$ denote the processed and remaining documents, respectively, and $\zeta_t$ is an interaction control flag, such as whether the process is still in the graph-construction stage or has already switched to the RAG stage. The action is generated by the language model in JSON format:
\begin{equation}
    \act_t \sim \constructor(\cdot\mid \state_t),
\end{equation}
and is restricted to two types of legal actions:
\begin{enumerate}[leftmargin=18pt]
    \item \textbf{Triple action}: output a JSON array, where each element is of the form $\{\texttt{subject},\texttt{relation},\texttt{object}\}$, representing the set of facts $\triples_t$ written in the current round;
    \item \textbf{Termination action}: after graph construction is completed, output a JSON object carrying the terminal fields required by the reader side, such as \texttt{answer}, \texttt{recall}, \texttt{precision}, \texttt{deducible}, and so on.
\end{enumerate}

In implementation, the environment first checks whether the action can be parsed by \texttt{json\_repair}, and strictly cleans the triples: items with missing keys, empty strings, or non-dictionary entries are all removed. If illegal JSON is output during the graph-construction stage, the interaction terminates immediately and returns zero reward; if legal triples are output, the environment proceeds to the next round and returns a round-level format reward. The corresponding environment transition can be written as
\begin{equation}
    \state_{t+1} = P(\state_t, \act_t) = 
    \begin{cases}
        (q,\kg_t\oplus\triples_t,\mathcal{D}^{\mathrm{proc}}_t\cup\{d_t\},\mathcal{D}^{\mathrm{rem}}_t\setminus\{d_t\},\zeta_{t+1}), & \act_t \text{ is legal },\\
        (q,\kg_t,\mathcal{D}^{\mathrm{proc}}_t,\mathcal{D}^{\mathrm{rem}}_t,\mathrm{STOP}), & \act_t \text{ is illegal},\\
        (q,\kg_t,\mathcal{D}^{\mathrm{proc}}_t,\mathcal{D}^{\mathrm{rem}}_t,\mathrm{RAG}), & \act_t \text{ triggers the reading stage}.
    \end{cases}
\end{equation}

\paragraph{Iterative and non-iterative writing.}
Two strategies are supported. In the non-iterative mode, the model reads the entire context $\contextset$ at once and outputs all triples. In the iterative mode, the environment reads the documents segment by segment in document order, and in each round the model is only allowed to write triples for the current document. After all documents have been processed, the environment then switches to the RAG stage. If $\triples_i$ denotes the set of triples output for document $d_i$, then the final graph constructed in the iterative mode is \(\kg = \bigoplus_{i=1}^{m} \triples_i\), where $\oplus$ denotes edge-set union and node deduplication. We adopt the iterative strategy by default, because it decomposes the long-context problem into a sequence of local writing decisions, significantly reducing the difficulty of performing global planning in advance. At the same time, it also allows the source document of each triple to be precisely recorded, providing explicit source edges for subsequent text-graph retrieval.

\paragraph{Constructing text-graph memory from output triples.}
To enable the frozen retriever to operate under the \textbf{graph-guided text retrieval} setting, the environment does not directly pass the raw triple strings to the retriever. Instead, it first constructs a text graph with document nodes:
\begin{equation}
    \kg=(\entityset \cup \docset, \edgeset_{ee} \cup \edgeset_{ed}),
\end{equation}
where the entity node set $\entityset$ comes from the subjects and objects in the triples, and the document node set $\docset=\{d_1,\dots,d_m\}$ corresponds to the original documents in the context. The entity-entity edges are defined as
\begin{equation}
    \edgeset_{ee}=\{(u,r,v)\mid (u,r,v)\in \triples\},
\end{equation}
and the entity-document source edges are defined as
\begin{equation}
    \edgeset_{ed}=\{(u,\texttt{source},d_i),(v,\texttt{source},d_i)\mid (u,r,v)\in \triples_i\}.
\end{equation}
In the iterative mode, the source edges are explicit, because the environment already knows that the triples in each round come from the current document. In the non-iterative mode, we use a heuristic alignment method based on tokenizer token overlap to map each triple to the most similar document. The significance of this design is that, after separating writing from reading, the graph constructor is only responsible for deciding ``what to write into memory''; as for how the reader aggregates entities on the graph and retrieves documents, this is entirely determined by the frozen $\retriever$.

\paragraph{Frozen GFM retrieval environment.}
When training the graph constructor, the reader $\retriever$ is fixed as the already trained GFM retriever. Let the entity set be $\entityset=\{e_1,\dots,e_n\}$ and the document set be $\docset=\{d_1,\dots,d_M\}$. We then construct:
\begin{enumerate}[leftmargin=18pt]
    \item the relation-edge index $\mathbf{E}$ with both forward and reverse directions, together with the relation types $\mathbf{r}$;
    \item the sparse entity-document matrix $\mathbf{M}\in\{0,1\}^{n\times M}$, where $M_{ij}=1$ if entity $e_i$ appears in document $d_j$;
    \item the question-related entity mask $\mathbf{m}_q\in\{0,1\}^n$, which is obtained preferentially through lexical matching with the question; if lexical matching fails, it falls back to a heuristic seed set ranked by entity degree.
\end{enumerate}

After encoding the question as a vector $\mathbf{q}$ and the relation names as a matrix $\mathbf{R}$, the frozen GFM forward pass computes the entity relevance scores:
\begin{equation}
    \mathbf{s}_e = \retriever(\kg, \mathbf{q}, \mathbf{m}_q; \phi) \in \R^n.
\end{equation}
The entity scores are then projected into document scores. Let $\mathbf{M}_{\TopK}(\mathbf{s}_e)$ denote the masking operation that retains only the top-$K$ entity scores, and let $\mathbf{w}_{\mathrm{idf}}$ denote the inverse-frequency weights defined according to the document frequency of each entity. The four document-scoring modes can be written uniformly as
\begin{equation}
    \tilde{\mathbf{s}}_e=
    \begin{cases}
        \mathbf{s}_e, & \texttt{raw},\\
        \mathbf{M}_{\TopK}(\mathbf{s}_e), & \texttt{topk},\\
        \mathbf{w}_{\mathrm{idf}}\odot \mathbf{s}_e, & \texttt{idf},\\
        \mathbf{w}_{\mathrm{idf}}\odot \mathbf{M}_{\TopK}(\mathbf{s}_e), & \texttt{idf\_topk},
    \end{cases}
\end{equation}
and the document scores are obtained by
\begin{equation}
    \mathbf{s}_d = \mathbf{M}^{\top}\tilde{\mathbf{s}}_e.
\end{equation}
We then take $\TopK(\mathbf{s}_d)$ as the retrieval result. In actual use, we also enable \texttt{init\_entities\_weight}, that is, during the GFM forward pass, a $1/f(e)$ weight is applied to high-frequency entities to suppress the dominance of entities connected to too many documents in the retrieval results.
\section{Additional Detailed Experimental Results}
\label{app:additional_detailed_experimental_results}

The results of retrieval performance on multi-hop QA benchmarks are in Table~\ref{tab:res_retrieve_multihop}.

\begin{table*}[t]
\renewcommand{\arraystretch}{0.85}
\setlength{\tabcolsep}{3.8pt}
  \centering
  \caption{Results of retrieval performance on multi-hop QA benchmarks. We report document-level Recall (\%) at top-2 and top-5. Best results are in \textbf{bold} and runner-ups are \underline{underlined}. The darker the cell, the better.}
  \vspace{0.1cm}
  \resizebox{\textwidth}{!}{
    \begin{tabular}{lccccccr}
    \toprule
    \textbf{Dataset} & \multicolumn{2}{c}{\texttt{HotpotQA}}  & \multicolumn{2}{c}{\texttt{MuSiQue}}   & \multicolumn{2}{c}{\texttt{2WikiMultiHopQA}} & \multirow{2}[2]{*}{\textbf{Avg. Rank}} \\
    \cmidrule(r{1mm}){1-1}  \cmidrule(l{0.5mm}r{1mm}){2-3} \cmidrule(l{0.5mm}r{1mm}){4-5} \cmidrule(l{0.5mm}r{1mm}){6-7} 
    \textbf{Method} & R@2 & R@5 & R@2 & R@5 & R@2 & R@5 & \\
    \cmidrule(r{1mm}){1-1}  \cmidrule(l{0.5mm}r{1mm}){2-3} \cmidrule(l{0.5mm}r{1mm}){4-5} \cmidrule(l{0.5mm}r{1mm}){6-7} \cmidrule(l{0.5mm}){8-8}

    \texttt{BM25} \scalebox{0.68}{(\mycite{robertson1994some}~\textit{SIGIR'94})} & \cellcolor{myblue!28.19}{55.\scalebox{0.75}{4}} & \cellcolor{myblue!36.10}{72.\scalebox{0.75}{2}} & \cellcolor{myblue!24.27}{32.\scalebox{0.75}{3}} & \cellcolor{myblue!21.05}{41.\scalebox{0.75}{2}} & \cellcolor{myblue!12.92}{51.\scalebox{0.75}{8}} & \cellcolor{myblue!15.74}{61.\scalebox{0.75}{9}} & \cellcolor{myred!10.54}{18.\scalebox{0.75}{2}} \\
    \texttt{Contriever} \scalebox{0.68}{(\mycite{izacard2022unsupervised}~\textit{TMLR'22})} & \cellcolor{myblue!31.03}{57.\scalebox{0.75}{2}} & \cellcolor{myblue!42.53}{75.\scalebox{0.75}{5}} & \cellcolor{myblue!31.69}{34.\scalebox{0.75}{8}} & \cellcolor{myblue!36.88}{46.\scalebox{0.75}{6}} & \cellcolor{myblue!4.92}{46.\scalebox{0.75}{6}} & \cellcolor{myblue!8.27}{57.\scalebox{0.75}{5}} & \cellcolor{myred!14.20}{17.\scalebox{0.75}{2}} \\
    \texttt{GTR} \scalebox{0.68}{(\mycite{ni2022large}~\textit{EMNLP'22})} & \cellcolor{myblue!34.50}{59.\scalebox{0.75}{4}} & \cellcolor{myblue!38.24}{73.\scalebox{0.75}{3}} & \cellcolor{myblue!39.41}{37.\scalebox{0.75}{4}} & \cellcolor{myblue!44.21}{49.\scalebox{0.75}{1}} & \cellcolor{myblue!25.85}{60.\scalebox{0.75}{2}} & \cellcolor{myblue!25.91}{67.\scalebox{0.75}{9}} & \cellcolor{myred!26.39}{13.\scalebox{0.75}{8}} \\
    \texttt{ColBERTv2} \scalebox{0.68}{(\mycite{santhanam2022colbertv2}~\textit{NAACL'22})} & \cellcolor{myblue!42.86}{64.\scalebox{0.75}{7}} & \cellcolor{myblue!49.93}{79.\scalebox{0.75}{3}} & \cellcolor{myblue!40.90}{37.\scalebox{0.75}{9}} & \cellcolor{myblue!44.50}{49.\scalebox{0.75}{2}} & \cellcolor{myblue!24.31}{59.\scalebox{0.75}{2}} & \cellcolor{myblue!26.42}{68.\scalebox{0.75}{2}} & \cellcolor{myred!34.93}{11.\scalebox{0.75}{5}} \\
    \texttt{RAPTOR} \scalebox{0.68}{(\mycite{sarthi2024raptor}~\textit{ICLR'24})} & \cellcolor{myblue!32.45}{58.\scalebox{0.75}{1}} & \cellcolor{myblue!34.15}{71.\scalebox{0.75}{2}} & \cellcolor{myblue!34.37}{35.\scalebox{0.75}{7}} & \cellcolor{myblue!33.07}{45.\scalebox{0.75}{3}} & \cellcolor{myblue!4.46}{46.\scalebox{0.75}{3}} & \cellcolor{myblue!2.00}{53.\scalebox{0.75}{8}} & \cellcolor{myred!12.37}{17.\scalebox{0.75}{7}} \\
    \texttt{Proposition} \scalebox{0.68}{(\mycite{chen2024dense}~\textit{EMNLP'24})} & \cellcolor{myblue!33.40}{58.\scalebox{0.75}{7}} & \cellcolor{myblue!33.95}{71.\scalebox{0.75}{1}} & \cellcolor{myblue!40.01}{37.\scalebox{0.75}{6}} & \cellcolor{myblue!44.79}{49.\scalebox{0.75}{3}} & \cellcolor{myblue!20.00}{56.\scalebox{0.75}{4}} & \cellcolor{myblue!17.77}{63.\scalebox{0.75}{1}} & \cellcolor{myred!22.74}{14.\scalebox{0.75}{8}} \\

    \cmidrule(r{1mm}){1-1}  \cmidrule(l{0.5mm}r{1mm}){2-3} \cmidrule(l{0.5mm}r{1mm}){4-5} \cmidrule(l{0.5mm}r{1mm}){6-7} \cmidrule(l{0.5mm}){8-8}

    \texttt{GraphRAG} \scalebox{0.68}{(\mycite{edge2024local}~\textit{arXiv'24})} & \cellcolor{myblue!32.77}{58.\scalebox{0.75}{3}} & \cellcolor{myblue!44.67}{76.\scalebox{0.75}{6}} & \cellcolor{myblue!33.48}{35.\scalebox{0.75}{4}} & \cellcolor{myblue!44.79}{49.\scalebox{0.75}{3}} & \cellcolor{myblue!28.00}{61.\scalebox{0.75}{6}} & \cellcolor{myblue!41.85}{77.\scalebox{0.75}{3}} & \cellcolor{myred!32.49}{12.\scalebox{0.75}{2}} \\
    \texttt{G-Retriever} \scalebox{0.68}{(\mycite{he2024g}~\textit{NeurIPS'24})} & \cellcolor{myblue!24.88}{53.\scalebox{0.75}{3}} & \cellcolor{myblue!23.04}{65.\scalebox{0.75}{5}} & \cellcolor{myblue!43.57}{38.\scalebox{0.75}{8}} & \cellcolor{myblue!32.48}{45.\scalebox{0.75}{1}} & \cellcolor{myblue!26.77}{60.\scalebox{0.75}{8}} & \cellcolor{myblue!25.74}{67.\scalebox{0.75}{8}} & \cellcolor{myred!19.69}{15.\scalebox{0.75}{7}} \\
    \texttt{LightRAG} \scalebox{0.68}{(\mycite{guo2024lightrag}~\textit{arXiv'24})} & \cellcolor{myblue!2.00}{38.\scalebox{0.75}{8}} & \cellcolor{myblue!2.00}{54.\scalebox{0.75}{7}} & \cellcolor{myblue!2.00}{24.\scalebox{0.75}{8}} & \cellcolor{myblue!2.00}{34.\scalebox{0.75}{7}} & \cellcolor{myblue!2.62}{45.\scalebox{0.75}{1}} & \cellcolor{myblue!10.99}{59.\scalebox{0.75}{1}} & \cellcolor{myred!2.00}{20.\scalebox{0.75}{5}} \\
    \texttt{HippoRAG} \scalebox{0.68}{(\mycite{gutierrez2024hipporag}~\textit{NeurIPS'24})} & \cellcolor{myblue!35.61}{60.\scalebox{0.75}{1}} & \cellcolor{myblue!48.37}{78.\scalebox{0.75}{5}} & \cellcolor{myblue!50.70}{41.\scalebox{0.75}{2}} & \cellcolor{myblue!56.22}{53.\scalebox{0.75}{2}} & \cellcolor{myblue!38.46}{68.\scalebox{0.75}{4}} & \cellcolor{myblue!58.30}{87.\scalebox{0.75}{0}} & \cellcolor{myred!45.30}{8.\scalebox{0.75}{7}} \\
    \texttt{HippoRAG} \texttt{2}
\scalebox{0.68}{(\mycite{gutierrez2025rag}~\textit{ICML'25})} & \cellcolor{myblue!67.79}{\underline{~80.\scalebox{0.75}{5}~}} & \cellcolor{myblue!67.08}{\underline{~88.\scalebox{0.75}{1}~}} & \cellcolor{myblue!67.92}{\underline{~47.\scalebox{0.75}{0}~}} & \cellcolor{myblue!66.48}{56.\scalebox{0.75}{7}} & \cellcolor{myblue!70.00}{\textbf{88.\scalebox{0.75}{9}}} & \cellcolor{myblue!63.56}{90.\scalebox{0.75}{1}} & \cellcolor{myred!68.17}{\underline{~2.\scalebox{0.75}{4}~}} \\
    \texttt{SubgraphRAG} \scalebox{0.68}{(\mycite{li2025simple}~\textit{ICLR'25})} & \cellcolor{myblue!37.81}{61.\scalebox{0.75}{5}} & \cellcolor{myblue!37.66}{73.\scalebox{0.75}{0}} & \cellcolor{myblue!53.37}{42.\scalebox{0.75}{1}} & \cellcolor{myblue!44.79}{49.\scalebox{0.75}{3}} & \cellcolor{myblue!42.00}{70.\scalebox{0.75}{7}} & \cellcolor{myblue!55.76}{85.\scalebox{0.75}{5}} & \cellcolor{myred!41.64}{9.\scalebox{0.75}{7}} \\
    \texttt{PropRAG} \scalebox{0.68}{(\mycite{wang2025proprag}~\textit{EMNLP'25})} & \cellcolor{myblue!70.00}{\textbf{81.\scalebox{0.75}{9}}} & \cellcolor{myblue!66.88}{88.\scalebox{0.75}{0}} & \cellcolor{myblue!70.00}{\textbf{47.\scalebox{0.75}{7}}} & \cellcolor{myblue!70.00}{\textbf{57.\scalebox{0.75}{9}}} & \cellcolor{myblue!68.46}{\underline{~87.\scalebox{0.75}{9}~}} & \cellcolor{myblue!63.56}{90.\scalebox{0.75}{1}} & \cellcolor{myred!70.00}{\textbf{1.\scalebox{0.75}{9}}} \\
    \texttt{GFM-RAG} \scalebox{0.68}{(\mycite{luo2025gfm}~\textit{NeurIPS'25})} & \cellcolor{myblue!60.06}{75.\scalebox{0.75}{6}} & \cellcolor{myblue!70.00}{\textbf{89.\scalebox{0.75}{6}}} & \cellcolor{myblue!57.53}{43.\scalebox{0.75}{5}} & \cellcolor{myblue!69.12}{\underline{~57.\scalebox{0.75}{6}~}} & \cellcolor{myblue!54.92}{79.\scalebox{0.75}{1}} & \cellcolor{myblue!67.46}{\underline{~92.\scalebox{0.75}{4}~}} & \cellcolor{myred!66.34}{2.\scalebox{0.75}{9}} \\

    \cmidrule(r{1mm}){1-1}  \cmidrule(l{0.5mm}r{1mm}){2-3} \cmidrule(l{0.5mm}r{1mm}){4-5} \cmidrule(l{0.5mm}r{1mm}){6-7} \cmidrule(l{0.5mm}){8-8}

    \texttt{FLARE} \scalebox{0.68}{(\mycite{jiang2023active}~\textit{EMNLP'23})} & \cellcolor{myblue!56.12}{73.\scalebox{0.75}{1}} & \cellcolor{myblue!53.83}{81.\scalebox{0.75}{3}} & \cellcolor{myblue!59.90}{44.\scalebox{0.75}{3}} & \cellcolor{myblue!61.79}{55.\scalebox{0.75}{1}} & \cellcolor{myblue!36.46}{67.\scalebox{0.75}{1}} & \cellcolor{myblue!34.73}{73.\scalebox{0.75}{1}} & \cellcolor{myred!53.23}{6.\scalebox{0.75}{5}} \\
    \texttt{Adaptive-RAG} \scalebox{0.68}{(\mycite{jeong2024adaptive}~\textit{NAACL'24})} & \cellcolor{myblue!37.03}{61.\scalebox{0.75}{0}} & \cellcolor{myblue!44.28}{76.\scalebox{0.75}{4}} & \cellcolor{myblue!32.59}{35.\scalebox{0.75}{1}} & \cellcolor{myblue!31.31}{44.\scalebox{0.75}{7}} & \cellcolor{myblue!2.00}{44.\scalebox{0.75}{7}} & \cellcolor{myblue!14.89}{61.\scalebox{0.75}{4}} & \cellcolor{myred!16.33}{16.\scalebox{0.75}{6}} \\
    \texttt{BM25} + \texttt{IRCoT} \scalebox{0.68}{(\mycite{trivedi2023interleaving}~\textit{ACL'23})} & \cellcolor{myblue!44.28}{65.\scalebox{0.75}{6}} & \cellcolor{myblue!49.35}{79.\scalebox{0.75}{0}} & \cellcolor{myblue!29.91}{34.\scalebox{0.75}{2}} & \cellcolor{myblue!31.31}{44.\scalebox{0.75}{7}} & \cellcolor{myblue!27.38}{61.\scalebox{0.75}{2}} & \cellcolor{myblue!38.97}{75.\scalebox{0.75}{6}} & \cellcolor{myred!31.58}{12.\scalebox{0.75}{4}} \\
    \texttt{Contriever} + \texttt{IRCoT} \scalebox{0.68}{(\mycite{trivedi2023interleaving}~\textit{ACL'23})} & \cellcolor{myblue!44.76}{65.\scalebox{0.75}{9}} & \cellcolor{myblue!54.41}{81.\scalebox{0.75}{6}} & \cellcolor{myblue!44.46}{39.\scalebox{0.75}{1}} & \cellcolor{myblue!53.29}{52.\scalebox{0.75}{2}} & \cellcolor{myblue!12.62}{51.\scalebox{0.75}{6}} & \cellcolor{myblue!18.96}{63.\scalebox{0.75}{8}} & \cellcolor{myred!37.98}{10.\scalebox{0.75}{7}} \\
    \texttt{ColBERTv2} + \texttt{IRCoT} \scalebox{0.68}{(\mycite{trivedi2023interleaving}~\textit{ACL'23})} & \cellcolor{myblue!47.91}{67.\scalebox{0.75}{9}} & \cellcolor{myblue!55.19}{82.\scalebox{0.75}{0}} & \cellcolor{myblue!52.18}{41.\scalebox{0.75}{7}} & \cellcolor{myblue!57.69}{53.\scalebox{0.75}{7}} & \cellcolor{myblue!31.85}{64.\scalebox{0.75}{1}} & \cellcolor{myblue!36.93}{74.\scalebox{0.75}{4}} & \cellcolor{myred!50.79}{7.\scalebox{0.75}{2}} \\
    \texttt{HippoRAG} + \texttt{IRCoT} \scalebox{0.68}{(\mycite{trivedi2023interleaving}~\textit{ACL'23})} & \cellcolor{myblue!46.49}{67.\scalebox{0.75}{0}} & \cellcolor{myblue!57.14}{83.\scalebox{0.75}{0}} & \cellcolor{myblue!62.87}{45.\scalebox{0.75}{3}} & \cellcolor{myblue!69.12}{\underline{~57.\scalebox{0.75}{6}~}} & \cellcolor{myblue!49.85}{75.\scalebox{0.75}{8}} & \cellcolor{myblue!70.00}{\textbf{93.\scalebox{0.75}{9}}} & \cellcolor{myred!63.90}{3.\scalebox{0.75}{6}} \\

    \cmidrule(r{1mm}){1-1}  \cmidrule(l{0.5mm}r{1mm}){2-3} \cmidrule(l{0.5mm}r{1mm}){4-5} \cmidrule(l{0.5mm}r{1mm}){6-7} \cmidrule(l{0.5mm}){8-8}

    \textbf{\method{} (ours)} & \cellcolor{myblue!43.49}{65.\scalebox{0.75}{1}} & \cellcolor{myblue!46.62}{77.\scalebox{0.75}{6}} & \cellcolor{myblue!56.64}{43.\scalebox{0.75}{2}} & \cellcolor{myblue!55.93}{53.\scalebox{0.75}{1}} & \cellcolor{myblue!61.85}{83.\scalebox{0.75}{6}} & \cellcolor{myblue!61.01}{88.\scalebox{0.75}{6}} & \cellcolor{myred!51.40}{7.\scalebox{0.75}{0}} \\
    \bottomrule
    \end{tabular}%
  }
  \label{tab:res_retrieve_multihop}%
\end{table*}%

The results on AmazonQA are in Table~\ref{tab:amazonqa_fulltest_baselines}.

\begin{table*}[t]
\renewcommand{\arraystretch}{0.85}
\setlength{\tabcolsep}{5.5pt}
  \centering
  \caption{Performance of representative baselines on the original \texttt{AmazonQA} full-test protocol. BLEU-1/2/3/4 are denoted as B-1/2/3/4, and R denotes ROUGE. Best results are in \textbf{bold} and runner-ups are \underline{underlined}. \textbf{\textcolor{red}{Only rows marked with \zsmark{} are our zero-shot transfer results; baseline rows and trained variants are not marked as zero-shot.}}}
  \vspace{0.1cm}
  \resizebox{\textwidth}{!}{
    \begin{tabular}{lccccc}
    \toprule
    \multicolumn{6}{c}{\cellcolor{red!8}\textbf{\textcolor{red}{Zero-shot setting applies only to \texttt{Ours} rows marked with \zsmark{} on \texttt{AmazonQA}.}}} \\
    \midrule
    \textbf{Method} & B-1 & B-2 & B-3 & B-4 & R \\
    \midrule
    \multicolumn{6}{l}{\textit{Heuristic baselines from the original \texttt{AmazonQA} protocol}} \\
    \texttt{Random Sentence} \scalebox{0.68}{(\mycite{gupta2019amazonqa}~\textit{IJCAI'19})}
      & \cellcolor{myblue!56.62}{78.\scalebox{0.75}{56}}
      & \cellcolor{myblue!53.21}{63.\scalebox{0.75}{95}}
      & \cellcolor{myblue!44.74}{44.\scalebox{0.75}{37}}
      & \cellcolor{myblue!35.16}{29.\scalebox{0.75}{87}}
      & \cellcolor{myblue!34.08}{49.\scalebox{0.75}{12}} \\
    \texttt{Top-1 using IR} \scalebox{0.68}{(\mycite{gupta2019amazonqa}~\textit{IJCAI'19})}
      & \cellcolor{myblue!66.93}{\underline{~89.\scalebox{0.75}{49}~}}
      & \cellcolor{myblue!65.79}{\underline{~74.\scalebox{0.75}{80}~}}
      & \cellcolor{myblue!63.66}{\underline{~56.\scalebox{0.75}{76}~}}
      & \cellcolor{myblue!61.28}{\underline{~43.\scalebox{0.75}{52}~}}
      & \cellcolor{myblue!66.82}{61.\scalebox{0.75}{48}} \\
    \texttt{Top-1 Using BLEU} \scalebox{0.68}{(\mycite{gupta2019amazonqa}~\textit{IJCAI'19})}
      & \cellcolor{myblue!70.00}{\textbf{92.\scalebox{0.75}{74}}}
      & \cellcolor{myblue!70.00}{\textbf{78.\scalebox{0.75}{43}}}
      & \cellcolor{myblue!70.00}{\textbf{60.\scalebox{0.75}{91}}}
      & \cellcolor{myblue!70.00}{\textbf{48.\scalebox{0.75}{08}}}
      & \cellcolor{myblue!70.00}{\textbf{62.\scalebox{0.75}{68}}} \\
    \texttt{Top-1 Helpfulness} \scalebox{0.68}{(\mycite{gupta2019amazonqa}~\textit{IJCAI'19})}
      & \cellcolor{myblue!2.00}{20.\scalebox{0.75}{66}}
      & \cellcolor{myblue!2.00}{19.\scalebox{0.75}{78}}
      & \cellcolor{myblue!2.00}{16.\scalebox{0.75}{39}}
      & \cellcolor{myblue!2.00}{12.\scalebox{0.75}{54}}
      & \cellcolor{myblue!2.00}{37.\scalebox{0.75}{01}} \\
    \texttt{Top-1 Wilson Score} \scalebox{0.68}{(\mycite{gupta2019amazonqa}~\textit{IJCAI'19})}
      & \cellcolor{myblue!2.08}{20.\scalebox{0.75}{74}}
      & \cellcolor{myblue!2.07}{19.\scalebox{0.75}{84}}
      & \cellcolor{myblue!2.08}{16.\scalebox{0.75}{44}}
      & \cellcolor{myblue!2.08}{12.\scalebox{0.75}{58}}
      & \cellcolor{myblue!2.66}{37.\scalebox{0.75}{26}} \\
    \cmidrule(r{1mm}){1-1}
    \cmidrule(l{0.5mm}){2-6}
    \multicolumn{6}{l}{\textit{Neural baseline from the original \texttt{AmazonQA} protocol}} \\
    \texttt{R-Net} \scalebox{0.68}{(\mycite{gupta2019amazonqa}~\textit{IJCAI'19})}
      & \cellcolor{myblue!26.89}{47.\scalebox{0.75}{04}}
      & \cellcolor{myblue!25.81}{40.\scalebox{0.75}{32}}
      & \cellcolor{myblue!25.05}{31.\scalebox{0.75}{48}}
      & \cellcolor{myblue!23.77}{23.\scalebox{0.75}{92}}
      & \cellcolor{myblue!10.50}{40.\scalebox{0.75}{22}} \\
    \cmidrule(r{1mm}){1-1}
    \cmidrule(l{0.5mm}){2-6}
    \multicolumn{6}{l}{\textit{Human answers under the original \texttt{AmazonQA} protocol}} \\
    \texttt{Amazon User Community} \scalebox{0.68}{(\mycite{gupta2019amazonqa}~\textit{IJCAI'19})}
      & \cellcolor{myblue!58.81}{80.\scalebox{0.75}{88}}
      & \cellcolor{myblue!58.90}{68.\scalebox{0.75}{86}}
      & \cellcolor{myblue!60.00}{54.\scalebox{0.75}{36}}
      & \cellcolor{myblue!58.39}{42.\scalebox{0.75}{01}}
      & \cellcolor{myblue!68.68}{\underline{~62.\scalebox{0.75}{18}~}} \\
    \texttt{Expert (Spans)} \scalebox{0.68}{(\mycite{gupta2019amazonqa}~\textit{IJCAI'19})}
      & \cellcolor{myblue!46.97}{68.\scalebox{0.75}{33}}
      & \cellcolor{myblue!46.07}{57.\scalebox{0.75}{79}}
      & \cellcolor{myblue!45.10}{44.\scalebox{0.75}{61}}
      & \cellcolor{myblue!43.88}{34.\scalebox{0.75}{43}}
      & \cellcolor{myblue!39.30}{51.\scalebox{0.75}{09}} \\
    \texttt{Expert (Descriptive)} \scalebox{0.68}{(\mycite{gupta2019amazonqa}~\textit{IJCAI'19})}
      & \cellcolor{myblue!33.14}{53.\scalebox{0.75}{67}}
      & \cellcolor{myblue!33.05}{46.\scalebox{0.75}{56}}
      & \cellcolor{myblue!34.72}{37.\scalebox{0.75}{81}}
      & \cellcolor{myblue!36.86}{30.\scalebox{0.75}{76}}
      & \cellcolor{myblue!45.18}{53.\scalebox{0.75}{31}} \\
    \cmidrule(r{1mm}){1-1}
    \cmidrule(l{0.5mm}){2-6}
    \multicolumn{6}{l}{\textit{Our method}} \\
    \textbf{\texttt{Ours} (0-shot)\zsmark}
      & \cellcolor{myblue!40.85}{61.\scalebox{0.75}{84}\zsmark}
      & \cellcolor{myblue!36.30}{49.\scalebox{0.75}{36}\zsmark}
      & \cellcolor{myblue!34.73}{37.\scalebox{0.75}{82}\zsmark}
      & \cellcolor{myblue!32.36}{28.\scalebox{0.75}{41}\zsmark}
      & \cellcolor{myblue!27.75}{46.\scalebox{0.75}{73}\zsmark} \\
    \textbf{\texttt{Ours} (trained)}
      & \cellcolor{myblue!53.19}{74.\scalebox{0.75}{92}}
      & \cellcolor{myblue!50.46}{61.\scalebox{0.75}{58}}
      & \cellcolor{myblue!49.72}{47.\scalebox{0.75}{63}}
      & \cellcolor{myblue!46.62}{35.\scalebox{0.75}{86}}
      & \cellcolor{myblue!49.44}{54.\scalebox{0.75}{92}} \\
    \textbf{\texttt{Ours} +1 round}
      & \cellcolor{myblue!60.58}{82.\scalebox{0.75}{76}}
      & \cellcolor{myblue!58.96}{68.\scalebox{0.75}{91}}
      & \cellcolor{myblue!57.52}{52.\scalebox{0.75}{74}}
      & \cellcolor{myblue!53.93}{39.\scalebox{0.75}{68}}
      & \cellcolor{myblue!59.80}{58.\scalebox{0.75}{83}} \\
    \bottomrule
    \end{tabular}%
  }
  \label{tab:amazonqa_fulltest_baselines}%
\end{table*}%

The HaluMem results are shown in Table~\ref{tab:halumem_medium_baselines}. 

\begin{table*}[t]
\renewcommand{\arraystretch}{0.85}
\setlength{\tabcolsep}{3.8pt}
\centering

\caption{
Results on \texttt{HaluMem-Medium}.
We report memory extraction, memory updating, and memory question-answering
metrics.
R denotes Recall, W-R denotes Weighted Recall, T-P denotes Target Memory
Precision, Acc. denotes Memory Accuracy, FMR denotes False Memory Resistance,
F1 denotes Memory Extraction F1-score, C denotes Correct Rate, H denotes
Hallucination Rate, and O denotes Omission Rate.
For R, W-R, T-P, Acc., FMR, F1, and C, higher is better; for H and O, lower
is better.
Best numerical results are in \textbf{bold}, and runner-ups are
\underline{underlined}; ties at the runner-up value are all underlined. Missing entries are denoted by ``--''.
}
\label{tab:halumem_medium_baselines}
\vspace{0.1cm}

\resizebox{\textwidth}{!}{
\begin{tabular}{lcccccccccccc}
\toprule
\textbf{Dataset}
& \multicolumn{6}{c}{\texttt{Memory Extraction}}
& \multicolumn{3}{c}{\texttt{Memory Updating}}
& \multicolumn{3}{c}{\texttt{Memory QA}} \\
\cmidrule(r{1mm}){1-1}
\cmidrule(l{0.5mm}r{1mm}){2-7}
\cmidrule(l{0.5mm}r{1mm}){8-10}
\cmidrule(l{0.5mm}){11-13}
\textbf{Method}
& R$\uparrow$
& W-R$\uparrow$
& T-P$\uparrow$
& Acc.$\uparrow$
& FMR$\uparrow$
& F1$\uparrow$
& C$\uparrow$
& H$\downarrow$
& O$\downarrow$
& C$\uparrow$
& H$\downarrow$
& O$\downarrow$ \\
\cmidrule(r{1mm}){1-1}
\cmidrule(l{0.5mm}r{1mm}){2-7}
\cmidrule(l{0.5mm}r{1mm}){8-10}
\cmidrule(l{0.5mm}){11-13}

\multicolumn{13}{l}{\textit{Other memory-system baselines}} \\

\texttt{Supermemory}\footnotemark[1]
& \cellcolor{myblue!29.9}{41.\scalebox{0.75}{53}}
& \cellcolor{myblue!44.6}{64.\scalebox{0.75}{76}}
& \cellcolor{myblue!48.2}{90.\scalebox{0.75}{32}}
& \cellcolor{myblue!56.1}{60.\scalebox{0.75}{83}}
& \cellcolor{myblue!13.8}{51.\scalebox{0.75}{77}}
& \cellcolor{myblue!37.9}{56.\scalebox{0.75}{90}}
& \cellcolor{myblue!13.4}{16.\scalebox{0.75}{37}}
& \cellcolor{myred!10.0}{1.\scalebox{0.75}{15}}
& \cellcolor{myred!14.1}{82.\scalebox{0.75}{47}}
& \cellcolor{myblue!38.1}{54.\scalebox{0.75}{07}}
& \cellcolor{myred!30.4}{22.\scalebox{0.75}{24}}
& \cellcolor{myred!45.0}{23.\scalebox{0.75}{69}} \\

\texttt{Mem0}
\scalebox{0.68}{(\mycite{chhikara2025mem0}~\textit{arXiv'25})}
& \cellcolor{myblue!31.3}{42.\scalebox{0.75}{91}}
& \cellcolor{myblue!44.9}{65.\scalebox{0.75}{03}}
& \cellcolor{myblue!10.1}{86.\scalebox{0.75}{26}}
& \cellcolor{myblue!56.1}{60.\scalebox{0.75}{86}}
& \cellcolor{myblue!22.5}{56.\scalebox{0.75}{80}}
& \cellcolor{myblue!38.3}{57.\scalebox{0.75}{31}}
& \cellcolor{myblue!22.7}{25.\scalebox{0.75}{50}}
& \cellcolor{myred!64.6}{0.\scalebox{0.75}{45}}
& \cellcolor{myred!22.7}{74.\scalebox{0.75}{02}}
& \cellcolor{myblue!36.1}{53.\scalebox{0.75}{02}}
& \cellcolor{myred!41.7}{19.\scalebox{0.75}{17}}
& \cellcolor{myred!28.9}{27.\scalebox{0.75}{81}} \\

\texttt{Zep}
\scalebox{0.68}{(\mycite{rasmussen2025zep}~\textit{arXiv'25})}
& \placeholdercell
& \placeholdercell
& \placeholdercell
& \placeholdercell
& \placeholdercell
& \placeholdercell
& \cellcolor{myblue!45.0}{47.\scalebox{0.75}{28}}
& \cellcolor{myred!67.3}{\underline{~0.\scalebox{0.75}{42}~}}
& \cellcolor{myred!44.9}{52.\scalebox{0.75}{31}}
& \cellcolor{myblue!40.8}{55.\scalebox{0.75}{47}}
& \cellcolor{myred!31.6}{21.\scalebox{0.75}{92}}
& \cellcolor{myred!49.1}{22.\scalebox{0.75}{62}} \\

\cmidrule(r{1mm}){1-1}
\cmidrule(l{0.5mm}r{1mm}){2-7}
\cmidrule(l{0.5mm}r{1mm}){8-10}
\cmidrule(l{0.5mm}){11-13}
\multicolumn{13}{l}{\textit{Matched baseline--\method{} comparisons}} \\

\texttt{MemOS}
\scalebox{0.68}{(\mycite{li2025memos}~\textit{arXiv'25})}
& \cellcolor{myblue!63.5}{74.\scalebox{0.75}{07}}
& \cellcolor{myblue!66.6}{84.\scalebox{0.75}{81}}
& \cellcolor{myblue!10.0}{86.\scalebox{0.75}{25}}
& \cellcolor{myblue!53.6}{59.\scalebox{0.75}{55}}
& \cellcolor{myblue!10.0}{44.\scalebox{0.75}{94}}
& \cellcolor{myblue!63.6}{79.\scalebox{0.75}{70}}
& \cellcolor{myblue!60.1}{62.\scalebox{0.75}{11}}
& \cellcolor{myred!67.3}{\underline{~0.\scalebox{0.75}{42}~}}
& \cellcolor{myred!60.1}{37.\scalebox{0.75}{48}}
& \cellcolor{myblue!63.5}{67.\scalebox{0.75}{23}}
& \cellcolor{myred!56.4}{15.\scalebox{0.75}{17}}
& \cellcolor{myred!68.8}{\underline{~17.\scalebox{0.75}{59}~}} \\

\texttt{MemOS} + \method{}
& \cellcolor{myblue!68.1}{\underline{~78.\scalebox{0.75}{46}~}}
& \cellcolor{myblue!69.9}{\underline{~87.\scalebox{0.75}{86}~}}
& \cellcolor{myblue!40.7}{89.\scalebox{0.75}{66}}
& \cellcolor{myblue!61.8}{63.\scalebox{0.75}{88}}
& \cellcolor{myblue!51.0}{73.\scalebox{0.75}{26}}
& \cellcolor{myblue!68.1}{\underline{~83.\scalebox{0.75}{69}~}}
& \cellcolor{myblue!68.4}{\underline{~70.\scalebox{0.75}{25}~}}
& \cellcolor{myred!67.3}{\underline{~0.\scalebox{0.75}{42}~}}
& \cellcolor{myred!68.4}{\underline{~29.\scalebox{0.75}{33}~}}
& \cellcolor{myblue!68.1}{\underline{~69.\scalebox{0.75}{65}~}}
& \cellcolor{myred!64.2}{\underline{~13.\scalebox{0.75}{07}~}}
& \cellcolor{myred!70.0}{\textbf{17.\scalebox{0.75}{27}}} \\

\texttt{MemOS} + \method{} (+1 round)
& \cellcolor{myblue!70.0}{\textbf{80.\scalebox{0.75}{31}}}
& \cellcolor{myblue!70.0}{\textbf{87.\scalebox{0.75}{93}}}
& \cellcolor{myblue!56.2}{\underline{~91.\scalebox{0.75}{02}~}}
& \cellcolor{myblue!66.6}{66.\scalebox{0.75}{41}}
& \cellcolor{myblue!52.5}{74.\scalebox{0.75}{09}}
& \cellcolor{myblue!70.0}{\textbf{85.\scalebox{0.75}{33}}}
& \cellcolor{myblue!70.0}{\textbf{71.\scalebox{0.75}{80}}}
& \cellcolor{myred!70.0}{\textbf{0.\scalebox{0.75}{39}}}
& \cellcolor{myred!70.0}{\textbf{27.\scalebox{0.75}{81}}}
& \cellcolor{myblue!70.0}{\textbf{70.\scalebox{0.75}{61}}}
& \cellcolor{myred!70.0}{\textbf{11.\scalebox{0.75}{48}}}
& \cellcolor{myred!67.5}{17.\scalebox{0.75}{91}} \\

\addlinespace[1.5pt]

\texttt{Memobase}\footnotemark[2]
& \cellcolor{myblue!10.0}{14.\scalebox{0.75}{55}}
& \cellcolor{myblue!10.0}{25.\scalebox{0.75}{88}}
& \cellcolor{myblue!70.0}{\textbf{92.\scalebox{0.75}{24}}}
& \cellcolor{myblue!10.0}{32.\scalebox{0.75}{29}}
& \cellcolor{myblue!64.0}{80.\scalebox{0.75}{78}}
& \cellcolor{myblue!10.0}{25.\scalebox{0.75}{13}}
& \cellcolor{myblue!10.0}{5.\scalebox{0.75}{20}}
& \cellcolor{myred!55.7}{0.\scalebox{0.75}{55}}
& \cellcolor{myred!10.0}{94.\scalebox{0.75}{25}}
& \cellcolor{myblue!10.0}{35.\scalebox{0.75}{33}}
& \cellcolor{myred!10.0}{29.\scalebox{0.75}{97}}
& \cellcolor{myred!10.0}{34.\scalebox{0.75}{71}} \\

\texttt{Memobase} + \method{}
& \cellcolor{myblue!43.6}{54.\scalebox{0.75}{79}}
& \cellcolor{myblue!55.6}{74.\scalebox{0.75}{75}}
& \cellcolor{myblue!45.9}{90.\scalebox{0.75}{12}}
& \cellcolor{myblue!67.2}{\underline{~66.\scalebox{0.75}{72}~}}
& \cellcolor{myblue!67.4}{\underline{~82.\scalebox{0.75}{73}~}}
& \cellcolor{myblue!50.6}{68.\scalebox{0.75}{15}}
& \cellcolor{myblue!31.6}{34.\scalebox{0.75}{16}}
& \cellcolor{myred!59.3}{0.\scalebox{0.75}{51}}
& \cellcolor{myred!31.6}{65.\scalebox{0.75}{33}}
& \cellcolor{myblue!31.1}{50.\scalebox{0.75}{43}}
& \cellcolor{myred!31.4}{21.\scalebox{0.75}{97}}
& \cellcolor{myred!29.7}{27.\scalebox{0.75}{60}} \\

\texttt{Memobase} + \method{} (+1 round)
& \cellcolor{myblue!47.3}{58.\scalebox{0.75}{37}}
& \cellcolor{myblue!56.2}{75.\scalebox{0.75}{30}}
& \cellcolor{myblue!42.9}{89.\scalebox{0.75}{85}}
& \cellcolor{myblue!70.0}{\textbf{68.\scalebox{0.75}{19}}}
& \cellcolor{myblue!70.0}{\textbf{84.\scalebox{0.75}{23}}}
& \cellcolor{myblue!53.6}{70.\scalebox{0.75}{77}}
& \cellcolor{myblue!30.7}{33.\scalebox{0.75}{28}}
& \cellcolor{myred!63.7}{0.\scalebox{0.75}{46}}
& \cellcolor{myred!30.6}{66.\scalebox{0.75}{26}}
& \cellcolor{myblue!41.5}{55.\scalebox{0.75}{84}}
& \cellcolor{myred!37.2}{20.\scalebox{0.75}{40}}
& \cellcolor{myred!44.7}{23.\scalebox{0.75}{76}} \\

\bottomrule
\end{tabular}%
}
\end{table*}

\footnotetext[1]{%
\texttt{Supermemory}: \url{https://github.com/supermemoryai/supermemory}.}

\footnotetext[2]{%
\texttt{Memobase}: \url{https://github.com/memodb-io/memobase}.}
\subsection{Path Interpretations}\label{sec:interpretation}

We provide path interpretations of \method{} for multi-hop reasoning in \Cref{tab:case_frank_lowy}. The importance of each path to the final prediction can be measured by the partial derivative of the prediction score with respect to the triples at each reasoning layer. The top-$k$ path interpretations are then obtained by selecting the top-$k$ longest paths with beam search.

As shown in \Cref{tab:case_frank_lowy}, \method{} successfully identifies the answer by connecting two key constraints in the question: the person who presented the Australia 2022 FIFA World Cup bid and the person born on October 22, 1930. Specifically, the first path starts from the entity ``the bid for the 2022 FIFA World Cup'' and follows the inverse relation of ``was one of the representatives of'' to reach ``Frank Lowy''. Then, through an entity-equivalence relation, it links ``Frank Lowy'' to ``Sir Frank P. Lowy'', whose birth date is ``22 October 1930''. The second path verifies the reasoning in the reverse direction by starting from the birth date and tracing back to the representative of the World Cup bid. These paths demonstrate that \method{} can effectively align different surface forms of the same entity and integrate multiple question constraints within a single-step retrieval process, showing its ability to perform interpretable multi-hop reasoning.

\begin{table*}[]
    \centering
    \caption{Path interpretations for multi-hop reasoning, where $r^{-1}$ denotes the inverse of original relation.}
    \label{tab:case_frank_lowy}
    \resizebox{.9\textwidth}{!}{%
    \begin{tabular}{@{}c|p{5in}@{}}
        \toprule
        \textbf{Question}             
        & Which man who presented the \textit{Australia 2022 FIFA World Cup bid} was born on \textit{October 22, 1930}? 
        \\ \midrule

        \textbf{Answer}               
        & Frank Lowy
        \\ \midrule

        \textbf{Sup. Doc.}        
        & [ ``Frank Lowy'', ``Australia 2022 FIFA World Cup bid''] 
        \\ \midrule

        \textbf{Paths}                                   
        & \begin{tabular}[c]{@{}p{5in}@{}} 
            1: (the bid for the 2022 fifa world cup, was one of the representatives of$^{-1}$, frank lowy) 
            $\to$ (frank lowy, equivalent, sir frank p lowy) 
            $\to$ (sir frank p lowy, was born on, 22 october 1930)
            \\
            2: (22 october 1930, was born on$^{-1}$, sir frank p lowy) 
            $\to$ (sir frank p lowy, equivalent, frank lowy) 
            $\to$ (frank lowy, was one of the representatives of, the bid for the 2022 fifa world cup)
        \end{tabular} 
        \\ \bottomrule
    \end{tabular}
    }
\end{table*}

\begin{figure}[t]
    \centering
    \includegraphics[width=\linewidth]{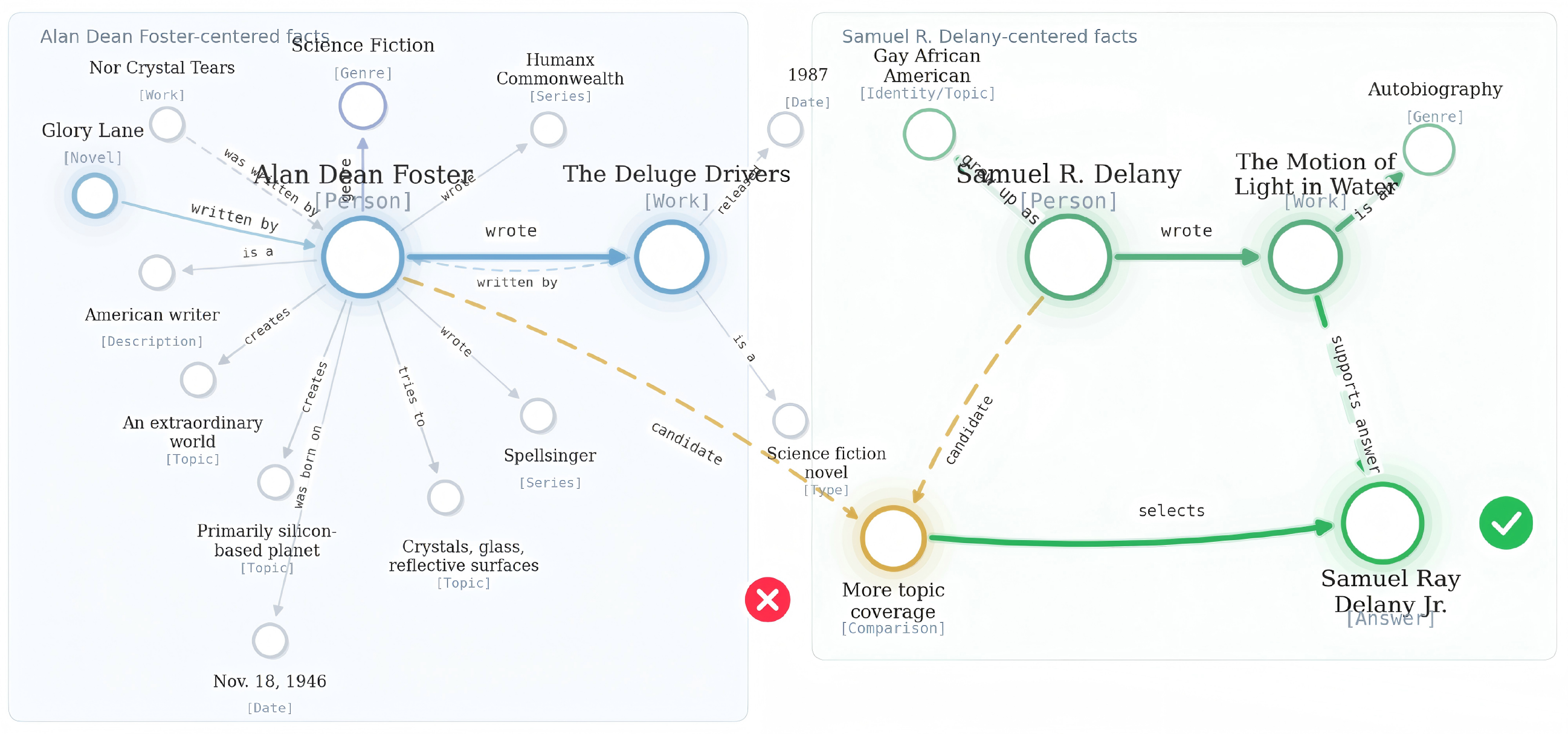}
    \caption{Visualization of the retrieved results.}
    \label{fig:retrieved_subgraph}
\end{figure}
\section{Dataset Details}

Table~\ref{tab:dataset_grid} summarizes the details of each dataset.

\begin{table*}[t]
\renewcommand{\arraystretch}{0.95}
\setlength{\tabcolsep}{3.4pt}
\centering
\caption{Dataset statistics and evaluation scenarios. We evaluate \method{} on three complementary categories: general and multi-hop QA, practical e-commerce review QA, and long-term agent memory. ``Train/Dev/Test'' denotes the standard split when available. For benchmark-only datasets without a conventional supervised training split, we report the total number of evaluation instances or benchmark scale.}
\vspace{0.55cm}
\resizebox{\textwidth}{!}{
\begin{tabular}{llp{3.2cm}p{3.5cm}p{3.8cm}p{4.4cm}p{3.5cm}}
\toprule
\textbf{Category} & \textbf{Dataset} & \textbf{Scale / Split} & \textbf{Evidence Source} & \textbf{Task Type} & \textbf{Key Capabilities} & \textbf{Main Metrics} \\
\midrule

\multirow{5}{*}{\makecell[l]{General /\\Multi-hop QA}}
& \texttt{NQ-Open} \scalebox{0.68}{(\mycite{kwiatkowski2019natural})}
& 79,168 / 8,757 / 3,610
& English Wikipedia
& Open-domain short-answer QA
& Factual retrieval; entity-level knowledge access; open-domain answer generation
& EM / F1 / Acc.; Recall@k \\

& \texttt{PopQA} \scalebox{0.68}{(\mycite{mallen2023not})}
& 14,267 QA pairs
& Wikidata triples + Wikipedia page-view popularity
& Entity-centric open-domain QA
& Long-tail factual recall; parametric vs. non-parametric memory; retrieval under entity popularity shift
& Acc. / EM; long-tail breakdown \\

& \texttt{HotpotQA} \scalebox{0.68}{(\mycite{yang2018hotpotqa})}
& 90,447 / 7,405 / 7,405
& Wikipedia paragraphs; 10-paragraph distractor setting
& Explainable 2-hop QA
& Bridge-entity recovery; comparison reasoning; sentence-level supporting facts
& Answer EM/F1; Support EM/F1; Joint EM/F1 \\

& \texttt{2WikiMultiHopQA} \scalebox{0.68}{(\mycite{ho2020constructing})}
& 167,454 / 12,576 / 12,576
& Wikipedia + Wikidata; 10 passages per instance
& 2--4 hop multi-hop QA
& Reasoning-path recovery; comparison, bridge, and bridge-comparison reasoning
& Answer EM/F1; Evidence / path recall \\

& \texttt{MuSiQue} \scalebox{0.68}{(\mycite{trivedi2022musique})}
& 19,938 / 2,417 / 2,459 \newline (24,814 total)
& Composed single-hop QA over textual passages
& 2--4 hop connected multi-hop QA
& Connected reasoning; shortcut-resistant evidence aggregation; multi-hop compositionality
& Answer EM/F1; Support / evidence recall \\
\midrule

\makecell[l]{E-commerce\\Review QA}
& \texttt{AmazonQA} \scalebox{0.68}{(\mycite{gupta2019amazonqa})}
& 923K questions; 3.6M answers; 14M reviews; 156K products
& Amazon product reviews, questions, answers, and product metadata
& Review-based QA with answerability annotation
& Noisy review retrieval; answerable / unanswerable detection; evidence synthesis from user-generated reviews
& BLEU / ROUGE; answerability Acc./F1; groundedness \\
\midrule

\multirow{2}{*}{\makecell[l]{Long-term\\Agent Memory}}
& \texttt{LongMemEval} \scalebox{0.68}{(\mycite{wu2024longmemeval})}
& 500 eval. instances per file; \texttt{S}: $\sim$115K tokens / 30--40 sessions; \texttt{M}: $\sim$1.5M tokens / $\sim$500 sessions; \texttt{Oracle}: evidence sessions only
& Long multi-session human--AI chat histories
& Long-term interactive memory QA
& Information extraction; multi-session reasoning; temporal reasoning; knowledge update; abstention
& Overall Acc.; category-wise Acc.; context tokens; latency \\

& \texttt{HaluMem} \scalebox{0.68}{(\mycite{chen2025halumem})}
& \texttt{Medium}: 20 users, 30,073 dialogue rounds, $\sim$160K tokens/user, 14,948 memory points, 3,467 QA pairs; \texttt{Long}: 53,516 rounds, $\sim$1M tokens/user
& Synthetic long-term human--AI interaction histories with memory points and multi-type questions
& Operation-level memory hallucination benchmark
& Memory extraction; memory updating; memory QA; hallucination, omission, and conflict propagation across memory operations
& Extraction R/P/F1; Updating C/H/O; QA C/H/O \\
\bottomrule
\end{tabular}
}
\label{tab:dataset_grid}
\end{table*}

\paragraph{General and Multi-hop QA.}
We first evaluate \method{} on a set of general open-domain and multi-hop QA benchmarks that stress different aspects of retrieval-augmented reasoning. \texttt{NQ-Open} is derived from Natural Questions and is widely used as a standard open-domain short-answer QA benchmark; it evaluates whether a system can retrieve and ground factual answers from a large Wikipedia-scale corpus. \texttt{PopQA} complements NQ by focusing on entity-centric factual questions whose subjects span different popularity levels, making it particularly useful for testing whether a retrieval or memory system can recover long-tail factual knowledge rather than relying only on parametric memorization. \texttt{HotpotQA} contains Wikipedia-based multi-hop questions with sentence-level supporting facts, allowing us to evaluate not only answer correctness but also whether the system can recover bridge evidence and produce interpretable reasoning chains. \texttt{2WikiMultiHopQA} further stresses structured multi-hop reasoning by combining Wikipedia text with Wikidata-derived relations and providing evidence paths for 2--4 hop questions. Finally, \texttt{MuSiQue} is designed to reduce shortcut reasoning by composing connected single-hop questions into 2--4 hop questions, making it a strong testbed for evaluating whether \method{} can retrieve and integrate multiple pieces of evidence in a genuinely compositional manner.

\paragraph{E-commerce Review-based QA.}
We use \texttt{AmazonQA} to evaluate \method{} in a practical, noisy, user-generated e-commerce setting. Unlike Wikipedia-style QA benchmarks, \texttt{AmazonQA} consists of real product questions, community answers, product reviews, and product metadata, and includes answerability annotations indicating whether a question can be answered from available reviews. This makes it a suitable benchmark for testing whether a memory system can identify useful evidence from noisy review collections, distinguish answerable from unanswerable questions, and synthesize grounded answers from multiple user-generated snippets. From the perspective of self-evolving memory, AmazonQA is especially valuable because the system must learn which review facts, product attributes, and user opinions are worth indexing for future retrieval, rather than simply matching a question to a clean encyclopedic passage.

\paragraph{Long-term Agent Memory.}
To move beyond conventional RAG evaluation, we further evaluate \method{} on long-term agent memory benchmarks. \texttt{LongMemEval} is designed to assess the long-term memory abilities of chat assistants over extended multi-session interaction histories. It covers five core memory abilities: information extraction, multi-session reasoning, temporal reasoning, knowledge updates, and abstention. This benchmark directly tests whether \method{} can retrieve sparse but relevant memory traces from long histories, combine evidence across sessions, respect temporal order, and update previously stored information when new interactions supersede old memories. We use \texttt{HaluMem} as a complementary benchmark for evaluating hallucination in memory systems. Rather than only measuring end-to-end QA accuracy, HaluMem decomposes memory evaluation into memory extraction, memory updating, and memory question answering, thereby revealing at which operational stage hallucinations, omissions, or conflicts arise. This is particularly important for our setting because errors introduced during graph construction or memory updating may propagate to graph-guided retrieval and final answer generation.

\paragraph{Evaluation Rationale.}
Together, these datasets form a progressively broader evaluation suite. NQ and PopQA test factual open-domain retrieval; HotpotQA, 2WikiMultiHopQA, and MuSiQue test multi-hop evidence composition; AmazonQA evaluates noisy real-world review memory in an e-commerce domain; LongMemEval tests long-horizon interactive memory; and HaluMem diagnoses operation-level hallucinations in memory systems. This combination allows us to evaluate \method{} not merely as a retrieval-augmented QA pipeline, but as a self-evolving memory system that must decide what to store, how to organize stored information, how to retrieve it under different query conditions, and how to update or suppress unreliable memories over time.
\section{Baselines and Metrics}
\label{subsec:Baselines_Metrics}
\paragraph{Baselines.}
We evaluate \method{} against state-of-the-art baselines, including their combined variants, which are grouped into \textbf{four} categories:
\begin{itemize}[leftmargin=*]
    \item \textbf{\textit{Base LLM:}} \texttt{GPT-4o-mini}~\cite{hurst2024gpt}.
    \item \textbf{\textit{Single-step RAGs:}} including \texttt{BM25}~\cite{robertson1994some}, \texttt{Contriever}~\cite{izacard2022unsupervised}, \texttt{GTR}~\cite{ni2022large}, \texttt{ColBERTv2}~\cite{santhanam2022colbertv2}, \texttt{RAPTOR}~\cite{sarthi2024raptor}, and \texttt{Proposition}~\cite{chen2024dense}.
    \item \textbf{\textit{Graph-enhanced RAGs:}} including \texttt{GraphRAG}~\cite{edge2024local}, \texttt{G-Retriever}~\cite{he2024g}, \texttt{LightRAG}~\cite{guo2024lightrag}, \texttt{HippoRAG}~\cite{gutierrez2024hipporag}, \texttt{HippoRAG} \texttt{2}~\cite{gutierrez2025rag}, \texttt{SubgraphRAG}~\cite{li2025simple}, \texttt{PropRAG}~\cite{wang2025proprag}, and the closely related \texttt{GFM-RAG}~\cite{luo2025gfm}.
    \item \textbf{\textit{Multi-step RAGs:}} \texttt{IRCoT}~\cite{trivedi2023interleaving}, \texttt{FLARE}~\cite{jiang2023active}, and \texttt{Adaptive-RAG}~\cite{jeong2024adaptive}.
\end{itemize}
In particular, \texttt{IRCoT}~\cite{trivedi2023interleaving} is a general multi-step reasoning framework that can be integrated with non-iterative retrievers, allowing both single-step RAG and graph-based methods to conduct multi-hop reasoning through interleaved retrieval and generation. 

\paragraph{Metrics.}
To evaluate retrieval quality, we report Recall@2 and Recall@5 for both retrieved entities and documents, denoted as R@2/5$_\textsf{E}$ and R@2/5$_\textsf{D}$, respectively.
For end-to-end QA evaluation, we use standard metrics, including Exact Match (EM), F1 score, Precision (P), and Recall (R), in the main experiments to comprehensively measure answer correctness and coverage. 
\section{Integrate SAGE into Multiple Memory Systems}
\label{app:host_integration_details}

\subsection{Host implementations}

\paragraph{Zep.}
The native context budget is 16,384 tokens.  Extraction uses GPT-4o-mini; answer generation and judging use  GPT-4o with temperature zero.

\paragraph{Memobase.}
We use Memobase v0.0.42.
Both arms share ingestion, durable memory state, profiles, queries,
reference times, local Qwen3-Embedding-8B embeddings, and answer
configuration.  Extraction and summarization use GPT-4o-mini; answers and judges
request GPT-4o at temperature zero.

\paragraph{MemOS.}
Native embeddings and the SAGE branch use all-mpnet-base-v2.  The graph index is prebuilt: normal MemOS writes do not
synchronize new memories into that index online.
LongMemEval answers and judges use GPT-4o-mini; HaluMem answers
and judges use GPT-4o. The MemReader extraction temperature is
0.6, distinct from the zero-temperature answer/judge settings.

\subsection{Deployment cost}

\begin{table}[t]
\centering
\small
\caption{Descriptive latency measurements.
All entries are milliseconds. Retrieve-plus-answer excludes the
offline evaluator.}
\label{tab:host_latency}
\begin{tabular}{llrr}
\toprule
Host/task & Stage & Native & SAGE \\
\midrule
Memobase HaluMem & Retrieval mean & 253.1 & 297.9 \\
 & Retrieval p95 & 393.5 & 523.2 \\
Memobase LongMemEval & Retrieval mean & 221.8 & 314.5 \\
 & Retrieval p95 & 287.1 & 376.0 \\
 & Retrieve+answer mean & 2,192.3 & 2,795.7 \\
MemOS LongMemEval & Search mean & -- & 70,385.9 \\
 & Search+answer mean & -- & 72,110.8 \\
\bottomrule
\end{tabular}
\end{table}

\section{Matched Retrieval Protocol for Restricted-Corpus Transfer}
\label{app:matched_protocol_details}

\subsection{Evaluation unit and restricted corpora}

The NQ and PopQA measurements use 1,000 queries each and
restricted document corpora. NQ contains 964 documents and
17,035 passages after 256-token segmentation; PopQA contains
1,793 documents and 8,722 passages. 

Each retriever produces ten documents. The retrieved documents
are segmented into passages and reranked using the same MPNet
passage reranker. Five passages are supplied to the common
answer generator. All 1,000 queries in each dataset were
evaluated without dropping failed or incomplete cases in the
reported matched rerun.

\subsection{Retrieval metric}

Passage R@k is a question-level hit rate: a query succeeds if
at least one selected passage contains a normalized gold-answer
string. Gold answers are used only for scoring, not for
retrieval or reranking.
This metric differs from fractional recall of annotated
support documents. Retrieving a support page does not ensure
that an answer-bearing passage appears in the final context.
Conversely, a valid answer-bearing passage may come from a
document that is not one of the annotated support pages.

The current paper uses the passage-level metric consistently
and does not compare it numerically with an earlier
document-title-based measurement. QA EM/F1 is reported
alongside retrieval hit rate to measure the final answer
effect under the same context-construction protocol.

\subsection{Meaning of no-target-adaptation(0-shot) transfer}

The writer and reader were trained on HotpotQA,
2WikiMultiHopQA, and MuSiQue, with no NQ- or PopQA-specific
gradient updates, tuning, or evolution rounds for the
reported transfer setting. This is target-task/label transfer,
not a claim of novel Wikipedia content or complete absence
of knowledge in generic pretrained encoders and LLMs.

\section{Diagnostic Analysis of Memory Failures}
\label{app:memory_failure_analysis}

This appendix analyzes \emph{where} long-term-memory errors arise rather than
treating every incorrect answer as a retrieval failure.  The audit is based on
the earlier standalone SAGE memory runs and is reported separately from the
paired host-integration experiments in
Appendix~\ref{app:host_integration_details}.  The audited sets contain 110
LongMemEval examples and 763 HaluMem-Medium examples.  Its purpose is
diagnostic: it localizes missing evidence, ranking failures, and post-retrieval
reasoning failures, but it is not used as the primary system-level comparison
for the host-integration results.

\subsection{Diagnostic Protocol and Scope}
\label{app:memory_failure_protocol}

\paragraph{Sequential failure attribution.}
For an incorrect question, we assign the first applicable stage in the
following order:
\begin{enumerate}
    \item \textbf{Writer omission / source-anchor failure.}
    At least one annotated gold memory is absent from the stored corpus, or a
    gold document has no usable entity in the \texttt{document2entities}
    mapping and therefore cannot be anchored into the KG.
    \item \textbf{Graph reachability failure.}
    The gold memories are written and anchorable, but at least one required
    gold document lies outside the two-hop \emph{undirected} KG neighborhood
    of the exact-query seed entities.
    \item \textbf{Top-$k$ ranking failure.}
    All required gold documents are structurally reachable under the above
    proxy, but at least one required document is absent from the reader's
    top-10 retrieved context.
    \item \textbf{Answer-generation failure.}
    All annotated gold documents are present in the top-10 context, but the
    final answer is still scored incorrect.
\end{enumerate}
Because the stages are applied sequentially, they are mutually exclusive
within each audited question set.

\paragraph{What the diagnostic does \emph{not} establish.}
The two-hop neighborhood is a structural proxy, not a reconstruction of the
actual GFM computation path.  A globally initialized or deeper reader can in
principle use evidence outside this proxy, and the presence of a two-hop path
does not prove that the GFM actually propagated information along that path.
Similarly, the source-anchor test detects a missing or unusable anchor; it
does not verify that an existing anchor is semantically the correct entity,
because the benchmarks do not provide gold entity-level anchoring
annotations.  Finally, the last stage localizes an error to the
post-retrieval portion of the pipeline, but it can include answer selection,
temporal arithmetic, aggregation, polarity conversion, or evaluator mismatch
rather than a single unique causal mechanism.

\subsection{Gold-Memory and Source-Anchor Coverage}
\label{app:memory_failure_coverage}

Both audited datasets have 100\% gold-memory coverage: the annotated evidence
is present in the stored corpus for every audited question.  The
category-specific source-anchor audit is shown in
Table~\ref{tab:memory_failure_anchor_coverage}.

\begin{table*}[t]
\centering
\caption{Gold-document writing and source-anchor coverage in the diagnostic
audit.  ``Usable KG anchor'' means that the gold document has at least one
entity in the \texttt{document2entities} mapping.  HaluMem's temporal row is a
cue-based proxy rather than an official benchmark category.}
\label{tab:memory_failure_anchor_coverage}
\small
\setlength{\tabcolsep}{7pt}
\begin{tabular}{lrrr}
\toprule
Dataset / type & Gold documents & Written in corpus & Usable KG anchor\\
\midrule
LongMemEval KU & 36 & 36 & 36\\
LongMemEval TR & 61 & 61 & 61\\
LongMemEval MS & 82 & 82 & 82\\
HaluMem KU & 62 & 62 & 62\\
HaluMem temporal-cue proxy & 810 & 810 & 806\\
HaluMem MS & 93 & 93 & 93\\
\bottomrule
\end{tabular}
\end{table*}

The four anchor-incomplete documents in the HaluMem temporal-cue proxy occur
only in negative/unknown questions for which the model correctly returns an
unknown answer.  They therefore do not contribute a writer/source-anchor
failure among the incorrect temporal-proxy examples.  This result should not
be read as proof that the writer never creates a semantically wrong anchor;
the proxy only tests whether an anchor is available.

\subsection{LongMemEval: Overall and Category-Level Attribution}
\label{app:memory_failure_longmemeval}

Across all 110 audited LongMemEval questions, 38 are correct, 27 are assigned
to graph-reachability failure, 27 to top-10 ranking failure, and 18 to
answer-generation failure.  No question is assigned to writer omission under
the above gold-memory-coverage test.  Thus, the errors are distributed across
structure, ranking, and post-retrieval reasoning rather than being explained
by one single stage.

\begin{table}[t]
\centering
\caption{Overall failure-stage attribution for the 110-question LongMemEval
audit.}
\label{tab:memory_failure_long_overall}
\small
\setlength{\tabcolsep}{8pt}
\begin{tabular}{lrr}
\toprule
Outcome / stage & Count & Share\\
\midrule
Correct & 38 & 34.55\%\\
Writer/source-anchor & 0 & 0.00\%\\
Graph reachability & 27 & 24.55\%\\
Top-10 ranking & 27 & 24.55\%\\
Answer generation & 18 & 16.36\%\\
\bottomrule
\end{tabular}
\end{table}

LongMemEval provides official question types including Knowledge Update (KU),
Temporal Reasoning (TR), and Multi-Session Reasoning (MS).  These three
subsets contain 18, 29, and 32 audited questions, respectively.  The
category-level decomposition is reported in
Table~\ref{tab:memory_failure_long}.  Percentages in the last four columns use
the \emph{incorrect examples of that category} as the denominator.

\begin{table*}[t]
\centering
\caption{LongMemEval category-level error attribution.  Each entry in the
failure columns is count (percentage of the incorrect examples in that row).}
\label{tab:memory_failure_long}
\small
\setlength{\tabcolsep}{5pt}
\begin{tabular}{lrrrrrrr}
\toprule
Type & Total & Correct & Errors & Writer/anchor & Graph & Top-$k$ & Generation\\
\midrule
KU & 18 & 5 & 13
& 0 (0.00\%)
& 4 (30.77\%)
& \textbf{7 (53.85\%)}
& 2 (15.38\%)\\
TR & 29 & 7 & 22
& 0 (0.00\%)
& 3 (13.64\%)
& 9 (40.91\%)
& \textbf{10 (45.45\%)}\\
MS & 32 & 4 & 28
& 0 (0.00\%)
& \textbf{13 (46.43\%)}
& 11 (39.29\%)
& 4 (14.29\%)\\
\bottomrule
\end{tabular}
\end{table*}

\subsubsection{Knowledge Update: Missing the Latest State or Selecting a Stale One}
\label{app:memory_failure_long_ku}

For LongMemEval KU, 7/13 errors (53.85\%) are top-10 ranking failures,
4/13 (30.77\%) are structural-reachability failures, and 2/13 (15.38\%)
remain after all annotated evidence has entered the answer context.  The
dominant pattern is therefore not that the update was never written; rather,
the latest state is inaccessible, is ranked below an older state, or is
available but not treated as superseding the stale fact.

\paragraph{Case KU-1: latest fact written but structurally unreachable.}
For \texttt{longmemeval\_oracle\_train\_159}, the question is
``How many issues of National Geographic have I finished reading?''  The gold
answer is \emph{five}, whereas the model answers \emph{three issues}.  An
earlier session saying that three issues had been completed is ranked first.
A later session explicitly stating that the fifth issue has been finished is
present in memory, but it does not enter the top-10 and lies outside the
two-hop neighborhood of the exact-query seed.  The failure chain is therefore
\[
\text{latest fact written}
\rightarrow
\text{latest session structurally unreachable}
\rightarrow
\text{old session retrieved}
\rightarrow
\text{stale answer}.
\]

\paragraph{Case KU-2: latest fact reachable but ranked out.}
For \texttt{longmemeval\_oracle\_train\_142}, the question asks
``How many MCU films did I watch in the last 3 months?''  The gold answer is
5 and the model answers 4.  The older four-film session is ranked eighth.
The later session updating the count to five is written and structurally
reachable, but it does not enter the top-10.  This isolates a reader-ranking
failure from a memory-writing failure.

\paragraph{Case KU-3: both versions retrieved, but the answerer selects the
obsolete value.}
For \texttt{longmemeval\_oracle\_train\_137}, the question asks
``How many stars do I need to reach the gold level on my Starbucks Rewards
app?''  The updated value, 120 stars, is ranked first; the older 125-star
memory is ranked third.  The model nevertheless answers 125.  Because the
newer fact is both retrieved and ranked above the stale fact, this is a
post-retrieval update-resolution failure: the answerer does not use the
update relation to invalidate the obsolete value.

\subsubsection{Temporal Reasoning: Evidence Omission versus Temporal Computation}
\label{app:memory_failure_long_tr}

LongMemEval TR exhibits a different profile.  Among its 22 errors,
9 (40.91\%) are ranking failures and 10 (45.45\%) occur after all annotated
evidence is already in the top-10.  A finer audit identifies 12 errors with
temporal-evidence omission, 7 with duration/time-difference mistakes, and
3 with ordering or temporal-qualifier mistakes.  Therefore, improving
retrieval alone cannot resolve the full TR error set.

\paragraph{Case TR-1: both temporal endpoints reachable but omitted from
top-10.}
For \texttt{longmemeval\_oracle\_train\_6}, the question is
``How many days did it take for me to find a house I loved after starting to
work with Rachel?''  The required sessions record February 15 as the start of
working with Rachel and March 1 as the date on which a liked house was found.
Both sessions are structurally reachable, but neither appears in the top-10,
so the answerer cannot reconstruct the interval and returns unknown instead of
the gold 14 days.

\paragraph{Case TR-2: both dates retrieved, but duration arithmetic is wrong.}
For \texttt{longmemeval\_oracle\_train\_27}, the question asks for the number
of days between buying the Adidas running shoes and the Converse shoelaces
breaking.  The January 10 purchase is ranked first and the January 24
shoelace event is ranked tenth.  Both temporal endpoints are therefore
available, but the model answers 25 days instead of the correct 14 days.
This is a post-retrieval duration-arithmetic failure.

\paragraph{Case TR-3: both dates retrieved, but event ordering is reversed.}
For \texttt{longmemeval\_oracle\_train\_32}, the question is
``Which device did I set up first, the smart thermostat or the new router?''
The router event is dated January 15 and ranked first; the thermostat event is
dated February 10 and ranked third.  Despite retrieving both dates, the model
answers that the thermostat came first.  The evidence acquisition succeeds,
but the answerer mishandles the temporal ordering implied by \emph{first}.

\subsubsection{Multi-Session Reasoning: Incomplete Support Coverage and
Aggregation}
\label{app:memory_failure_long_ms}

MS is the LongMemEval category most strongly associated with graph/evidence
coverage.  Graph reachability and top-10 ranking jointly account for 24/28
errors (85.72\%).  The main retrieval-side problem is that evidence distributed
across several sessions is not simultaneously available; when all sessions are
available, aggregation, category boundaries, and arithmetic can still fail.

\paragraph{Case MS-1: only part of a distributed weekly pattern is retrieved.}
For \texttt{longmemeval\_oracle\_train\_94}, the question is
``How many fitness classes do I attend in a typical week?''  The gold answer
is 5.  The required evidence is distributed over several sessions, including
Monday BodyPump, Saturday Hip Hop Abs, Sunday yoga, and other weekly classes.
Only the Monday BodyPump session enters the top-10; one key document lies
outside the two-hop proxy and other required sessions are also absent from the
reader context.  Because ``a typical week'' is not itself a single explicit
entity, the reader does not recover a stable bridge that connects all of the
distributed supporting sessions.

\paragraph{Case MS-2: all annotated evidence retrieved, but aggregation is
incorrect.}
For \texttt{longmemeval\_oracle\_train\_114}, the answer requires combining
17 poems (rank 7), 5 short stories (rank 5), and one writing-challenge piece
(rank 1).  All supporting evidence is present, so the correct total is
\(17+5+1=23\).  The model answers 24 and incorrectly expands the number of
short stories.  This case isolates multi-session aggregation, category
boundaries, and arithmetic from retrieval coverage.

\subsection{HaluMem: Category-Level Attribution}
\label{app:memory_failure_halumem}

HaluMem-Medium shows a more concentrated retrieval bottleneck than
LongMemEval.  Across the 763-example audit, gold-memory coverage is 100\%.
Moreover, 98.30\% of questions have all annotated gold memories inside the
two-hop structural proxy.  In contrast, only 21.76\% have at least one gold
memory in the reader top-10, and only 6.42\% have \emph{all} gold memories in
the top-10.  Conditioned on structural reachability, the resulting
top-$k$ ranking-failure rate is 93.47\%.

\begin{table}[t]
\centering
\caption{HaluMem structural-coverage diagnostics.}
\label{tab:memory_failure_halu_overall}
\small
\setlength{\tabcolsep}{6pt}
\begin{tabular}{lr}
\toprule
Metric & Result\\
\midrule
Gold memory coverage & 100.00\%\\
All gold two-hop reachable & 98.30\%\\
At least one gold in top-10 & 21.76\%\\
All gold in top-10 & 6.42\%\\
Top-$k$ failure conditioned on reachability & 93.47\%\\
\bottomrule
\end{tabular}
\end{table}

HaluMem provides an official \emph{Dynamic Update} category, which we denote
KU, and an official \emph{Multi-hop Inference} category, which we denote MS.
It does \emph{not} provide an official TR category.  For diagnostic purposes
only, we construct a temporal-cue proxy consisting of questions containing
expressions such as \emph{latest}, \emph{current}, \emph{as of},
\emph{before}, \emph{after}, \emph{by}, or \emph{how long}.  This proxy
contains 387 examples and can overlap both KU and MS; consequently, the rows
of Table~\ref{tab:memory_failure_halu} must not be summed.

\begin{table*}[t]
\centering
\caption{HaluMem category-level diagnostic attribution.  Percentages in the
failure columns use incorrect examples in the corresponding row as the
denominator.  The temporal-cue proxy overlaps other categories.  The last row
is a manual evaluator-sensitivity view rather than a replacement of the
automated score.}
\label{tab:memory_failure_halu}
\small
\setlength{\tabcolsep}{4pt}
\begin{tabular}{lrrrrrrr}
\toprule
Type & Total & Correct & Errors & Writer/anchor & Graph & Top-$k$ & Generation\\
\midrule
KU & 43 & 4 & 39
& 0 (0.00\%)
& 0 (0.00\%)
& \textbf{36 (92.31\%)}
& 3 (7.69\%)\\
Temporal-cue proxy & 387 & 94 & 293
& 0 (0.00\%)
& 4 (1.37\%)
& \textbf{271 (92.49\%)}
& 18 (6.14\%)\\
MS, automated & 39 & 0 & 39
& 0 (0.00\%)
& 0 (0.00\%)
& \textbf{38 (97.44\%)}
& 1 (2.56\%)\\
MS, manual sensitivity & 39 & 1 & 38
& 0 (0.00\%)
& 0 (0.00\%)
& \textbf{38 (100.00\%)}
& 0 (0.00\%)\\
\bottomrule
\end{tabular}
\end{table*}

\subsubsection{Knowledge Update: Stale Memories Outrank the Latest State}
\label{app:memory_failure_halu_ku}

HaluMem KU is overwhelmingly a top-10 coverage problem:
36/39 errors (92.31\%) occur after the relevant memories are written and
structurally reachable but before they enter the reader context.  In 6/43 KU
questions, the retrieved context also contains a detectable stale value,
illustrating the competition between old and updated states.

\paragraph{Case KU-4: the latest health update is omitted while an older state
is retrieved.}
For \texttt{halumem\_medium\_train\_271}, the question asks
``What is the latest health status of Johnson Joseph as of December 20, 2035?''
The gold answer is \emph{Poor}, whereas the model answers
\emph{Highly Positive}.  One gold memory states that Johnson Joseph's physical
health status changed from ``Excellent'' to ``Poor'', and another supporting
memory records chronic fatigue.  Both are written and structurally reachable,
but neither enters the top-10.  Rank 4 instead contains an older update from
``Good'' to ``Excellent'', while higher-ranked context also contains unrelated
records about other people and topics.  This is a concrete stale-state ranking
failure:
\[
\text{latest state exists}
\rightarrow
\text{latest evidence ranked out}
\rightarrow
\text{stale state retrieved}
\rightarrow
\text{outdated answer}.
\]

\paragraph{Case KU-5: the current update is retrieved but the answerer still
abstains.}
For \texttt{halumem\_medium\_train\_308}, the question asks
``What is Johnson Joseph's current monthly income as of July 18, 2039?''
The gold answer is \(\$8{,}000\), and rank 2 explicitly states that the
monthly income was updated from \(\$7{,}000\) to \(\$8{,}000\).  The model
nevertheless answers unknown.  This is an answer-generation omission rather
than a retrieval failure.

\subsubsection{Temporal-Cue Proxy: Time-Conditioned Retrieval and Polarity}
\label{app:memory_failure_halu_tr}

The temporal-cue proxy contains 293 incorrect examples.  Of these,
271 (92.49\%) are top-10 ranking failures, 4 (1.37\%) fail the structural
reachability proxy, and 18 (6.14\%) remain incorrect after annotated evidence
is in the reader context.  Because this subset is cue-derived and overlaps
other HaluMem categories, these numbers are used only for diagnostic behavior,
not as an official HaluMem TR score.

\paragraph{Case TR-4: all memories at the relevant date are reachable but
ranked out.}
For \texttt{halumem\_medium\_train\_113}, the question asks
``What kind of pet ownership practices might Martin adopt after March 23,
2032?''  The gold answer is conservation-focused and ethical practices.
Three gold memories at the specified date mention conservation and ethical
considerations, conservation-focused pet ownership, and responsible ownership
and conservation efforts.  All three are structurally reachable, but none
enters the top-10; the model consequently reports insufficient information.
The failure is time-conditioned reader coverage rather than missing
timestamps in memory.

\paragraph{Case TR-5: temporal/causal evidence is retrieved but the polarity
conversion fails.}
For \texttt{halumem\_medium\_train\_550}, the question asks
``Did Sarah decide to continue consuming sugary sodas after the seminar on
Aug 13, 2027?''  A rank-3 memory explicitly states that the seminar led to a
decision to ``avoid sugary sodas entirely.''  The model still answers unknown.
Here the evidence is available, but the answerer fails to combine the temporal
condition \emph{after the seminar}, the causal relation between the seminar
and the decision, and the polarity conversion
\[
\text{avoid entirely}
\Rightarrow
\text{did not continue}.
\]

\subsubsection{Multi-Session Reasoning: Complete Support Coverage Is the Main
Bottleneck}
\label{app:memory_failure_halu_ms}

The official HaluMem Multi-hop Inference subset contains 39 examples.  Under
the automated diagnostic, 38 are top-10 ranking failures and one is labeled
as an answer-generation failure.  Manual inspection of that one generation
case, \texttt{halumem\_medium\_train\_2394}, shows that the model answer
contains the annotated advice ``expanding the studio's creative team'', but
the longer response falls below the local token-overlap threshold.  Moving
only this case to correct yields the sensitivity row in
Table~\ref{tab:memory_failure_halu}: all 38 remaining errors are top-10
coverage failures.  We keep the automated result as the primary record and
use the manual row only to expose evaluator sensitivity.

\paragraph{Case MS-3: the motivation is retrieved but the concrete strategy is
missing.}
For \texttt{halumem\_medium\_train\_207}, the question asks
``How did Johnson's exploratory motivation influence his stress management
strategies?''  The answer requires combining three memories:
(i) his desire to explore new skills through a part-time role;
(ii) the exploratory motivation that led him to test new approaches; and
(iii) the concrete stress-management strategies of mindfulness and scheduled
downtime.  The reader retrieves only the second memory (rank 4); the first
and third do not enter the top-10.  The system therefore recovers the
motivation but not the endpoint that the question asks for.

\paragraph{Case MS-4: all supporting memories are omitted and the model falls
back to a plausible generic statement.}
For \texttt{halumem\_medium\_train\_180}, the question asks
``What role did community support play in Johnson Joseph's career progress?''
The gold answer is ``Increased guest traffic and referrals.''  Three
supporting memories describe local ties producing referrals and guest numbers,
a new partnership increasing guest traffic, and Johnson's appreciation of
community support.  All three are structurally reachable but none appears in
the top-10.  The model answers that community support provided collaborative
opportunities that improved his career prospects.  The response is plausible
in general, but it does not recover the specific effect supported by memory.
This example illustrates how incomplete evidence coverage can be followed by
a fluent but underspecified answer based on generic prior knowledge.

\subsection{Cross-Category Interpretation}
\label{app:memory_failure_cross_category}

The audits reveal different dominant failure mechanisms across memory
abilities.

\paragraph{Knowledge Update.}
The latest memory is generally written, but the updated session can be
structurally inaccessible or ranked below an older state.  Even when both
versions enter the context, the answerer can still select the stale value or
ignore qualifiers such as \emph{latest}, \emph{current}, and \emph{as of}.
This indicates that robust long-term memory requires not only retrieval
coverage but also explicit handling of version/state validity.

\paragraph{Temporal Reasoning.}
TR contains both retrieval-side and reasoning-side errors.  Retrieval can omit
one or both temporal endpoints, while successful retrieval can still be
followed by incorrect duration arithmetic, relative-time normalization, event
ordering, or handling of qualifiers such as \emph{first}, \emph{before}, and
\emph{after}.  Therefore, increasing top-$k$ coverage alone is insufficient
for complete temporal-memory reasoning.

\paragraph{Multi-Session Reasoning.}
MS is most sensitive to simultaneous support coverage.  Relevant facts are
distributed across sessions, and a missing bridge or a single ranked-out
support can break the complete evidence set.  LongMemEval additionally shows
that even complete retrieval does not eliminate aggregation, deduplication,
category-boundary, and arithmetic errors.

Overall, LongMemEval exhibits a mixed bottleneck across graph reachability,
reader ranking, and post-retrieval reasoning.  HaluMem is more strongly
concentrated at reader ranking/coverage: the annotated memories are nearly
always written and structurally reachable, yet the complete evidence set is
rarely present in the top-10.

\subsection{Relation to the Paired Host-Integration Experiments}
\label{app:memory_failure_host_relation}

The dominant HaluMem ranking failures motivate our paired host experiments:
strong memory systems already provide mature extraction, persistence,
updating, and answer-generation machinery, while SAGE specifically targets
structure-aware memory access.  Replacing the host's native
memory-access/retrieval backend with SAGE therefore tests whether improving
this layer can add value without requiring the rest of the host memory stack
to be replaced; the exact intervention boundary is described in
Appendix~\ref{app:host_integration_details}.

However, the diagnostic audit above and the host-integration experiments are
different evaluations.  We therefore do \emph{not} claim that the earlier
failure percentages causally explain every Native$\rightarrow$+SAGE
improvement.  Establishing that stronger causal link would require
question-level provenance for the discordant pairs: which new evidence entered
the SAGE context, whether it was semantically required, whether native
evidence was displaced under the matched context budget, whether a stale fact
was replaced by the current fact, and whether the answer actually used the
added evidence.  We consequently use the present audit to identify failure
modes and motivate the intervention, while reserving performance attribution
for the within-host matched comparisons.


\begin{table*}[t]
\centering
\renewcommand{\arraystretch}{0.90}
\setlength{\tabcolsep}{4.0pt}

\caption{End-to-end training ablations and selector diagnostics.
Panel (a) reports EM and F1 (\%) separately on three multi-hop QA benchmarks. The one-shot and equal-compute
controls isolate the effects of additional optimization and writer-reader
coupling, respectively. Panel (b) further reports selector diagnostics.
Best results are in \textbf{bold} and runner-ups are
\underline{underlined}; darker cells indicate better performance.}
\label{tab:coupling_ablation}

\vspace{0.15cm}

\resizebox{\textwidth}{!}{
\begin{tabular}{p{0.37\textwidth}cccccc}
\toprule
\multirow{2}{*}{\textbf{Configuration}}
& \multicolumn{2}{c}{\textbf{HotpotQA}}
& \multicolumn{2}{c}{\textbf{MuSiQue}}
& \multicolumn{2}{c}{\textbf{2Wiki}} \\
\cmidrule(lr{0.5mm}){2-3}
\cmidrule(lr{0.5mm}){4-5}
\cmidrule(l{0.5mm}){6-7}

&
\cellcolor{myblue!12}\textbf{EM}
& \cellcolor{myred!12}\textbf{F1}
& \cellcolor{myblue!12}\textbf{EM}
& \cellcolor{myred!12}\textbf{F1}
& \cellcolor{myblue!12}\textbf{EM}
& \cellcolor{myred!12}\textbf{F1} \\

\midrule
\multicolumn{7}{l}{\textit{Full end-to-end model}} \\

\textbf{Full \method{}, +2 rounds}
& \cellcolor{myblue!70}{\textbf{56.1}}
& \cellcolor{myred!70}{\textbf{80.8}}
& \cellcolor{myblue!70}{\textbf{34.2}}
& \cellcolor{myred!70}{\textbf{53.1}}
& \cellcolor{myblue!70}{\textbf{74.0}}
& \cellcolor{myred!70}{\textbf{80.4}} \\

\cmidrule(r{1mm}){1-1}
\cmidrule(l{0.5mm}){2-7}
\multicolumn{7}{l}{\textit{Component ablations}} \\

w/o reader-aware writer training
& \cellcolor{myblue!20}{52.4}
& \cellcolor{myred!41}{76.7}
& \cellcolor{myblue!31}{30.4}
& \cellcolor{myred!20}{48.3}
& \cellcolor{myblue!34}{68.1}
& \cellcolor{myred!29}{76.8} \\

w/o structured query planning
& \cellcolor{myblue!41}{54.0}
& \cellcolor{myred!59}{79.2}
& \cellcolor{myblue!44}{31.7}
& \cellcolor{myred!39}{50.1}
& \cellcolor{myblue!54}{71.4}
& \cellcolor{myred!45}{78.2} \\

w/o global soft addressing
& \cellcolor{myblue!20}{52.4}
& \cellcolor{myred!39}{76.5}
& \cellcolor{myblue!25}{29.8}
& \cellcolor{myred!14}{47.7}
& \cellcolor{myblue!38}{68.7}
& \cellcolor{myred!10}{75.1} \\

w/o structural gating
& \cellcolor{myblue!32}{53.3}
& \cellcolor{myred!44}{77.1}
& \cellcolor{myblue!29}{30.2}
& \cellcolor{myred!33}{49.5}
& \cellcolor{myblue!51}{70.8}
& \cellcolor{myred!39}{77.7} \\

w/o context--schema decomposition, parameter-matched
& \cellcolor{myblue!31}{53.2}
& \cellcolor{myred!44}{77.1}
& \cellcolor{myblue!34}{30.7}
& \cellcolor{myred!31}{49.3}
& \cellcolor{myblue!50}{70.7}
& \cellcolor{myred!36}{77.4} \\

w/o GFM pretraining, same architecture
& \cellcolor{myblue!33}{53.4}
& \cellcolor{myred!43}{77.0}
& \cellcolor{myblue!33}{30.6}
& \cellcolor{myred!37}{49.9}
& \cellcolor{myblue!48}{70.3}
& \cellcolor{myred!33}{77.1} \\

w/o selector regularization
& \cellcolor{myblue!65}{\underline{~55.7~}}
& \cellcolor{myred!69}{\underline{~80.7~}}
& \cellcolor{myblue!63}{\underline{~33.5~}}
& \cellcolor{myred!62}{\underline{~52.3~}}
& \cellcolor{myblue!64}{\underline{~73.0~}}
& \cellcolor{myred!63}{79.8} \\

w/o query-conditioned selector
& \cellcolor{myblue!55}{55.0}
& \cellcolor{myred!60}{79.4}
& \cellcolor{myblue!58}{33.0}
& \cellcolor{myred!49}{51.1}
& \cellcolor{myblue!62}{72.6}
& \cellcolor{myred!64}{\underline{~79.9~}} \\

\cmidrule(r{1mm}){1-1}
\cmidrule(l{0.5mm}){2-7}
\multicolumn{7}{l}{\textit{Optimization and coupling controls}} \\

w/o self-evolution, one-shot
& \cellcolor{myblue!10}{51.7}
& \cellcolor{myred!10}{72.4}
& \cellcolor{myblue!10}{28.4}
& \cellcolor{myred!10}{47.3}
& \cellcolor{myblue!10}{64.1}
& \cellcolor{myred!17}{75.7} \\

w/o closed-loop coupling, equal compute
& \cellcolor{myblue!50}{54.6}
& \cellcolor{myred!54}{78.6}
& \cellcolor{myblue!40}{31.3}
& \cellcolor{myred!30}{49.2}
& \cellcolor{myblue!46}{70.1}
& \cellcolor{myred!35}{77.3} \\

\bottomrule
\end{tabular}%
}

\vspace{0.25cm}

\resizebox{\textwidth}{!}{
\begin{tabular}{lcccc}
\toprule
\multicolumn{5}{l}{\textbf{(b) Selector diagnostics}} \\
\textbf{Variant}
& \textbf{Avg. QA F1}
& \textbf{Avg. R@5}
& \textbf{Activated nodes}
& \textbf{Subgraph components} \\
\midrule

Full selector + all regularizers
& \cellcolor{myred!70}{\textbf{71.4}}
& \cellcolor{myblue!70}{\textbf{73.1}}
& 18.6\%
& 1.7 \\

w/o NCE
& \cellcolor{myred!43}{70.9}
& \cellcolor{myblue!43}{72.6}
& 26.2\%
& 2.0 \\

w/o size regularization
& \cellcolor{myred!65}{\underline{~71.3~}}
& \cellcolor{myblue!65}{\underline{~73.0~}}
& 34.7\%
& 1.5 \\

w/o connectivity regularization
& \cellcolor{myred!48}{71.0}
& \cellcolor{myblue!48}{72.7}
& 18.9\%
& 3.8 \\

w/o all selector regularizers
& \cellcolor{myred!37}{70.8}
& \cellcolor{myblue!37}{72.5}
& 31.8\%
& 3.4 \\

w/o selector
& \cellcolor{myred!10}{70.3}
& \cellcolor{myblue!10}{72.0}
& 100.0\%
& -- \\

\bottomrule
\end{tabular}%
}

\end{table*}
\section*{Limitations}
SAGE treats graph memory as a dynamic substrate for writing, reading, and self-evolution, but its effectiveness still depends on the quality of entity extraction, relation writing, source anchoring, and reader feedback. Our experiments show promising results across multi-hop QA, open-domain retrieval, review-based QA, and long-term agent-memory benchmarks, but the current system still leaves room for improvement on memory updating, high-coverage extraction, and hallucination control in more realistic deployments. 

\section*{Broader Impact}
This work may have positive societal impact by improving the reliability and grounding of long-horizon language agents. A structure-aware and self-evolving memory system can help agents recover evidence chains from fragmented cues, maintain more consistent long-term interactions, and reduce unsupported answers in applications such as knowledge assistance, research support, customer support, and review-based question answering. At the same time, long-term agent memory raises important risks. If deployed on personal or sensitive interaction histories, such systems may store private information, infer user preferences, preserve outdated or incorrect memories, or enable profiling and surveillance. Incorrect graph writes or retrieval failures may also lead to confidently grounded but wrong answers. Practical deployments should therefore use consent-based data collection, data minimization, access control, deletion and forgetting mechanisms, provenance tracking, auditing, and human oversight for high-stakes use cases.

\section*{Compute Resources}
All experiments were run on a server equipped with 8 NVIDIA A100 GPUs. The main computational cost of SAGE comes from graph-memory construction, GFM-based graph propagation, selector regularization, and entity-to-document projection. Appendix~\ref{app:complexity2} analyzes the training and inference complexity in terms of the number of graph nodes $n$, edges $m$, hidden dimension $d$, propagation layers $L$, batch size $B$, pseudo-queries $M$, and entity-document links. In our implementation, structural features and entity-document indices can be precomputed and cached, while edge-level gates are computed in chunks to reduce peak GPU memory from $O(|E|d)$ to $O(C_e d)$ for chunk size $C_e$. The dominant inference cost is one or a small number of query-conditioned graph propagations followed by sparse document projection, making the reader suitable for repeated evaluation inside the self-evolving writer-reader loop.

\section*{Licenses and Existing Assets}
This paper uses existing public benchmarks and baselines, including NQ-Open, PopQA, HotpotQA, 2WikiMultiHopQA, MuSiQue, AmazonQA, LongMemEval, HaluMem, BM25, Contriever, GTR, ColBERTv2, RAPTOR, GraphRAG, G-Retriever, LightRAG, HippoRAG, HippoRAG 2, SubgraphRAG, PropRAG, GFM-RAG, IRCoT, FLARE, and Adaptive-RAG. We cite the original papers or repositories for these assets and use them only for research evaluation under their stated licenses and terms of use. We do not redistribute modified versions of the datasets beyond the preprocessing scripts and instructions needed for reproducibility. The released code is intended for research use and includes documentation for environment setup, data preparation, training, and evaluation.

\subsubsection*{AI use statement}
Generative AI tools are only used for English grammar and format checking. No scientific claims, results, analyses, citations, or other substantive manuscript content were changed.

\end{document}